\PassOptionsToPackage{unicode}{hyperref}
\PassOptionsToPackage{hyphens}{url}
\PassOptionsToPackage{dvipsnames,svgnames,x11names}{xcolor}
\documentclass[12pt]{article}

\usepackage{bbm}
\usepackage{algorithm}
\usepackage{algpseudocode}
\usepackage{amsthm}
\usepackage{authblk}

\usepackage{amsmath,amssymb}
\usepackage{iftex}
\ifPDFTeX
  \usepackage[T1]{fontenc}
  \usepackage[utf8]{inputenc}
  \usepackage{textcomp} 
\else 
  \usepackage{unicode-math}
  \defaultfontfeatures{Scale=MatchLowercase}
  \defaultfontfeatures[\rmfamily]{Ligatures=TeX,Scale=1}
\fi
\usepackage{lmodern}
\ifPDFTeX\else  
\fi
\IfFileExists{upquote.sty}{\usepackage{upquote}}{}
\IfFileExists{microtype.sty}{
  \usepackage[]{microtype}
  \UseMicrotypeSet[protrusion]{basicmath} 
}{}
\makeatletter
\@ifundefined{KOMAClassName}{
  \IfFileExists{parskip.sty}{%
    \usepackage{parskip}
  }{
    \setlength{\parindent}{0pt}
    \setlength{\parskip}{6pt plus 2pt minus 1pt}}
}{
  \KOMAoptions{parskip=half}}
\makeatother
\usepackage{xcolor}
\makeatletter
\ifx\paragraph\undefined\else
  \let\oldparagraph\paragraph
  \renewcommand{\paragraph}{
    \@ifstar
      \xxxParagraphStar
      \xxxParagraphNoStar
  }
  \newcommand{\xxxParagraphStar}[1]{\oldparagraph*{#1}\mbox{}}
  \newcommand{\xxxParagraphNoStar}[1]{\oldparagraph{#1}\mbox{}}
\fi
\ifx\subparagraph\undefined\else
  \let\oldsubparagraph\subparagraph
  \renewcommand{\subparagraph}{
    \@ifstar
      \xxxSubParagraphStar
      \xxxSubParagraphNoStar
  }
  \newcommand{\xxxSubParagraphStar}[1]{\oldsubparagraph*{#1}\mbox{}}
  \newcommand{\xxxSubParagraphNoStar}[1]{\oldsubparagraph{#1}\mbox{}}
\fi
\makeatother

\usepackage{longtable,booktabs,array}
\usepackage{calc} 
\usepackage{etoolbox}
\makeatletter
\patchcmd\longtable{\par}{\if@noskipsec\mbox{}\fi\par}{}{}
\makeatother
\IfFileExists{footnotehyper.sty}{\usepackage{footnotehyper}}{\usepackage{footnote}}
\makesavenoteenv{longtable}
\usepackage{graphicx}
\makeatletter
\def\maxwidth{\ifdim\Gin@nat@width>\linewidth\linewidth\else\Gin@nat@width\fi}
\def\maxheight{\ifdim\Gin@nat@height>\textheight\textheight\else\Gin@nat@height\fi}
\makeatother
\setkeys{Gin}{width=\maxwidth,height=\maxheight,keepaspectratio}
\makeatletter
\def\fps@figure{htbp}
\makeatother

\makeatletter
\@ifpackageloaded{caption}{}{\usepackage{caption}}
\AtBeginDocument{%
\ifdefined\contentsname
  \renewcommand*\contentsname{Table of contents}
\else
  \newcommand\contentsname{Table of contents}
\fi
\ifdefined\listfigurename
  \renewcommand*\listfigurename{List of Figures}
\else
  \newcommand\listfigurename{List of Figures}
\fi
\ifdefined\listtablename
  \renewcommand*\listtablename{List of Tables}
\else
  \newcommand\listtablename{List of Tables}
\fi
\ifdefined\figurename
  \renewcommand*\figurename{Figure}
\else
  \newcommand\figurename{Figure}
\fi
\ifdefined\tablename
  \renewcommand*\tablename{Table}
\else
  \newcommand\tablename{Table}
\fi
}
\@ifpackageloaded{float}{}{\usepackage{float}}
\floatstyle{ruled}
\@ifundefined{c@chapter}{\newfloat{codelisting}{h}{lop}}{\newfloat{codelisting}{h}{lop}[chapter]}
\floatname{codelisting}{Listing}

\makeatother
\makeatletter
\@ifpackageloaded{caption}{}{\usepackage{caption}}
\@ifpackageloaded{subcaption}{}{\usepackage{subcaption}}
\makeatother

\ifLuaTeX
  \usepackage{selnolig}  
\fi
\usepackage[]{natbib}
\usepackage{bookmark}

\IfFileExists{xurl.sty}{\usepackage{xurl}}{} 
\hypersetup{
  pdftitle={Title},
  pdfauthor={Author 1; Author 2},
  pdfkeywords={3 to 6 keywords, that do not appear in the title},
  colorlinks=true,
  linkcolor={blue},
  filecolor={Maroon},
  citecolor={Blue},
  urlcolor={Blue},
  pdfcreator={LaTeX via pandoc}}

\newtheorem{theorem}{Theorem}[section]
\newtheorem{corollary}{Corollary}[theorem]
\newtheorem{lemma}[theorem]{Lemma}
\newtheorem{assumption}{Assumption}
\newtheorem{definition}{Definition}

\newcommand{\anon}{1}

\begin{document}

\def\spacingset#1{\renewcommand{\baselinestretch}%
{#1}\small\normalsize} \spacingset{1}


\if1\anon
{
  \title{\bf Measuring Differences between Conditional Distributions using Kernel Embeddings}
    \author[1,2]{Peter Moskvichev\thanks{Corresponding author: Peter Moskvichev \texttt{petervic001@e.ntu.edu.sg}}}
    \author[2]{Siu Lun Chau}
    \author[1,2]{Dino Sejdinovic}
    
    \affil[1]{School of Mathematical Sciences, Adelaide University}
    \affil[2]{College of Computing and Data Science, Nanyang Technological University}
  \maketitle
} \fi

\if0\anon
{
  \bigskip
  \bigskip
  \bigskip
  \begin{center}
    {\LARGE\bf Measuring Differences between Conditional Distributions using Kernel Embeddings}
\end{center}
  \medskip
} \fi

\bigskip
\begin{abstract}
Comparing conditional distributions is a fundamental challenge in statistics and machine learning, with applications across a wide range of domains. While proposed methods for measuring discrepancies using kernel embeddings of distributions in a reproducing kernel Hilbert space (RKHS) provide powerful non-parametric techniques, the existing literature remains fragmented and lacks a unified theoretical treatment. This paper addresses this gap by establishing a coherent framework for studying kernel-based methods to measure divergence between conditional distributions through what we refer to as conditional maximum mean discrepancy (CMMD). The CMMD consists of a family of metrics which we call levels, with three special cases each using a different type of RKHS embedding: CMMD$_0$ (conditional mean operators), CMMD$_1$ (conditional mean embeddings), and CMMD$_2$ (joint mean embeddings). We additionally introduce a general level $s$ CMMD, clarifying the required assumptions, and establishing mathematical connections between the levels through the lens of operator-based smoothing. In addition to reviewing previously proposed estimators, we introduce a novel doubly robust estimator for the CMMD that maintains consistency provided the product of errors of the underlying models reduces sufficiently fast. We provide numerical experiments demonstrating that the CMMD effectively captures complex conditional dependencies for statistical testing.
\end{abstract}

\noindent%
{\it Keywords:} Conditional Probability Distributions; Kernel Methods; Statistical Distance.
\vfill

\newpage
\spacingset{1.8} 

\section{Introduction}

A central problem in statistics is determining the distance between two conditional distributions and testing whether they are equal. More formally, let $X$, $Y$ and $Z$ be random variables with joint distributions $P_{XY} = P_{Y|X} \otimes P_X$ and $Q_{XZ} = Q_{Z|X} \otimes Q_X$. The goal is to determine whether the conditional relationship $P_{Y|X}$ is equivalent to $Q_{Z|X}$. This problem arises naturally in applications across statistics and machine learning. In causal inference, for example, researchers are interested in knowing if there is a significant difference in the distribution of outcomes for the treated and untreated population conditioned on certain covariates \citep{pearl_causal_2009, Park_21_codite, singh_causalfunc_2024}. On the other hand, machine learning practitioners are often concerned with the problem of covariate shifts, and ensuring that conditional relationships are consistent across training and testing environments \citep{shimodaira_improving_2000, bickel_discriminative_2009, sugiyama_covshift_2012, ma_covshift_2023}. Furthermore, this problem is studied in uncertainty quantification through the lens of calibration, where the reliability of probabilistic models requires a match between distributions of real labels and model predictions conditioned on prediction confidence \citep{Widmann_2021, marx_calibration_2023}. However, due to the difficulty of modeling conditional distributions from limited data, determining differences between them remains a challenge.

To address these complexities, kernel-based methods have emerged as a versatile and non-parametric framework, enabling new ways for both discrepancy measurement and statistical hypothesis testing. By mapping data into a high-dimensional (often infinite-dimensional) reproducing kernel Hilbert space (RKHS), these methods can capture intricate, non-linear dependencies without the restrictive assumptions of parametric models. Kernel methods have already reshaped several classical testing problems, such as two-sample testing \citep{Gretton_2012, song_generalized_2023}, independence testing \citep{gretton_hsic_2007} and goodness-of-fit testing \citep{chwialkowski_goodness_2016, key_composite_2025}. Despite this success, the literature on comparing conditional distributions specifically remains fragmented, with different assumptions on the data-generating process and problem formulations. While numerous approaches have been proposed (see Section \ref{sec:related_work}), they are typically presented in isolation, obscuring the mathematical connections between them. 

\begin{figure}[t]
    \centering
    \begin{subfigure}[b]{0.32\textwidth}
        \centering
        \includegraphics[width=\linewidth]{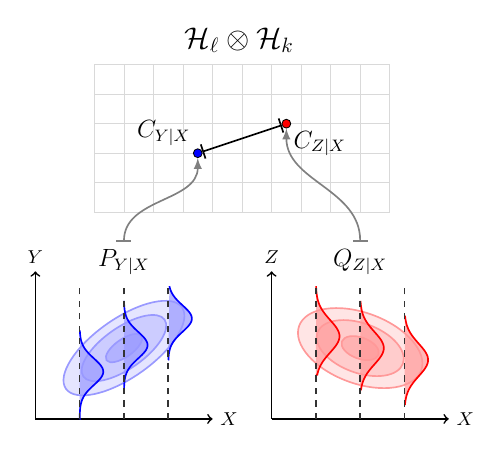}
        \caption{CMMD$_0$}
    \end{subfigure}
    \hfill 
    \begin{subfigure}[b]{0.32\textwidth}
        \centering
        \includegraphics[width=\linewidth]{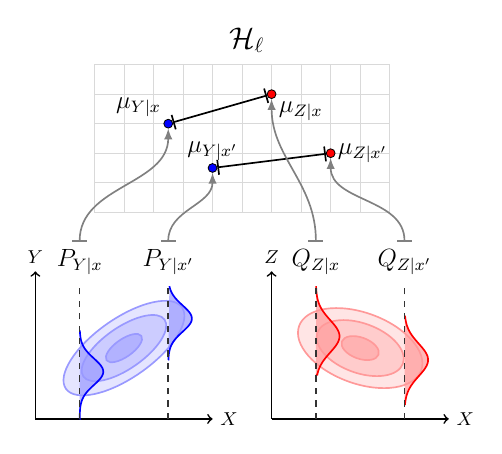}
        \caption{CMMD$_1$}
    \end{subfigure}
    \hfill
    \begin{subfigure}[b]{0.32\textwidth}
        \centering
        \includegraphics[width=\linewidth]{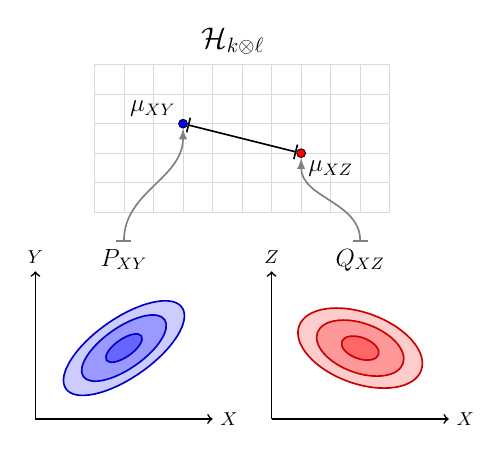}
        \caption{CMMD$_2$}
    \end{subfigure}
    
    \caption{Illustration of the three levels of CMMD using different types of RKHS embedding. (a) CMMD$_0$ compares \textbf{conditional mean operators} directly, measuring the discrepancy between the entire conditional relationships without weighting covariates by their frequency. (b) CMMD$_1$ uses \textbf{conditional mean embeddings} to measure discrepancy between conditional distributions at each value of the covariate and averages over the distribution of $X$. (c) CMMD$_2$ computes the distance between \textbf{joint mean embeddings} of the combined outcome and covariate distribution, introducing even greater smoothing by the covariate.}
    \label{fig:cmmd_illustration}
\end{figure}

We resolve this through the \emph{conditional maximum mean discrepancy} (CMMD), which comprises a family of methods for measuring the divergence between conditional distributions. It is an extension of the \emph{maximum mean discrepancy} (MMD) \citep{Gretton_2012}, which measures the distance between marginal probability distributions by representing them as points in an RKHS via the kernel mean embedding \citep{smola_hilbert_2007}. By considering different ways of embedding a conditional distribution in an RKHS, we identify three distinct notions of CMMD, illustrated in Figure \ref{fig:cmmd_illustration}, which we shall refer to as \emph{levels}.

The main contribution of this work is providing a unified perspective on kernel embedding methods for conditional distribution comparison. Although all three levels have previously been used in various contexts, here we establish explicit connections between the metrics through an operator smoothing perspective. We show that higher levels correspond to greater amounts of smoothing via the covariance operator derived from the marginal distribution of the conditioning variable. Furthermore, we also consider fractional levels of smoothing, which gives rise to a general level $s$ CMMD. An additional, secondary contribution is the introduction of a novel doubly robust estimator for CMMD which converges to the population quantity as long as the product of errors for two nuisance models (for the conditional mean embedding or the propensity score) reduces sufficiently fast. 

The paper is structured as follows. We begin with a brief overview of mathematical preliminaries and related work. The CMMD is then defined, and the theoretical properties of the different levels are analyzed in more detail. This includes new results regarding relationships between the metrics and a formulation in terms of Hilbert-Schmidt operators. Next, we introduce existing empirical estimators for CMMD based on kernel regression, as well as our novel doubly robust estimator. This is followed by a description of hypothesis testing algorithms for conducting two-sample tests. Given two sets of data with paired covariate and outcome variables, the hypothesis testing procedure allows us to determine whether the conditional relationships are equivalent. We conclude with a numerical analysis of the CMMD metrics to demonstrate the practical feasibility of the proposed methods. Our experiments show that CMMD-based test statistics are capable of discerning conditional distributions in a variety of settings, both with synthetic and real data.

\section{Background}

\subsection{Problem Setup}

Throughout this paper, we will consider the measurable space $(\Omega, \mathcal F)$ and two underlying probability measures: $P$ and $Q$. We further take measurable spaces $(\mathcal X, \mathfrak X)$, $(\mathcal Y, \mathfrak Y)$ and define random variables $X: \Omega \to \mathcal X$ and $Y, Z: \Omega \to \mathcal Y$, where the former represents the covariate (or conditioning) variable and the latter represent the outcome variables. We use subscripts to denote the distributions of random variables, e.g. $P_X$ and $Q_X$ represent the marginal distribution of $X$ under $P$ and $Q$, respectively. We assume that $P_X$ and $Q_X$ have full support on $\mathcal X$ and are absolutely continuous with respect to each other. Throughout this work, we consider their mixture $R_X = \rho P_X + (1-\rho) Q_X$ for some $\rho \in [0,1]$.

We are interested in comparing conditional distributions $P_{Y|X}$ and $Q_{Z|X}$. Equality of conditional distributions is typically understood only in an $R_X$ almost sure sense, with conditional probabilities being arbitrary on sets of measure zero. Throughout this work, we assume that the conditional distributions admit regular versions, that is, there exists a Markov kernel\footnote{Markov kernels should not be confused with reproducing kernels that are used for kernel embeddings of distributions. In this work, the word ``kernel'' in isolation refers to a reproducing kernel.} $\kappa: \mathcal X \times \mathfrak Y \to [0,1]$ representing it. Note that Markov kernels are not unique, and instead form an equivalence class for conditional distributions that are equal almost surely, not pointwise. A chosen Markov kernel for $P_{Y|X}$ is denoted $\kappa_P$, and likewise $\kappa_Q$ for $Q_{Z|X}$. For each $x \in \mathcal X$, $\kappa_P(x, \cdot)$ is a probability measure on $\mathcal Y$ and we can compute conditional expectations by $\mathbb E_{Y\sim \kappa_P}[Y |X = x] = \int_{\mathcal Y} y \, \kappa_P (x, dy)$.

\subsection{Kernel Embeddings of Distributions} \label{sec:kernel_emb}

We provide a short review of some preliminaries on kernel embeddings of probability distributions in a reproducing kernel Hilbert space (RKHS), see \cite{muandetKernelMeanEmbedding2017a} for further details. Consider a random variable $X$ taking values in $\mathcal{X}$ with law $P_X$. Let $k:\mathcal{X} \times \mathcal{X} \to \mathbb{R}$ be any positive definite function with associated RKHS $\mathcal{H}_k$. The \emph{kernel mean embedding} of $P_X$ is given by $\mu_{X} = \mathbb{E}_X k(X, \cdot) \in \mathcal{H}_k$ and is well defined whenever $\mathbb{E}_X \sqrt{k(X,X)} < \infty$ \citep{smola_hilbert_2007}. The kernel mean embedding satisfies that for any $f \in \mathcal{H}_k$, we get $\langle \mu_{X}, f \rangle_{\mathcal{H}_k} = \mathbb{E}_X f(X)$, which follows from the reproducing property of $k$. Provided data $\{x_i\}_{i=1}^n$ sampled identically and independently (iid) from $P_X$, an empirical estimator for the kernel mean embedding is $\hat \mu_X = \frac{1}{n}\sum_{i=1}^n k(\cdot, x_i)$.

Given another distribution $P_{X'}$ with embedding $\mu_{X'}$, we can compute the distance between $P_X$ and $P_{X'}$ using the \emph{maximum mean discrepancy} (MMD)
\begin{equation} \label{eq:mmd}
    \text{MMD}(P_X, P_{X'}) = \| \mu_{X} - \mu_{X'} \|_{\mathcal{H}_k}.
\end{equation}
The MMD forms a pseudometric on the space of probability measures over $\mathcal X$. However, when using a \emph{characteristic} kernel (for which mean embeddings are injective), the MMD becomes a proper metric and equals zero if and only if the distributions $P_X$ and $P_{X'}$ are equal \citep{Fukumizu_2007}. Characteristic kernels such as the Gaussian and Matérn are frequently used in practice. The squared MMD can be expanded and expressed solely in terms of expectations, allowing for straightforward estimation provided samples from $P_X$ and $P_{X'}$\citep{Gretton_2012}. For this reason, kernel mean embeddings and the MMD have been used to tackle a broad range of tasks from hypothesis testing~\citep{Gretton_2012} to parameter estimation~\citep{briol2019statistical,Cherief-Abdellatif2020_MMDBayes}, causal inference~\citep{muandet2021counterfactual,Sejdinovic2024,zhang2026instrumental}, feature attribution~\citep{chau2022rkhs,chau2023explaining}, and learning on distributions~\citep{Muandet12:SMM,szabo2016learning}.

By choosing a kernel on the product domain $\mathcal{X} \times \mathcal{Y}$, we can likewise define the kernel mean embedding of a joint distribution $P_{XY}$. Given kernels $k$ and $\ell$ on $\mathcal{X}$ and $\mathcal{Y}$ respectively, a common choice is the tensor product kernel $k \otimes \ell$, which can be evaluated by $(k \otimes \ell)((x,y),(x',y')) = k(x,x')\ell(y,y')$. This leads to the \emph{joint mean embedding} $\mu_{X Y} = \mathbb{E}_{XY} [k(\cdot, X) \otimes \ell(\cdot, Y)] \in \mathcal{H}_{k\otimes \ell}$ \citep{Fukumizu_2004}. Using iid data $\{(x_i, y_i)\}_{i=1}^n$ sampled from $P_{XY}$, the joint mean embedding can be estimated in a similar manner to the standard kernel mean embedding, $\hat \mu_{XY} = \frac{1}{n} \sum_{i=1}^n k(\cdot, x_i) \otimes \ell(\cdot, y_i)$. By isometry between $\mathcal{H}_{k\otimes \ell}$ and $\mathcal{H}_{k}\otimes \mathcal{H}_\ell$ (where $\mathcal H_\ell$ is the RKHS associated with $\ell$) the joint mean embedding $\mu_{XY}$ can be identified with the uncentered cross-covariance operator $C_{XY}: \mathcal{H}_\ell \to \mathcal{H}_k$, which has the property $\langle f , C_{XY}g \rangle_{\mathcal{H}_k} = \mathbb{E}_{XY}[f(X) g(Y)]$. Of course, the covariance operator $C_{XX} = \mathbb{E}_{X} [k(\cdot, X) \otimes k(\cdot, X)] \in \mathcal{H}_{k} \otimes \mathcal{H}_{k}$ can be defined analogously. 

The tensor product space $\mathcal{H}_{k}\otimes \mathcal{H}_\ell$ is also isomorphic to the space of Hilbert-Schmidt operators mapping from $\mathcal{H}_\ell$ to $\mathcal{H}_k$. The inner product between operators can be expressed via $\langle A, B \rangle_{\mathcal{H}_{k}\otimes \mathcal{H}_\ell} = Tr(A^* B)$ where $^*$ indicates the adjoint of an operator. This leads to the notion of Hilbert-Schmidt norm, which we denote by $\| \cdot \|_{\mathcal{H}_{k}\otimes \mathcal{H}_\ell}$. Given orthonormal bases $\{\phi_i\}_{i \in I}$ and $\{\psi_j\}_{j \in J}$ for $\mathcal H_k$ and $\mathcal H_\ell$ respectively, the squared Hilbert-Schmidt norm of an operator $A : \mathcal H_\ell \to \mathcal H_k$ can be expressed by $\| A \|^2_{\mathcal{H}_{k}\otimes \mathcal{H}_\ell} = \sum_{j \in J} \| A \psi_j \|^2_{\mathcal H_k} = \sum_{i \in I, j \in J} | \langle A \psi_j, \phi_i \rangle_{\mathcal H_k} |^2$, or equivalently $\| A \|^2_{\mathcal{H}_{k}\otimes \mathcal{H}_\ell} = Tr(A^* A)$. An operator is called Hilbert-Schmidt if it has finite Hilbert-Schmidt norm. In finite dimensions, the Hilbert-Schmidt norm is equivalent to the Frobenius norm for matrices.

\subsection{Conditional Mean Embeddings and Operators} \label{sec:back_cmo}

When dealing with a conditional distribution $P_{Y|X}$, we can similarly embed it in an RKHS. Taking the Markov kernel $\kappa_P$, the \emph{conditional mean embedding} (CME) at $x \in \mathcal X$ is defined as $\mu_{Y|x} = \mathbb{E}_{Y\sim \kappa_P} [\ell(\cdot, Y) | X=x] = \int_{\mathcal Y} \ell(\cdot, y) \kappa_P (x, dy) \in \mathcal{H}_\ell$ \citep{Song_2009, park_measure_2020}. In analogy with mean embeddings of marginal distributions, the CME can recover conditional expectations: for any $g \in \mathcal H_\ell$, $\langle \mu_{Y|x}, g \rangle_{\mathcal H_\ell} = \mathbb E_{Y\sim \kappa_P}[g(Y) | X=x]$.

However, CMEs are not the only way to represent conditional distributions in an RKHS. Suppose that within the equivalence class of Markov kernels corresponding to $P_{Y|X}$, there exists a \emph{distinguished} Markov kernel $\tilde{\kappa}_P$ such that for all $g \in \mathcal H_\ell$, the mapping $x \mapsto \int_{\mathcal Y} g(y) \tilde{\kappa}_P (x, dy)$ is an element of $\mathcal H_k$. Intuitively, this means that $\mathcal H_k$ is a sufficiently rich class of functions relative to the choice of $\mathcal H_\ell$ and the smoothness of $\tilde{\kappa}_P$ in $x$. Throughout this paper, we consider only the class of conditional distributions that have such a Markov kernel for a chosen RKHS $\mathcal H_k$. We are then able to define an operator $C_{Y|X}^* : \mathcal H_\ell \to \mathcal H_k$ such that $(C_{Y|X}^* g)(x) = \mathbb E_{Y\sim \tilde{\kappa}_P}[g(Y) | X=x] =\int_{\mathcal Y} g(y) \tilde{\kappa}_P(x, dy)$. The adjoint, $C_{Y|X}: \mathcal H_k \to \mathcal H_\ell$, is called the \emph{conditional mean operator} (CMO). If the CME is defined in terms of the distinguished Markov kernel $\tilde{\kappa}_P$, then the CMO satisfies $\mu_{Y|x} = C_{Y|X}k(\cdot, x)$ for all $x \in \mathcal X$ \citep{Song_2009}. Under additional smoothness assumptions, $C_{Y|X}$ is a Hilbert-Schmidt operator (see Theorem \ref{thm:cmo_hs} for details) and can be considered as the kernel embedding of $P_{Y|X}$ in the RKHS $\mathcal H_\ell \otimes \mathcal H_k$. From here on, we assume that conditional expectations are always taken with respect to distinguished Markov kernels, allowing us to write them unambiguously without a subscript, e.g. $\mathbb E [Y|X=x]$. Likewise, CMEs will also be defined in terms of distinguished Markov kernels.

Since in general we do not observe many outcomes coming from the same covariate\footnote{Unless we are in distribution regression settings, see \citet{law2018variational} and \citet{chau2021deconditional}.}, this motivates a regression approach for estimating kernel embeddings of conditional distributions. Given samples $\{(x_i, y_i)\}_{i=1}^n \overset{iid}{\sim} P_{XY}$, an empirical estimator of the CMO is 
\begin{equation} \label{eq:cmo_est}
    \hat{C}_{Y|X} = \Psi_\mathbf{Y} (K_\mathbf{XX} + \lambda n I_n)^{-1} \Phi_\mathbf{X}^*
\end{equation}
where $\Psi_\mathbf{Y}: \mathbb R^n \to \mathcal H_\ell$ and $\Phi_\mathbf{X}: \mathbb R^n \to \mathcal H_k$ are operators which can be expressed as $\Psi_\mathbf{Y} = [\ell(\cdot, y_1), \dots, \ell( \cdot, y_n)]$, $\Phi_\mathbf{X} = [k(\cdot, x_1), \dots, k(\cdot, x_n)]$, $[K_\mathbf{XX}]_{ij} = k(x_i, x_j)$, $I_n$ is the $n\times n$ identity matrix and $\lambda$ is a regularization parameter \citep{Song_2010}. $K_\mathbf{XX}$ is known as the Gram matrix, and can be expressed as $K_\mathbf{XX} = \Phi_\mathbf{X}^* \Phi_\mathbf{X}$. We can, in turn, define an estimator for the CME evaluated at a particular point as $\hat \mu_{Y|x} = \hat{C}_{Y|X} k(\cdot, x)$, or equivalently
\begin{equation} \label{eq:cme_est}
    \hat \mu_{Y|x} = \Psi_\mathbf{Y} (K_\mathbf{XX} + \lambda n I_n)^{-1} K_{\mathbf{X}x}
\end{equation}
where $K_{\mathbf{X}x} = [k(x_1, x), \dots, k(x_n, x)]^\top \in \mathbb R^n$. Assuming that the parameter $\lambda$ is sufficiently reduced with increasing sample size, e.g. $\lambda = O(n^{-\frac{1}{4}})$, the CMO and CME estimators converge to the population quantities \citep{Song_2009, Song_2010}. See Appendix \ref{sec:app_cmo} for further details. An alternative perspective on the CMO estimator (\ref{eq:cmo_est}) is as the solution to an RKHS valued regression \citep{grunewalder_cmereg_2012}, leading to cross-validation as one approach to compute $\lambda$ \citep{craven_gcv_1979, singh_causalfunc_2024}. Bayesian approaches to estimating (\ref{eq:cmo_est}) also naturally lead to marginal likelihood-based optimization to select $\lambda$, see \citet{chau2021bayesimp} and \citet{dance2024interventional}.

\subsection{Related Work} \label{sec:related_work}

Testing for differences and measuring discrepancies between conditional distributions has received much attention in recent years. For example, \cite{Hu_conditional_conformal_2024} test for equality between conditional distributions using conformal prediction methods, while \citet{boeken_bayes_cond_two_samp_2021} use Bayesian methods. Furthermore, as pointed out by \citet{lee_general_2024}, conditional independence testing is equivalent to conditional two-sample testing, meaning techniques developed for the former problem can be applied for testing equality of conditional distributions. Testing for local differences in conditional distributions using kernel ridge regression was developed by \citet{massiani_kernel_2025}. However, our focus is specifically on methods that use kernel embeddings to provide a global measure of statistical discrepancy.

Using joint mean embeddings to test for conditional distribution equality via density ratio estimation was studied by \citet{lee_general_2024}. To test conditional goodness of fit, where marginal distributions of covariates are assumed to be the same, \citet{Glaser_24_ACMMD} likewise used joint mean embeddings. In the context of uncertainty calibration, \citet{Widmann_2021} used a similar approach. Conditional mean embeddings have been applied for two-sample conditional testing by \cite{chatterjee_kernel-based_2024} and \citet{yan_distance_2024}, albeit with different estimators from each other and to the estimator we introduce in Section \ref{sec:estimation}. CMEs were also used by \citet{Park_21_codite} to measure the conditional distribution treatment effect in a causal inference setting. Although conditional mean operators have not been applied for testing equality of conditional distributions, several authors, including \citet{Ren_2016} and \citet{moskvichev_ckce_2025}, have used them for measuring discrepancy. In terms of comparison of the different kernel discrepancies, \citet{huang_evaluating_2022} establish some basic connections which we expand upon in Section \ref{sec:cmmd_rel}. In concurrent work, \citet{jain2026conditionaldistributionaltreatmenteffects} apply CMMD-like metrics for testing for conditional distribution treatment effects. They propose using a smoothing operator, but without explicit links to taking powers of the covariance operator $C_{XX}$ for data dependent smoothing which we investigate in our work.

\section{MMD between Conditional Distributions} \label{sec:cmmd}

We are now ready to introduce the \emph{conditional maximum mean discrepancy} (CMMD) for measuring and testing differences between conditional distributions. The term CMMD was first used by \cite{Ren_2016} to describe the metric using conditional mean operators. In this work, we use CMMD to refer to all kernel-based measures of difference between conditional distributions, and use subscripts to refer to the different levels under consideration. Throughout this section, we will take kernels $k:\mathcal X \times \mathcal X \to \mathbb{R}$ and $\ell:\mathcal Y \times \mathcal Y \to \mathbb{R}$ with associated RKHS $\mathcal H_{k}$ and $\mathcal H_{\ell}$. We also require the following assumptions.
\begin{assumption} \label{asm:full_support}
    $P_X$ and $Q_X$ have full support on $\mathcal X$ and are absolutely continuous with respect to each other. 
\end{assumption}
\begin{assumption} \label{asm:regular}
    The conditional distributions $P_{Y|X}$ and $Q_{Z|X}$ admit regular versions, that is, they correspond to equivalence classes of Markov kernels $\kappa_P: \mathcal X \times \mathfrak Y \to [0,1]$ and $\kappa_Q: \mathcal X \times \mathfrak Y \to [0,1]$.
\end{assumption}
\begin{assumption} \label{asm:k_bounded}
    The kernel $k$ is continuous and bounded, that is, there exists $k_{\max} > 0$ such that $k(x, x) \leq k_{\max}$ for all $x \in \mathcal X$. 
\end{assumption}
\begin{assumption} \label{asm:ell_char}
    The kernel $\ell$ is characteristic, with $\mathbb E_{X\sim R_X}\mathbb E[\ell(Y,Y) | X]<\infty$ and $\mathbb E_{X\sim R_X}\mathbb E[\ell(Z,Z) | X]<\infty$.
\end{assumption}
\begin{assumption}\label{asm:cmo_hs} 
    Within the equivalence class of Markov kernels for $P_{Y|X}$, there exists a representative $\tilde{\kappa}_P$ such that for any $g \in \mathcal H_\ell$, the map $x \mapsto \mathbb E [g(Y)|X=x] = \int_{\mathcal Y} g(y) \tilde{\kappa}_P(x,dy) \in \textup{range}(C_{XX}^\gamma)$ for some $\gamma \geq \frac{1}{2}$ where $C_{XX}$ is the covariance operator for $R_X$. An analogous condition holds for $Q_{Z|X}$. All CMEs and CMOs are defined with respect to these distinguished Markov kernels. 
\end{assumption}

By Assumption \ref{asm:full_support}, almost sure equality on $P_X$, $Q_X$ and $R_X$ are equivalent. It is also needed in order for the covariance operator $C_{XX}$ to have desired smoothness properties. Assumption \ref{asm:regular} is required to be able to represent pointwise conditional distributions and kernel embeddings in an RKHS rigorously, rather than always treating them as random objects. Assumptions \ref{asm:k_bounded} and \ref{asm:ell_char} are satisfied by many commonly used kernels, such as the Gaussian or Laplacian kernel. Continuity of $k$ also ensures that the RKHS $\mathcal H_k$ is separable. Assumption \ref{asm:cmo_hs} is stricter than the typical condition required for the so-called \emph{well-specified} setting $\mathbb E [g(Y)|X=\cdot] \in \mathcal H_k$ \citep{Fukumizu_2004, Song_2009}, but is sufficient for the CMO $C_{Y|X}$ to be Hilbert-Schmidt. Intuitively, it requires that the conditional distribution is sufficiently smooth relative to the marginal distribution of $X$ and the kernel $k$. However, we note that this assumption may be difficult to validate in practice. Alternative assumptions are given by \citet{klebanov_rigorous_2020}, who provide further detail on kernel embeddings of conditional distributions as linear operators in a Hilbert space.

\subsection{Level 0 CMMD}

Let $C_{Y|X}$ and $C_{Z|X}$ be the conditional mean operators corresponding to the distributions $P_{Y|X}$ and $Q_{Z|X}$, respectively. Then a measure of discrepancy between conditional distributions is given by 
\begin{equation}
    \text{CMMD}_0(P_{Y|X}, Q_{Z|X}) = \| C_{Y|X} - C_{Z|X} \|_{\mathcal H_\ell \otimes \mathcal H_k}.
\end{equation}
CMMD$_0$ has been used by \cite{Ren_2016} and \cite{huang_evaluating_2022} for training generative models. As proposed by \citet[Theorem 3]{Ren_2016}, if the CMOs are equal, then $P_{Y|X} = Q_{Z|X}$, $P_X$ almost surely. By considering the assumptions above, we extend the result by including the reverse implication. This means that CMMD$_0$ is a valid metric for conditional probability distributions. 

\begin{theorem} \label{thm:cmmd0_metric}
    Let $P_{Y|X}$, $Q_{Z|X}$ be two conditional distributions and suppose that Assumptions \ref{asm:full_support}–\ref{asm:cmo_hs} hold. Then $P_{Y|X} = Q_{Z|X}$ almost surely if and only if $\| C_{Y|X} - C_{Z|X}\|_{\mathcal H_\ell \otimes \mathcal H_k} = 0$.
\end{theorem}

For ease of notation, we define the operator $\Delta^*: \mathcal H_\ell \to \mathcal H_k$ such that $(\Delta^* g)(x) = \int_{\mathcal Y} g(y) (\tilde{\kappa}_P(x, dy) - \tilde{\kappa}_Q(x, dy))$ whose adjoint (under Assumption \ref{asm:cmo_hs}) is simply the difference of CMOs $\Delta = C_{Y|X} - C_{Z|X}$. Then the squared CMMD$_0$ can be expressed as
\begin{equation}
    \text{CMMD}_0^2 (P_{Y|X}, Q_{Z|X}) = Tr(\Delta^* \Delta)
\end{equation}
which follows directly from the definition of Hilbert-Schmidt norm given in Section \ref{sec:kernel_emb}.

\subsection{Level 1 CMMD}

While CMMD$_0$ provides a direct comparison between conditional distributions, it does not account for which values of the covariate are more likely to occur. Perhaps $P_{Y|X}$ and $Q_{Z|X}$ are different only in regions of low probability which should be weighted less. This motivates the use of conditional mean embeddings and averaging over $X$. Let $\mu_{Y|X}$ and $\mu_{Z|X}$ be the conditional mean embeddings corresponding to the distributions $P_{Y|X}$ and $Q_{Z|X}$. Then 
\begin{equation} \label{eq:cmmd0}
    \text{CMMD}_1(P_{Y|X}, Q_{Z|X}) = \sqrt{\mathbb{E}_X \|\mu_{Y|X} - \mu_{Z|X}\|_{\mathcal H_\ell}^2}
\end{equation}
is the level 1 conditional maximum mean discrepancy. If $P_X \neq Q_X$, then the expectation is taken over the mixture $R_X = \rho P_X + (1-\rho) Q_X$ for some $\rho\in[0,1]$. The CMMD$_1$ has been used for conditional two-sample testing \citep{chatterjee_kernel-based_2024, yan_distance_2024} and causal inference \citep{Park_21_codite}. Under Assumptions \ref{asm:full_support}–\ref{asm:cmo_hs}, CMMD$_1$ provides a valid metric between conditional distributions. This result is proven by e.g. \citet[Theorem 5.2]{park_measure_2020} and \citet[Proposition 2.3]{chatterjee_kernel-based_2024} so we will omit it from this work. Similar to CMMD$_0$, it is easier to work with the squared quantity.

\begin{theorem} \label{thm:cmmd1}
    Suppose that Assumption \ref{asm:cmo_hs} holds. Then the squared \textup{CMMD}$_1$ can be expressed as 
    \begin{equation} \label{eq:cmmd_1_squared}
        \textup{CMMD}_1^2(P_{Y|X}, Q_{Z|X}) = Tr(\Delta^* \Delta C_{XX})
    \end{equation}
    where $C_{XX}$ is the covariance operator for the distribution $R_X$. 
\end{theorem}

Although the CMMD$_1$ is defined in terms of CMEs rather than CMOs, Assumption \ref{asm:cmo_hs} allows us to express $\mu_{Y|x} = C_{Y|X} k(\cdot, x)$ and $\mu_{Z|x} = C_{Z|X} k(\cdot, x)$, providing the required connection for Theorem \ref{thm:cmmd1} to hold. We choose to write CMMD$_1$ in terms of $\Delta$ and $C_{XX}$ in order to form an explicit relationship between the different levels of CMMD.

\subsection{Level 2 CMMD}

The final level, CMMD$_2$, compares the distance between the joint mean embeddings of $P_{XY}$ and $Q_{XZ}$. In the case when $P_X \neq Q_X$, we once again consider the mixture distribution $R_X = \rho P_X + (1-\rho) Q_X$ for some $\rho\in[0,1]$ and $\tilde P_{XY} = P_{Y|X} \otimes R_X$ and $\tilde Q_{XZ} = Q_{Z|X} \otimes R_X$. Let $\mu_{XY}$ and $\mu_{XZ}$ be the joint mean embeddings corresponding to the distributions $\tilde P_{XY}$ and $\tilde Q_{XZ}$. Then 
\begin{equation}
    \text{CMMD}_2(P_{Y|X}, Q_{Z|X}) = \|\mu_{XY} - \mu_{XZ}\|_{\mathcal H_{k \otimes \ell}}.
\end{equation}
The CMMD$_2$ may be interpreted as the MMD between joint distributions. When $k \otimes \ell$ is characteristic, the CMMD$_2$ provides a valid metric on the space of probability measures over $\mathcal X \times \mathcal Y$, which in turn measures the discrepancy between $P_{Y|X}$ and $Q_{Z|X}$ as we have a shared marginal $R_X$ \citep{Fukumizu_2004}. For example, if $k$ and $\ell$ are characteristic, continuous, bounded and translation-invariant kernels on $\mathbb R^d$, then $k \otimes \ell$ is characteristic \citep{szabo_charact_2018}. 

Variants of CMMD$_2$ have been applied for conditional two sample testing \citep{lee_general_2024}, conditional goodness of fit tests \citep{Glaser_24_ACMMD} and training generative models \citep{huang_evaluating_2022}. Just as with the previous levels, CMMD$_2$ has an explicit connection with the operators $\Delta$ and $C_{XX}$. 
\begin{theorem} \label{thm:cmmd2}
    Suppose that Assumption \ref{asm:cmo_hs} holds. Then the squared \textup{CMMD}$_2$ can be expressed as 
    \begin{equation}
        \textup{CMMD}_2^2(P_{Y|X}, Q_{Z|X}) = Tr(\Delta^* \Delta C_{XX}^2).
    \end{equation}
\end{theorem}
Once more, Assumption \ref{asm:cmo_hs} is required to ensure that CMO operators exist and are Hilbert-Schmidt. The proof of Theorem \ref{thm:cmmd2} also relies on being able to express cross-covariance operators as $C_{YX} = C_{Y|X} C_{XX}$, which is known as kernel chain rule \citep{Song_2009}.

\begin{table}[t]
    \centering
    \caption{Summary of kernel-based metrics for measuring discrepancy between conditionals.}
    \begin{tabular}{cc}
    \toprule
       CMMD$_0^2$ : & $\| C_{Y|X} - C_{Z|X} \|_{\mathcal H_\ell \otimes \mathcal H_k}^2 = Tr(\Delta^* \Delta) \qquad$ \\
       CMMD$_1^2$ : & $\mathbb{E}_{X}\| \mu_{Y|X} - \mu_{Z|X} \|_{\mathcal H_\ell}^2 = Tr(\Delta^* \Delta C_{XX})$  \\
       CMMD$_2^2$ : & $\| \mu_{XY} - \mu_{XZ} \|_{\mathcal H_k \otimes \mathcal H_\ell}^2 = Tr(\Delta^* \Delta C_{XX}^2)$ \\
    \bottomrule
    \end{tabular}
    \label{tab:summary}
\end{table}

\subsection{Relationship between CMMD Metrics} \label{sec:cmmd_rel}

The interaction between $\Delta^* \Delta$ and different powers of the covariance operator $C_{XX}$ motivates the notion of three levels of CMMD, where of course squared CMMD$_0$ may be interpreted as $Tr(\Delta^* \Delta C_{XX}^0)$. A summary of the three metrics expressed in terms of these operators is given in Table \ref{tab:summary}. However, we are not restricted to integer powers, allowing us to define the level $s$ squared CMMD as
\begin{equation} \label{eq:cmmd_s}
    \text{CMMD}_s^2(P_{Y|X}, Q_{Z|X}) = Tr(\Delta^* \Delta C_{XX}^s) = \sum_i \sigma_i^s \| \Delta\phi_i \|_{\mathcal H_\ell}^{2}
\end{equation}
where $s \geq 0$ and we have applied the spectral representation of the covariance operator $C_{XX} = \sum_i \sigma_i \phi_i \otimes \phi_i$, with $\{ \phi_i\}$ an orthonormal basis for $\mathcal H_k$ and $\sigma_i$ the corresponding eigenvalues. Although it is difficult to interpret a general value of $s$ in terms of kernel embeddings, intuitively higher levels correspond to greater amounts of smoothing caused by the marginal distribution of $X$. While smoothing may obscure differences between conditional relationships in some settings, in other cases it may be more relevant to highlight discrepancies in regions with high density of covariates. 

We next highlight some relationships between the CMMD metrics, starting with a result proposed by \citet{huang_evaluating_2022}.
\begin{theorem} \label{thm:cmmd_rel}
    For conditional distributions $P_{Y|X}$ and $Q_{Z|X}$, with distribution of covariates $R_X$, the following inequalities hold:
    \begin{align*}
        \textup{CMMD}_2^2(P_{Y|X}, Q_{Z|X}) &\leq \mathbb E_X[k(X,X)] \textup{CMMD}_1^2(P_{Y|X}, Q_{Z|X}), \\
        \textup{CMMD}_1^2(P_{Y|X}, Q_{Z|X}) &\leq \mathbb E_X[k(X,X)] \textup{CMMD}_0^2(P_{Y|X}, Q_{Z|X}).
    \end{align*}
\end{theorem}

From Theorem \ref{thm:cmmd_rel}, the corollary below immediately follows. 
\begin{corollary} \label{cor:cmmd_rel}
    If the kernel $k$ is chosen such that $k(x, x) = 1$ for all $x \in \mathcal X$, then
    \begin{equation*}
        \textup{CMMD}_2^2(P_{Y|X}, Q_{Z|X}) \leq \textup{CMMD}_1^2(P_{Y|X}, Q_{Z|X}) \leq \textup{CMMD}_0^2(P_{Y|X}, Q_{Z|X}).
    \end{equation*}
\end{corollary}
Many frequently used kernel functions, such as the Gaussian or Laplacian, fit the conditions in Corollary \ref{cor:cmmd_rel}. However, it is possible to make the bound in Theorem \ref{thm:cmmd_rel} tighter.

\begin{theorem} \label{thm:cmmd_rel_sigma}
    Let $\sigma_{\max}$ be the largest eigenvalue of the covariance operator $C_{XX}$. Then for $s\geq s'\geq 0$, we have
    \begin{equation*}
        \textup{CMMD}_s^2(P_{Y|X}, Q_{Z|X}) \leq \sigma_{\max}^{s-s'} \textup{CMMD}_{s'}^2(P_{Y|X}, Q_{Z|X}).
    \end{equation*}
\end{theorem}

These results establish a hierarchy among the CMMD metrics, showing that higher levels define progressively weaker notions of conditional discrepancy. In particular, convergence in CMMD$_0$ implies convergence in CMMD$_1$ and CMMD$_2$, allowing convergence results to transfer automatically across levels. The bounds also reveal the smoothing effect induced by the marginal distribution $R_X$. Inserting powers of $C_{XX}$ down-weights discrepancies in directions poorly supported by $R_X$. While CMMD$_0$ is independent of the distribution of covariates, higher level metrics are affected. The eigenvalue $\sigma_{\max}$ in Theorem \ref{thm:cmmd_rel_sigma} quantifies how strongly $R_X$ can amplify discrepancies between conditional embeddings, with smaller $\sigma_{\max}$ indicating tighter control between CMMD levels. To more intuitively highlight the differences between the CMMD metrics, we provide a toy example based on model selection in Appendix \ref{sec:toy_ex}.

\section{Estimation} \label{sec:estimation}

As mentioned previously, the different levels of the CMMD have been used by authors in various settings. However, the estimators used for the population quantities are not always the same. In Section \ref{sec:naive_est} we provide simple estimators for the CMMD, and refer to other papers for alternatives. A novel doubly robust estimator is introduced in Section \ref{sec:doubly_robust}. These estimators will form the basis of the test statistic employed in our hypothesis testing algorithms in Section \ref{sec:hyp_test}.

\subsection{Naive Estimators} \label{sec:naive_est}

Suppose we have available i.i.d. samples $\{(x_i, y_i)\}_{i=1}^n \sim P_{XY}$ and $\{(x_j', z_j)\}_{j=1}^m \sim Q_{XZ}$. From these, we are able to estimate each of the CMO, CME and JME, and in turn express the squared CMMD in closed form. Starting with the level 0 metric, an empirical estimator for the squared CMMD$_0$ is
\begin{equation} \label{eq:cmmd0_est}
    \widehat{\text{CMMD}}_0^2 = \| \hat{C}_{Y|X} - \hat{C}_{Z|X} \|_{\mathcal H_\ell \otimes \mathcal H_k}^2
\end{equation}
with CMOs estimated as in (\ref{eq:cmo_est}) using regularization parameters $\lambda_p$ and $\lambda_q$ respectively. 

\begin{lemma} \label{lem:cmmd0_conv}
    Suppose that Assumptions \ref{asm:k_bounded} and \ref{asm:cmo_hs} hold, and that the regularization terms satisfy $\lambda_p \to 0$ and $n\lambda_p^3 \to \infty$ and likewise for $\lambda_q$. Then $\widehat{\textup{CMMD}}_0^2 \overset{p}{\to} \textup{CMMD}_0^2(P_{Y|X}, Q_{Z|X})$.
\end{lemma}

As usual with kernel methods, CMMD estimators can be expressed in closed form via Gram matrices. For the squared CMMD$_0$ estimator in equation \ref{eq:cmmd0_est}, it is as follows:
\begin{lemma} \label{lem:cmmd0_est_closed}
We have
\begin{equation}
    \widehat{\textup{CMMD}}_0^2 = Tr(W_\mathbf{X} L_\mathbf{YY} W_\mathbf{X} K_\mathbf{XX}) - 2Tr(W_\mathbf{X} L_\mathbf{YZ} W_\mathbf{X'} K_\mathbf{X' X}) + Tr(W_\mathbf{X'} L_\mathbf{ZZ} W_\mathbf{X'} K_\mathbf{X' X'})
\end{equation}
where $K_\mathbf{XX}$, $K_\mathbf{X' X}$ and $K_\mathbf{X' X'}$ are defined as in Section \ref{sec:back_cmo}, $[L_\mathbf{YY}]_{ij} = \ell(y_i, y_j)$ and similar for $L_\mathbf{ZZ}$ and $L_\mathbf{YZ}$, $W_\mathbf{X} = (K_\mathbf{XX}+\lambda_p n I_n)^{-1}$ and $W_\mathbf{X'} = (K_\mathbf{X' X'}+\lambda_q m I_m)^{-1}$. 
\end{lemma}

Next, we define a simple plug-in estimator for the squared CMMD$_1$. This requires consideration of the parameter $\rho$ used in the mixture $R_X = \rho P_X + (1-\rho)Q_X$ when estimating the covariance operator $C_{XX}$. In general, the estimator can be expressed as
\begin{equation} \label{eq:cov_op_est}
    \hat C_{XX}^R = \frac{\rho}{n}\sum_{i=1}^n k(\cdot, x_i) \otimes k(\cdot, x_i) + \frac{1-\rho}{m}\sum_{j=1}^m k(\cdot, x_j') \otimes k(\cdot, x_j').
\end{equation}
In this paper, for simplicity, we concatenate the covariates to form the sample set $\{\tilde x_i \}_{i=1}^{n+m} = \{x_i\}_{i=1}^n \cup \{x_j'\}_{j=1}^m$ such that $\tilde x_i = x_i$ for $i=1,\dots,n$ and $\tilde x_i = x_{i-n}'$ for $i=n+1,\dots,n+m$ and use the estimator $\hat C_{\tilde X \tilde X} = \frac{1}{n+m} \sum_{i=1}^{n+m} k(\cdot, \tilde x_i) \otimes k(\cdot, \tilde x_i)$. This is equivalent to setting $\rho = \frac{n}{n+m}$ in equation (\ref{eq:cov_op_est}). We then define the estimator
\begin{equation} \label{eq:cmmd1_est}
    \widehat{\text{CMMD}}_1^2 = Tr(\hat \Delta^* \hat \Delta \hat C_{\tilde X \tilde X}) = \frac{1}{n+m} \sum_{i=1}^{n+m} \|\hat \mu_{Y|\tilde x_i} - \hat \mu_{Z|\tilde x_i} \|_{\mathcal H_\ell}^2
\end{equation}
where $\hat \Delta = \hat C_{Y|X} - \hat C_{Z|X}$, using standard CMO estimators computed from $\{(x_i, y_i)\}_{i=1}^n$ and $\{(x_j', z_j)\}_{j=1}^m$ respectively. As before, the estimator can be expressed in closed form.
\begin{lemma} \label{lem:cmmd1_est_closed}
We have
\begin{multline}
    \widehat{\textup{CMMD}}_1^2 = \frac{1}{n+m}[Tr(W_\mathbf{X} L_\mathbf{YY} W_\mathbf{X} K_{\mathbf{X \tilde X}} K_{\mathbf{\tilde X X}}) - 2Tr(W_\mathbf{X} L_\mathbf{YZ} W_\mathbf{X'} K_{\mathbf{X' \tilde X}} K_{\mathbf{\tilde X X}}) \\
     + Tr(W_\mathbf{X'} L_\mathbf{ZZ} W_\mathbf{X'} K_{\mathbf{X' \tilde X}} K_{\mathbf{\tilde X X'}})]
\end{multline}
where $[K_{\mathbf{X \tilde X}}]_{1\leq i \leq n, 1 \leq j \leq n+m} = k(x_i, \tilde x_j)$, and similarly for $K_{\mathbf{X' \tilde X}}$, $K_{\mathbf{\tilde X X}}$ and $K_{\mathbf{\tilde X X'}}$. 
\end{lemma}
The estimator (\ref{eq:cmmd1_est}) was introduced by \citet{Park_21_codite} in a causal inference setting, and the authors prove that it converges to the population quantity. Other options for estimation include a $K$ nearest neighbor approach \citep{chatterjee_kernel-based_2024}, and conditional U-statistics with kernel smoothing \citep{yan_distance_2024}.

For estimating CMMD$_2^2$, we consider two options. If $P_X = Q_X$, it is recommended to use the standard MMD estimator between the joint distributions $P_{XY}$ and $Q_{XZ}$
\begin{equation} \label{eq:jmmd_est}
    \widehat{\text{MMD}}^2 = \|\hat \mu_{XY} - \hat \mu_{XZ} \|_{\mathcal H_{k \otimes \ell}}^2
\end{equation}
where $\hat \mu_{XY} = \frac{1}{n} \sum_{i=1}^n k(\cdot, x_i) \otimes \ell(\cdot, y_i)$ and $\hat \mu_{XZ} = \frac{1}{m} \sum_{j=1}^m k(\cdot, x_j') \otimes \ell(\cdot, z_j)$. Equation (\ref{eq:jmmd_est}) converges to $\text{MMD}^2(P_{XY}, Q_{XZ}) = \text{CMMD}_2^2(P_{Y|X}, Q_{Z|X})$ using standard MMD arguments \citep{Gretton_2012}, and can be expressed in closed form as
\begin{equation} \label{eq:jmmd_est_closed}
    \widehat{\text{MMD}}^2 = \frac{1}{n^2} Tr(L_\mathbf{YY} K_\mathbf{XX}) - \frac{2}{nm}Tr(L_\mathbf{YZ} K_\mathbf{X' X}) + \frac{1}{m^2}Tr(L_\mathbf{ZZ} K_\mathbf{X' X'}).
\end{equation} 
However, when $P_X \neq Q_X$, estimation is not so straightforward, as we do not have samples coming from $P_{Y|X} \otimes R_X$ and $Q_{Z|X} \otimes R_X$. In such a case, we employ the estimator 
\begin{equation} \label{eq:cmmd2_est}
    \widehat{\text{CMMD}}_2^2 = \| (\hat C_{Y|X} - \hat C_{Z|X}) \hat C_{\tilde X \tilde X} \|_{\mathcal H_\ell \otimes \mathcal H_k}^2
\end{equation}
Since $\hat C_{Y|X}$ and $\hat C_{Z|X}$ converge in probability to $C_{Y|X}$ and $C_{Z|X}$ (under the same assumptions as in Lemma \ref{lem:cmmd0_conv}), and $\hat C_{\tilde X \tilde X}$ converges to $C_{\tilde X \tilde X}$ \citep{muandetKernelMeanEmbedding2017a}, it follows from the continuous mapping theorem that $\hat C_{Y|X} \hat C_{\tilde X \tilde X}$ converges in probability to $C_{Y|X} C_{\tilde X \tilde X}$. Using a similar argument to the proof of Lemma \ref{lem:cmmd0_conv}, we can show that $\widehat{\text{CMMD}}_2^2$ converges to $\| C_{Y|X} C_{\tilde X \tilde X} - C_{Z|X} C_{\tilde X \tilde X} \|_{\mathcal H_\ell \otimes \mathcal H_k}^2 = \| \mu_{X Y} - \mu_{X Z} \|_{\mathcal H_k \otimes \mathcal H_\ell}^2$, making it a consistent estimator for CMMD$_2^2$. A closed form expression for (\ref{eq:cmmd2_est}) is given below.
\begin{lemma} \label{lem:cmmd2_est_closed}
We have
\begin{multline}
    \widehat{\textup{CMMD}}_2^2 = \frac{1}{(n+m)^2}[Tr(W_\mathbf{X} L_\mathbf{YY} W_\mathbf{X} K_{\mathbf{X \tilde X}} K_{\mathbf{\tilde X \tilde X}} K_{\mathbf{\tilde X X}}) 
    - 2Tr(W_\mathbf{X} L_\mathbf{YZ} W_\mathbf{X'} K_{\mathbf{X' \tilde X}} K_{\mathbf{\tilde X \tilde X}} K_{\mathbf{\tilde X X}}) \\
     + Tr(W_\mathbf{X'} L_\mathbf{ZZ} W_\mathbf{X'} K_{\mathbf{X' \tilde X}}K_{\mathbf{\tilde X \tilde X}} K_{\mathbf{\tilde X X'}})]
\end{multline}
where $[K_{\mathbf{\tilde X \tilde X}}]_{1\leq i \leq n+m, 1 \leq j \leq n+m} = k(\tilde x_i, \tilde x_j)$. 
\end{lemma}
Apart from the CMO-based estimator of CMMD$_2$ given above, other options also exist. \citet{Glaser_24_ACMMD} assume paired samples of $X$, making their estimator closer in form to (\ref{eq:jmmd_est_closed}) but with a shared Gram matrix $K$. \citet{lee_general_2024} assume $P_X \neq Q_X$, and require density ratio estimation for their estimator in order to simulate a shared marginal. 

Finally, we consider the squared level $s$ CMMD, for which we define the empirical estimator
\begin{equation}
    \widehat{\textup{CMMD}}_s^2 = Tr(\hat \Delta^* \hat \Delta \hat C_{\tilde X \tilde X}^s).
\end{equation}
Once again, it is possible to express $\widehat{\textup{CMMD}}_s^2$ in closed form as given below.
\begin{theorem} \label{thm:cmmd_s_est_closed}
Let $\Pi_\mathbf{X} = [I_n \, \, 0] \in \mathbb R^{n\times(n+m)}$ and $\Pi_\mathbf{X'} = [0 \, \, I_m] \in \mathbb R^{m\times(n+m)}$ be projection matrices. Then
\begin{multline}
    \widehat{\textup{CMMD}}_s^2 = \frac{1}{(n+m)^s}[Tr(W_\mathbf{X} L_\mathbf{YY} W_\mathbf{X} \Pi_\mathbf{X} K_{\mathbf{\tilde X \tilde X}}^{s+1} \Pi_\mathbf{X}^\top) 
- 2Tr(W_\mathbf{X} L_\mathbf{YZ} W_\mathbf{X'} \Pi_\mathbf{X'} K_{\mathbf{\tilde X \tilde X}}^{s+1} \Pi_\mathbf{X}^\top) \\
 + Tr(W_\mathbf{X'} L_\mathbf{ZZ} W_\mathbf{X'} \Pi_\mathbf{X'} K_{\mathbf{\tilde X \tilde X}}^{s+1} \Pi_\mathbf{X'}^\top)].
\end{multline}
\end{theorem}

One can easily verify that setting $s$ to 0, 1 or 2 recovers the results in Lemmas \ref{lem:cmmd0_est_closed}, \ref{lem:cmmd1_est_closed} and \ref{lem:cmmd2_est_closed} respectively. Note that in all the above estimators, $\hat C_{\tilde X \tilde X}$ may be replaced by $\hat C_{XX}^R$ for an alternative CMMD estimator incorporating a general mixture proportion. Smoothing is with respect to the marginal distribution $R_X$, and hence depends on the parameter $\rho$. When there is a natural choice for using either $P_X$ or $Q_X$ as the reference distribution, one may set $\rho=1$ or $\rho = 0$ respectively. If removing the effect of imbalanced datasets is desired, choosing $\rho = \frac{1}{2}$ may be advisable.

\subsection{Doubly Robust Estimator} \label{sec:doubly_robust}

An issue that may arise in kernel methods is choosing a kernel that does not correctly capture the relationship in the data. In the context of CMMD, this may result in differences between conditional distributions not being picked up. This is reflected in Assumption $\ref{asm:cmo_hs}$ being violated due to the RKHS $\mathcal H_k$ not being a rich enough class of functions. To alleviate this problem, we additionally introduce a \emph{doubly robust estimator} for CMMD. This estimator combines a model for CMEs and propensity of covariates, and has the doubly robust property that if the product of the errors of the two models decreases sufficiently fast, then the estimator converges to the true value \citep{bang_doubly_2005}. Similar doubly robust estimators for measuring conditional distribution treatment effect in a causal setting were proposed in concurrent work by \citet{jain2026conditionaldistributionaltreatmenteffects}, which we now extend to the more general two-sample conditional distribution testing problem.

We begin by introducing the \emph{propensity score}, which is the probability that a given covariate comes from a particular distribution \citep{Imbens_Rubin_2015}. Let $T$ be an indicator variable, with $T=1$ if $x$ is sampled from $P_X$, and $T=0$ if $x$ is sampled from $Q_X$. Then the propensity score is defined as $e(x) = \mathbb E[T | X=x]$. Suppose that the marginal distributions $P_X$ and $Q_X$ have probability density functions $p$ and $q$, and consider the mixture $R_X = \rho P_X + (1-\rho) Q_X$ from earlier. Then an alternative expression for the propensity is
\begin{equation}
    e(x) = \frac{\rho p(x)}{\rho p(x) + (1-\rho)q(x)}.
\end{equation}
By this construction, one can show that $\mathbb E[T] = \rho$. Throughout this section, we assume that $0<e(X)<1$ with probability one, which is known as the overlap assumption. 

Next we introduce a random variable $W$ whose distribution is given by the mixture $\rho P_Y + (1-\rho)Q_Z$ and can be thought of as a combination of the outcomes $Y$ and $Z$. This allows us to express $Y = W|T=1$ and $Z = W | T=0$. Thus, taking samples from $P_{XY}$ and $Q_{XZ}$ individually can instead be replaced by taking samples from $P_{TXW}$. With this construction, we can introduce an alternative expression for conditional mean embeddings. 

\begin{lemma} \label{lem:cme_prop}
    The CMEs for the conditional distributions $P_{Y|X}$ and $Q_{Z|X}$ can be expressed as
    \begin{equation} \label{eq:cme_prop}
        \mu_{Y|x} = \mathbb{E}\left[ \frac{T \ell(\cdot, W)}{e(X)} \Big| X=x \right] \qquad \text{and} \qquad \mu_{Z|x} = \mathbb{E} \left[ \frac{(1-T) \ell(\cdot, W)}{(1-e(X))} \Big| X=x \right].
    \end{equation}
\end{lemma}

The expressions in line (\ref{eq:cme_prop}) can be interpreted as the population versions of the inverse probability weighting estimator for the CME \citep{Horvitz_ipw_1952}. With some basic manipulation, Lemma \ref{lem:cme_prop} further allows us to express the CME as,
\begin{equation} \label{eq:cme_dr}
    \mu_{Y|x} = \mathbb{E}\left[ \frac{T}{e(X)} (\ell(\cdot, W) - \mu_{Y|X}) + \mu_{Y|X} \Big| X=x \right].
\end{equation}

Suppose we fit models $\hat e$ and $\hat \mu_{Y|X}^{model}$ on training data. Then a doubly robust (DR) estimator for the CME of $P_{Y|X}$ at $X=x$ on test data $\{(t_i, \tilde x_i, w_i)\}_{i=1}^{n_t}$ sampled from $P_{TXW}$ is
\begin{equation}
    \hat \mu_{Y|x}^{DR} = \sum_{i=1}^{n_t} \beta_i(x) \left[ \frac{t_i}{\hat e(\tilde x_i)} (\ell(\cdot, w_i) - \hat \mu_{Y|\tilde x_i}^{model}) + \hat \mu_{Y|\tilde x_i}^{model} \right]
\end{equation}
where $\beta(x) = [\beta_1(x), \dots, \beta_{n_t}(x)]^\top = (K_{\mathbf{\tilde X \tilde X}} + \lambda n_t I_{n_t})^{-1} K_{\mathbf{\tilde X} x}$. The estimator is an extension of the doubly robust counterfactual mean embedding estimator proposed by \citet{fawkes_dr_2024}, applied to conditional distributions. The test data can be formed by combining the data $\{(x_i, y_i)\}_{i=1}^n$ and $\{(x_j', z_j)\}_{j=1}^m$ from earlier, with $n_t = n+m$. Thus, the combined samples from $P$ and $Q$ are used to estimate the CME, rather than just one set as before. 

\begin{theorem} \label{thm:cme_dr}
    Assume that both the true propensity and the estimated propensity are uniformly bounded away from zero, and that the estimators $\hat e$ and $\hat \mu_{Y|X}^{model}$ satisfy
    \begin{equation*}
        (\mathbb E_X | e(X) - \hat e(X) |^2)^{\frac{1}{2}} = O_p(\zeta_{n_t}) \qquad \text{and} \qquad (\mathbb E_X \| \mu_{Y|X} - \hat \mu_{Y|X}^{model} \|_{\mathcal H_\ell}^2)^{\frac{1}{2}} = O_p(\eta_{n_t})
    \end{equation*}
    with $\eta_{n_t} = O(1)$. If we further have that Assumptions \ref{asm:k_bounded}–\ref{asm:cmo_hs} hold, then 
    \begin{equation*}
        \| \mu_{Y|x} - \hat \mu_{Y|x}^{DR} \|_{\mathcal H_\ell} = O_p(\lambda^{\frac{1}{2}} + \lambda^{-1}n_t^{-\frac{1}{2}} + \lambda^{-1}\zeta_{n_t} \eta_{n_t}).
    \end{equation*}
\end{theorem}

Therefore, if $\zeta_{n_t} \eta_{n_t} = o(\lambda)$ along with $\lambda \to 0$ and $\sqrt{n_{t}} \lambda \to \infty$, we have that the estimator $\hat \mu_{Y|x}^{DR}$ converges in probability to the CME $\mu_{Y|x}$. For example, if the CME model is misspecified with bounded error, but the propensity estimator is such that $\zeta_{n_t} = o(\lambda)$ with the regularization parameter reducing sufficiently with increased sample size, then the estimator still converges to the true value. Thus, we get the doubly robust property that even if one nuisance model is not consistent, but the product of errors decreases appropriately fast, then the estimator is still consistent. We can similarly define the DR estimator for $\mu_{Z|x}$ as
\begin{equation*}
    \hat \mu_{Z|x}^{DR} = \sum_{i=1}^{n_t} \beta_i(x) \left[ \frac{(1-t_i)}{(1-\hat e(\tilde x_i))} (\ell(\cdot, w_i) - \hat \mu_{Z|\tilde x_i}^{model}) + \hat \mu_{Z|\tilde x_i}^{model} \right].
\end{equation*}
Taking the empirical mean gives an estimator for the squared CMMD$_1$,
\begin{equation}
    \widehat{\text{CMMD}}_{1, DR}^2 = \frac{1}{n_t}\sum_{i=1}^{n_t} \| \hat \mu_{Y|\tilde x_i}^{DR} - \hat \mu_{Z|\tilde x_i}^{DR}\|_{\mathcal H_\ell}^2
\end{equation}
which once again has the doubly robust property. However, we are not limited to the level one metric. If we define 
\begin{equation}
    \hat \Delta_{DR} = \Psi_{DR} (K_{\mathbf{\tilde X \tilde X}} + n_t \lambda I_{n_t})^{-1} \Phi_{\mathbf{\tilde X}}^*
\end{equation}
where the operator $\Psi_{DR}: \mathbb R^{n_t} \to \mathcal H_\ell$ has columns given by 
\begin{equation*}
    [\Psi_{DR}]_i = \frac{t_i}{\hat e(\tilde x_i)}(\ell(\cdot, w_i) - \hat\mu_{Y|\tilde x_i}^{model})+\hat\mu_{Y|\tilde x_i}^{model} - \frac{(1-t_i)}{(1-\hat e(\tilde x_i))}(\ell(\cdot, w_i) - \hat\mu_{Z|\tilde x_i}^{model}) - \hat\mu_{Z|\tilde x_i}^{model}
\end{equation*}
for $i = 1, \dots, n_t$ (known as pseudo-outcomes), then $\widehat{\text{CMMD}}_{1, DR}^2 = \frac{1}{n_t}\sum_{i=1}^{n_t} \| \hat \Delta_{DR} k(\cdot, \tilde x_i)\|_{\mathcal H_\ell}^2$. The operator $\hat \Delta_{DR}$ can be interpreted as estimating the difference between the CMOs $C_{Y|X}$ and $C_{Z|X}$. However, we now have shared covariates $\tilde X$ sampled from $R_X$ and individual feature representations $\ell(\cdot, y)$ and $\ell(\cdot, z)$ are replaced by the pseudo-outcome. Note that the regularization parameter $\lambda$ is also shared as we now only have a single dataset for estimation. Through this construction we consider the modified Assumption \ref{asm:cmo_hs}.

\renewcommand{\theassumption}{5*}
\begin{assumption} \label{asm:delta_hs}
    Within the equivalence class of Markov kernels for $P_{Y|X}$ and $Q_{Z|X}$, there exist representatives $\tilde{\kappa}_P$ and $\tilde{\kappa}_Q$ such that for any $g \in \mathcal H_\ell$, the map $x \mapsto \int_{\mathcal Y} g(y) (\tilde{\kappa}_P (x, dy) - \tilde{\kappa}_Q (x, dy)) \in \textup{range}(C_{XX}^\gamma)$ for some $\gamma \geq \frac{1}{2}$.
\end{assumption}

The consistency of the estimator is then given by the following theorem.
\begin{theorem} \label{thm:cmmd_dr_conv}
    Take the same assumptions as in Theorem \ref{thm:cme_dr}, with $\eta_{n_t}$ also bounding the error for $\hat \mu_{Z|X}^{model}$ and replacing Assumption \ref{asm:cmo_hs} by \ref{asm:delta_hs}. Furthermore, assume that both true and estimated propensity are uniformly bounded away from zero and one. If the regularization parameter $\lambda$ satisfies $\lambda \to 0$, $n_t \lambda^3 \to \infty$ and $\lambda^{-1} \zeta_{n_t} \eta_{n_t} \to 0$, then 
    \begin{equation*}
        \|\Delta - \hat \Delta_{DR}\|_{\mathcal H_\ell \otimes \mathcal H_k} \overset{p}{\to} 0.
    \end{equation*}
\end{theorem}

By once more invoking the continuous mapping theorem, we get that the DR estimator 
\begin{equation}
    \widehat{\text{CMMD}}_{s, DR}^2 = Tr(\hat \Delta_{DR}^* \hat\Delta_{DR} \hat C_{\tilde X \tilde X}^s)
\end{equation}
converges to the squared level $s$ CMMD. While the estimators of Section \ref{sec:naive_est} fit embeddings individually with samples from $P$ and $Q$ and then take their difference, the power of the DR estimators lies in their ability to directly fit the difference using the whole data. This means that even if individual conditional relationships cannot be captured by kernel embeddings, but the difference between conditional distributions can be, then we can still achieve consistent CMMD estimators through the DR method. Conversely, when the propensity model is not correctly learned, the DR estimators are once again consistent as long as the CME models converge to the true value sufficiently fast. This may be relevant when using CME models outside of the standard estimators given in Section \ref{sec:back_cmo}, such as neural-kernel conditional mean embeddings \citep{Shimizu_neural_kcme_2024}.

\subsection{Limitations of proposed estimators}

A practical limitation of the estimators introduced in the above sections is that they require finding the inverse of an $n \times n$ matrix, which has $O(n^3)$ computational complexity and $O(n^2)$ memory requirements. This may make the proposed estimators expensive to compute in large data regimes. A possible way to reduce complexity is through low rank approximation of the Gram matrices such as incomplete Cholesky decomposition \citep{Fine_2002}. Alternatively, when $\mathcal H_k$ is finite-dimensional, it is possible to estimate CMOs in the primal form, which we describe in more detail in Appendix \ref{ap:experiment_details_mnist}. For a $D$-dimensional RKHS, the complexity becomes $O(nD^2 + D^3)$ which is beneficial if $D \ll n$. 

Alternative kernel-based conditional discrepancy estimators may provide cheaper computation. For example, \citet{yan_distance_2024} propose an estimator with quadratic cost, but require kernel smoothing, resulting in a biased quantity. \citet{lee_general_2024} provide both quadratic and linear computation variants, exchanging CMO estimation for density-ratio estimation. Graph-based methods provide further alternatives, with \citet{huang_kpc_2022} reporting $O(K n\log n)$ computation for a $K$ nearest-neighbor approach for measuring conditional dependence. \citet{chatterjee_kernel-based_2024} give similar complexity for their CMMD$_1$ estimator, with power comparable to MMD-based tests, but assume outcomes are observed at shared covariates. Consequently, practical choice depends on data design, nuisance-estimation difficulty, and available computation. Our experiments compare smoothing levels and do not establish computational competitiveness against these alternatives.

Regarding the estimator in Section \ref{sec:doubly_robust}, note that our derived doubly robust property is slightly different to the usual formulation, which only requires consistency of either of the nuisance models. In contrast, we require that the product of errors of the CME and propensity model decreases faster than $\lambda$. However, we note that this is not a particularly restrictive condition as in our setting the regularization parameter typically reduces slowly, e.g. $\lambda = n^{-\frac{1}{4}}$. Furthermore, we wish to highlight that the doubly robust estimator is only applicable when the propensity is either known or can be reasonably learned, through e.g. density ratio estimation, and there is sufficient overlap between $P_X$ and $Q_X$. For example, if the data comes from two separate experiments with participants assigned deterministically to either one based on their covariate value, then using the DR estimator would not be appropriate. On the other hand, if $P_X = Q_X$, then the propensity is a constant, making it potentially easier to learn and apply the doubly robust approach. Such a scenario may occur in a paired sample setting, where for each $x$ we observe two outcomes $y$ and $z$, with the propensity being $e(x) = 0.5$.

\section{Hypothesis Testing} \label{sec:hyp_test}

For statistical testing, we aim to determine whether two conditional distributions are equal almost surely. More formally, we can state the hypotheses as
\begin{equation} \label{eq:hypotheses}
    H_0: R_X(P_{Y|X} = Q_{Z|X}) = 1 \qquad \text{versus} \qquad H_1: R_X(P_{Y|X} = Q_{Z|X}) < 1.
\end{equation}
The null hypothesis means that the conditional distributions are the same for almost all $X$, whereas the alternative implies that the conditional distributions are different for some covariates of positive measure. By Assumption \ref{asm:full_support}, $R_X$ above can be replaced by either $P_X$ or $Q_X$ without loss of generality. Hypothesis testing using kernel statistics is well studied, both for CMMD$_2$ \citep{Gretton_2012, Glaser_24_ACMMD, lee_general_2024} and CMMD$_1$ \citep{Park_21_codite, yan_distance_2024, chatterjee_kernel-based_2024}. In this section, we provide general testing algorithms that can be used with any of the test statistics described in Section \ref{sec:estimation}, including CMMD$_0$ and DR estimators. Since it is difficult to determine the distribution of proposed CMMD test statistics under the null hypothesis, this motivates us to use bootstrapping algorithms. 

We first consider the setting in which the marginal distributions of $X$ under $P$ and $Q$ coincide, i.e., $P_X = Q_X$. In this case, the joint distributions factorize as $P_{XY} = P_{Y|X} \otimes P_X$ and $Q_{XZ} = Q_{Z|X} \otimes P_X$. Under the null hypothesis, these joint distributions are the same, and consequently the observed i.i.d. data $\mathcal D_P = \{(x_i, y_i)\}_{i=1}^n \sim P_{XY}$ and $\mathcal D_Q = \{(x'_j, z_j)\}_{j=1}^m \sim Q_{XZ}$ are exchangeable. Algorithm~\ref{alg:1} describes the resulting bootstrapping test, with each replicate repeating any required label-dependent nuisance fitting and cross-validation. 

\begin{algorithm}[t]
\caption{Kernel two-sample test for conditional distributions ($P_X = Q_X$)}
\label{alg:1}
\begin{algorithmic}[1] 
\Require Data $\mathcal D_P = \{(x_i, y_i)\}_{i=1}^n$ and $\mathcal D_Q =\{(x_j', z_j)\}_{j=1}^m$, significance level $\alpha$, kernels $k$, $\ell$, number of bootstrap samples $B$.
\State Calculate CMMD test statistic $S$ using $\mathcal D_P$ and $\mathcal D_Q$.
\State Pool the data as $\tilde{\mathcal D} = \{(\tilde x_i,  w_i)\}_{i=1}^{n+m}$, ordered independently of source label.
\For{$b = 1, \dots, B$}
    \State Get $\mathcal D_P^{(b)}$ by randomly sampling without replacement $n$ points from $\tilde{\mathcal D}$.
    \State Set $\mathcal D_Q^{(b)} = \tilde{\mathcal D} \setminus \mathcal D_P^{(b)}$.
    \State Calculate $S^{(b)}$ from the new datasets $\mathcal D_P^{(b)}$ and $\mathcal D_Q^{(b)}$.
\EndFor
\State Calculate the $p$-value as $p = \frac{1 + \sum_{b=1}^B \mathbf{1}\{S^{(b)} \geq S\}}{1+B}$.
\If{$p \leq \alpha$}
    \State Reject $H_0$.
\EndIf
\end{algorithmic}
\end{algorithm}

\begin{algorithm}[!t]
\caption{Kernel two-sample test for conditional distributions ($P_X \neq Q_X$)}
\label{alg:2}
\begin{algorithmic}[1]
\Require Data $\mathcal D_P = \{(x_i, y_i)\}_{i=1}^n$ and
$\mathcal D_Q = \{(x_j', z_j)\}_{j=1}^m$, significance level $\alpha$,
kernels $k$, $\ell$, number of bootstrap samples $B$, propensity function $e$.

\State Calculate CMMD test statistic $S$ using $\mathcal D_P$ and $\mathcal D_Q$.
\State Pool the data as $\tilde{\mathcal D} = \{(\tilde x_i,  w_i)\}_{i=1}^{n+m}$, ordered independently of source label.
\For{$b = 1, \dots, B$}
    \State Sample $\mathbf t^{(b)}=(t_1^{(b)},\dots,t_{n+m}^{(b)})$
    from the conditional distribution
    \[
    \mathbb P(\mathbf T=\mathbf t \mid \tilde{\mathcal D},
    \textstyle\sum_{i=1}^{n+m} T_i = n)
    \propto 
    \mathbf 1\left\{\sum_{i=1}^{n+m} t_i=n\right\}
    \prod_{i=1}^{n+m} e(\tilde x_i)^{t_i}(1-e(\tilde x_i))^{1-t_i}.
    \]
    \State Set $\mathcal D_P^{(b)} = \{(\tilde x_i, w_i) : t_i^{(b)}=1\}$ and $\mathcal D_Q^{(b)} = \{(\tilde x_i, w_i) : t_i^{(b)}=0\}$. 
    \State Calculate $S^{(b)}$ from
    $\mathcal D_P^{(b)}$ and $\mathcal D_Q^{(b)}$.
\EndFor

\State Calculate the $p$-value as $p = \frac{1+\sum_{b=1}^B \mathbf 1\{S^{(b)} \geq S\}}{1+B}$.
\If{$p \leq \alpha$}
    \State Reject $H_0$.
\EndIf
\end{algorithmic}
\end{algorithm}

If we instead have that $P_X \neq Q_X$, further care must be taken to ensure that the algorithm satisfies correct Type I error control. We will assume we have access to the propensity function $e$, and utilize the conditional resampling scheme proposed by \citet{Rosenbaum_condtest_1984}. After computing the test statistic using $\mathcal D_P$ and $\mathcal D_Q$, data is resampled based on the propensity score $e(\tilde x_i)$ for each covariate, while ensuring that the two bootstrap samples maintain size $n$ and $m$ respectively. Fixed sample sizes are required for correct Type I error control \citep{lee_general_2024}, which we achieve by using a fixed-count conditional Bernoulli distribution \citep{Chen_conditionalbern_1997, branson_bernoullitesting_2019}. Repeating this procedure and recalculating the test statistic for each replicate yields an estimate of its sampling distribution under the null hypothesis. Algorithm \ref{alg:2} describes the bootstrapping procedure in more detail. A similar test was applied by \citet{Park_21_codite} in a causal inference setting with CMMD$_1$ as the test statistic, but without controlling for sample sizes. Furthermore, our proposed algorithm satisfies the following finite-sample guarantee.
\begin{theorem} \label{thm:test_val}
    Suppose the propensity function $e$ is known and satisfies $0<e(X)<1$ almost surely. Then the test in Algorithm~\ref{alg:2} has finite-sample Type I error at most $\alpha$, for any choice of test statistic described in Section \ref{sec:estimation}.
\end{theorem}

We wish to highlight that the effect of regularization and estimator bias does not impact the Type I error control of the tests, making it suitable for all test statistics introduced in Section \ref{sec:estimation}. Sampling labels $\mathbf t$ in line 4 of Algorithm \ref{alg:2} can be achieved by rejection sampling or dynamic programming. We utilise the latter in our experiments, with the procedure explained in more detail in the Appendix \ref{sec:app_testing}. However, when $P_X = Q_X$ the propensity is constant, resulting in Algorithm~\ref{alg:2} sampling points uniformly to generate the bootstrap sets. Therefore, one can view Algorithm~\ref{alg:1} as a special case of Algorithm~\ref{alg:2}, meaning that Theorem~\ref{thm:test_val} also implies validity of Algorithm~\ref{alg:1}.

\subsection{Practical guidance on choosing CMMD level}

A natural question that a practitioner may ask is which CMMD level to choose for their test. If the user's application is more concerned with differences in conditionals at highly likely covariates, then a greater value of $s$ may be preferred as these metrics may be better at detecting discrepancies in such areas. On the other hand, if avoiding discrepancies at all values of $X$ is important, then CMMD$_0$ should be selected since it avoids smoothing by the covariate distribution completely. However, we also provide a data-driven approach for CMMD level selection based on sample splitting. 

Split each dataset into selection and inference samples, and fix a finite candidate set $\mathcal S$ of CMMD levels. On the selection samples, compute the statistic $S_s$ for each $s\in\mathcal S$ and its resampling reference distribution using the bootstrapping procedures described in Algorithms \ref{alg:1} and \ref{alg:2}. Let $\hat q_{s,1-\alpha}$ and $\hat v_s$ denote the empirical $(1-\alpha)$ quantile and standard deviation of this distribution. Inspired by the kernel selection procedure of \citet{liu_deep_kernels_2020}, we propose selecting
\begin{equation} \label{eq:s_selected}
    s^* = \arg\max_{s\in\mathcal S} \frac{S_s - \hat q_{s,1-\alpha}}{\hat v_s}.
\end{equation}
This compares excess over the estimated rejection threshold on a common scale. Once the value $s^*$ is chosen and fixed, performing the test on the remaining inference samples maintains $\alpha$ level validity. Although the selection criterion does not guarantee an optimal level (in fact it may reduce test power due to fewer inference samples), selecting the level as in \eqref{eq:s_selected} can help avoid a worst case scenario when fixing a CMMD metric in advance.

\section{Experiments}

We are now ready to demonstrate the ability of CMMD statistics to measure and test for differences between conditional distributions through numerical analysis. Unless otherwise stated, we set $n=m$ when sampling from $P$ and $Q$, and assume that the distributions are equally prevalent. For hypothesis testing, we apply a significance level of $\alpha = 0.05$, with $B=200$ bootstrap samples and rejection rates are estimated from $200$ independent trials. We begin with toy examples on synthetic data, and then move to testing with real data. Further experimental details are provided in Appendix \ref{ap:experiment_details}.

\subsection{Synthetic Data: Hypothesis Testing} \label{sec:exp_synthetic}

We begin with an experiment on synthetic data inspired by \cite{park_measure_2020}. The covariate has distribution $X \sim \mathcal N(\theta,\frac{3}{4})$, where the parameter $\theta$ varies between -1 and 1. The conditional outcomes are $Y|X = \exp(-0.5 X^2) \sin(2X) + \epsilon$ and $Z|X = X + \epsilon$, where $\epsilon \sim 0.5 \mathcal N(0,1)$. We set $k$ and $\ell$ to be the Gaussian kernel, that is, $k(x, x') = \exp(-\frac{1}{2}h \|x - x'\|_2^2)$ with the bandwidth parameter chosen via the median heuristic \citep{Gretton_2012}. The regularization parameter is $\lambda = 0.1$. We take 100 samples from each of $P_{Y|X} \otimes P_X$ and $Q_{Z|X} \otimes P_X$ and perform a two-sample test via Algorithm \ref{alg:1}. Test statistics for the three CMMD levels are computed using the estimators in Section \ref{sec:naive_est}. In addition, we compare to choosing one of the three levels adaptively as in equation \eqref{eq:s_selected}, with 20\% of samples used for selecting $s^*$ and 80\% for testing. As a baseline, we also test on the held-out data with $s_{rand}$ chosen uniformly at random. Figure \ref{fig:synthetic_experiment} (left) shows a plot of the rejection rate over a range of $\theta$ values. The CMMD$_2$ and CMMD$_1$ tests, whose test statistics depend more heavily on the marginal distribution of $X$, experience a drop in power near $\theta = 0$. On the other hand, CMMD$_0$ is more stable in terms of power. Although CMMD$_{s^*}$ outperforms random selection and tests with the level $2$ metric, the lower power compared to CMMD$_0$ is likely due to less data remaining for testing and imperfect level selection. 

\begin{figure}[t]
    \centering
    \begin{subfigure}[b]{0.48\textwidth}
        \centering
        \includegraphics[width=\linewidth]{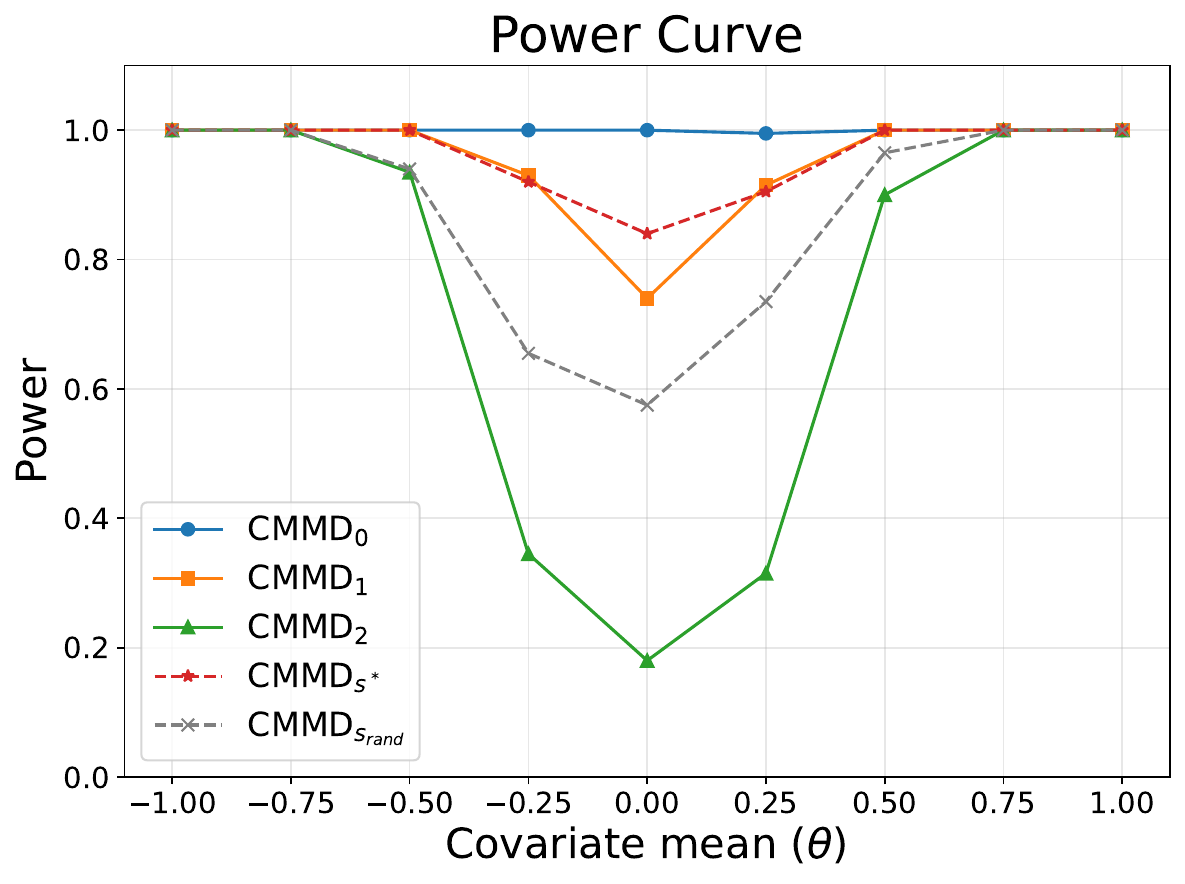}
    \end{subfigure}
    \hfill 
    \begin{subfigure}[b]{0.48\textwidth}
        \centering
        \includegraphics[width=\linewidth]{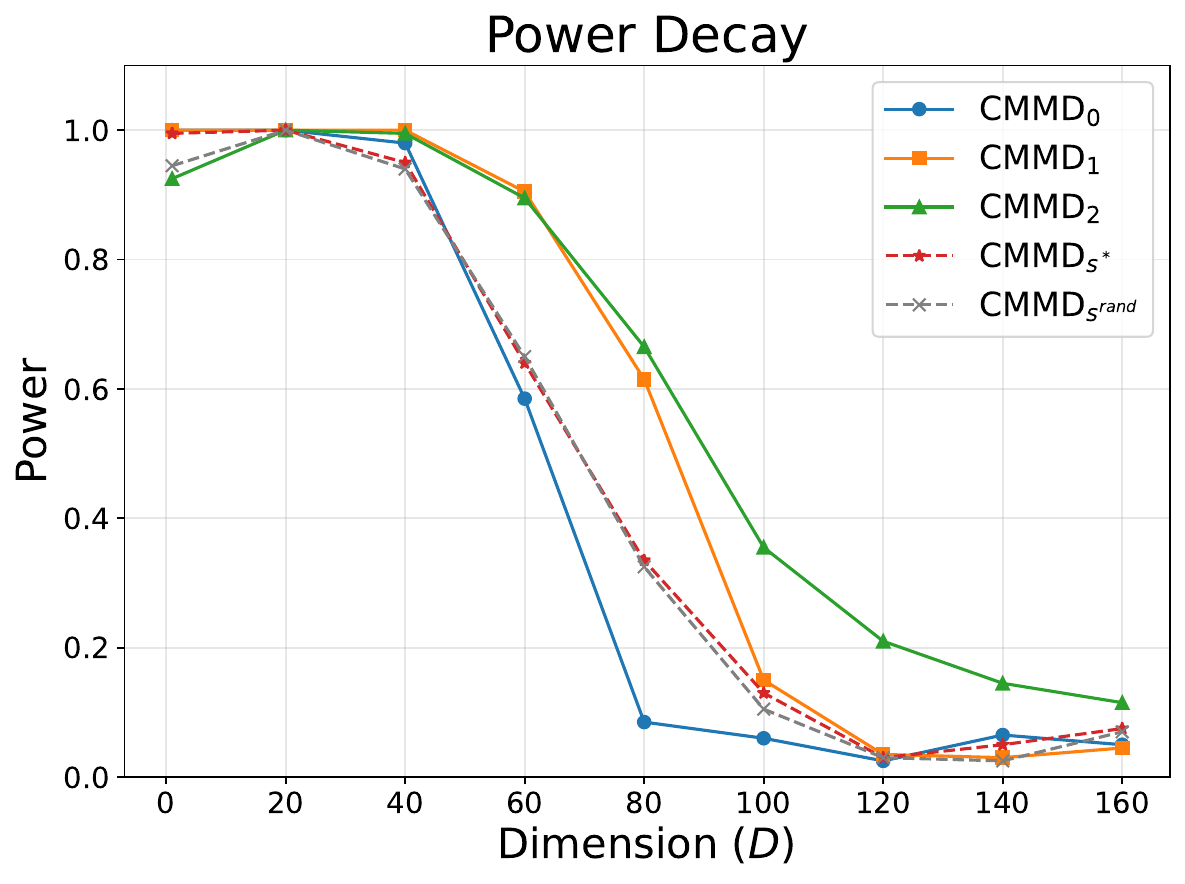}
    \end{subfigure}
    \caption{Rejection rates for CMMD test. Left: Power curve shows that CMMD$_2$ and CMMD$_1$ tests have worse power near $\theta = 0$, but CMMD$_0$ is more stable. Right: Test power decays as dimension increases, but this is slowest for CMMD$_2$.}
    \label{fig:synthetic_experiment}
\end{figure}

Next, we look at the effect of increasing dimension on the CMMD tests. We now consider variables $X$, $Y$ and $Z$ that are $D$-dimensional. To get data, we first independently sample $X_d \sim \mathcal N(\frac{1}{2}, \frac{3}{4})$ for $d = 1, \dots, D$ and then sample conditional outcomes $Y_d|X_d = \exp(-0.5 X_d^2) \sin(2X_d) + \epsilon_d$ and $Z_d|X_d = X_d + \epsilon_d$, where $\epsilon_d \sim (0.45 + 0.05d) \mathcal N(0,1)$. The noise in the outcomes increases with dimension, making the signal due to conditioning weaker. The remaining setup is the same as previously. In Figure \ref{fig:synthetic_experiment} (right), we plot test power against $D$. Power decays for all test statistics, however the decay is fastest for CMMD$_0$ and slowest for CMMD$_2$. In this setting, choosing $s^*$ adaptively has little advantage over random selection.

\begin{figure}[t]
    \centering
    \begin{subfigure}[b]{0.48\textwidth}
        \centering
        \includegraphics[width=\linewidth]{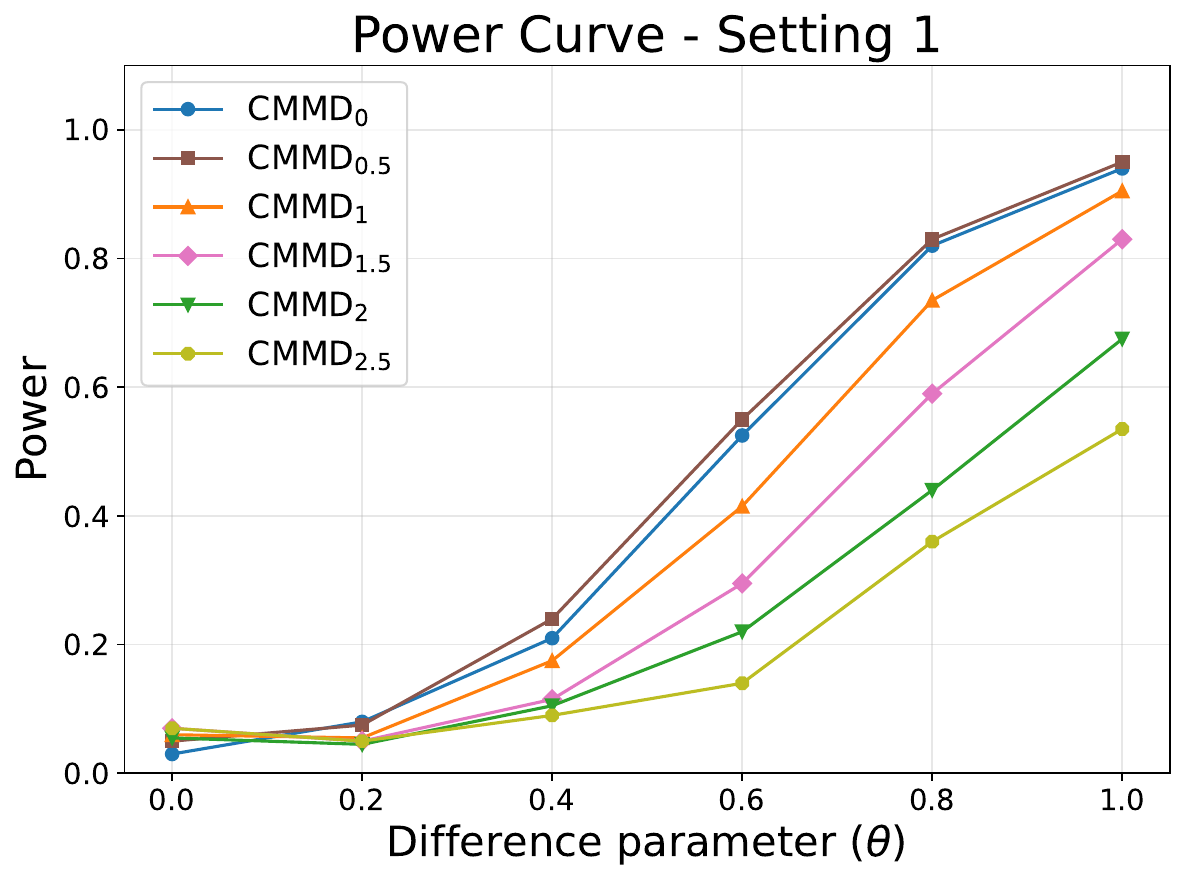}
    \end{subfigure}
    \hfill 
    \begin{subfigure}[b]{0.48\textwidth}
        \centering
        \includegraphics[width=\linewidth]{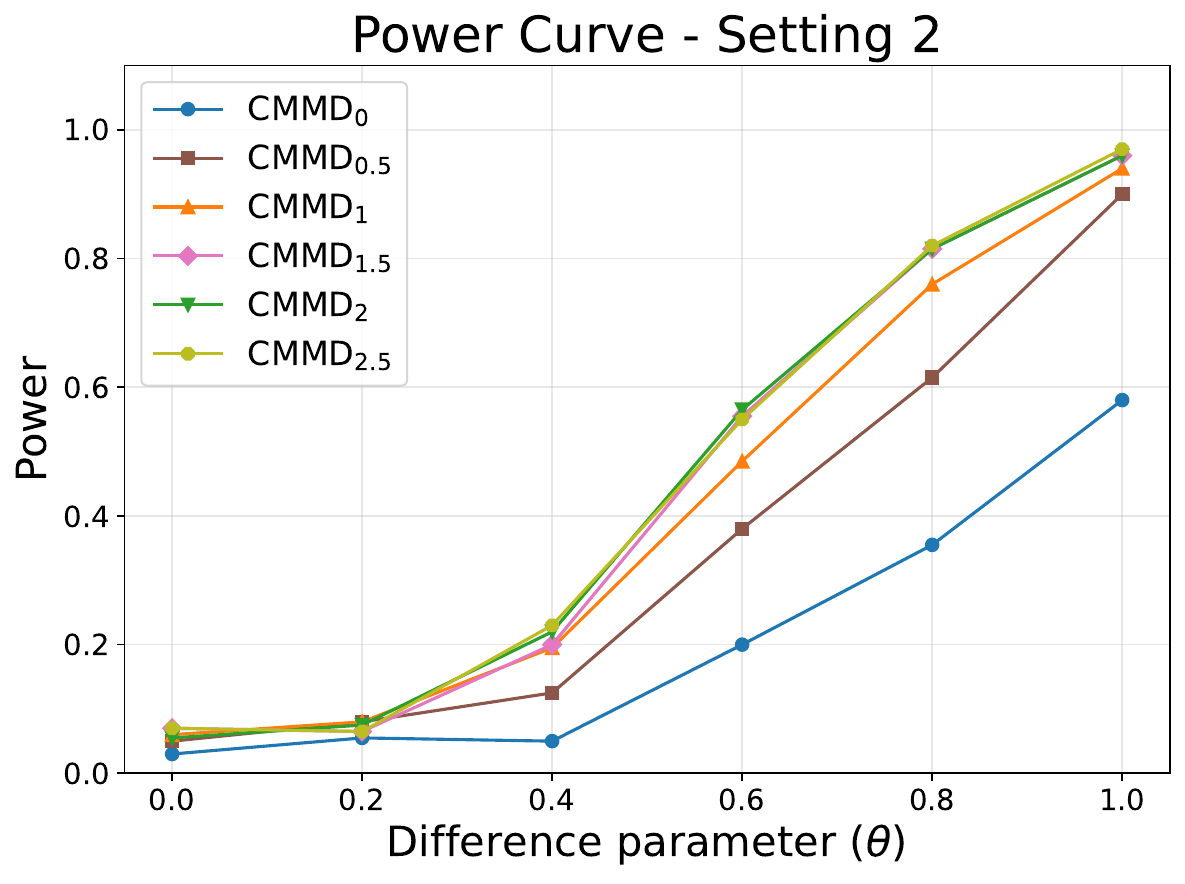}
    \end{subfigure}
    \caption{Rejection rates for level $s$ CMMD test. Left: Under Setting 1, extra smoothing deteriorates test power. Right: Under Setting 2, extra smoothing increases test power.}
    \label{fig:lvls_experiment}
\end{figure}

\subsection{Synthetic Data: CMMD Smoothing}

In this section, we illustrate the effect of smoothing in the general level $s$ CMMD. We use the standard CMMD estimators of Section \ref{sec:naive_est} for integer values of $s$, and estimators of the form given in Theorem \ref{thm:cmmd_s_est_closed} for fractional $s$. We compute fractional powers of the Gram matrix via spectral decomposition. We consider two synthetic settings. In both cases $X \sim \text{Beta}(4,4)$ and  $Y|X = \sin(\pi X) + \epsilon$. Under Setting 1, $Z|X = (1-\theta)\sin(\pi X) + \theta (3X - 0.5) + \epsilon$, whereas under Setting 2, $Z|X = (1-\theta)\sin(\pi X) + 0.5 \theta + \epsilon$. Here we have that $\theta$ is a parameter which varies between 0 and 1 and represents the discrepancy between the conditional distributions; for $\theta=0$ the distributions match and the difference increases as the parameter approaches 1. As before, we use a Gaussian kernel for $k$ and $\ell$ with median heuristic bandwidth. A regularization parameter of $\lambda = 0.1$ is applied and tests are conducted with $n=100$ samples. 

Figure \ref{fig:lvls_experiment} (left) shows the power curve for Setting 1. Starting at roughly the test significance level of 0.05, the curves for each value of $s$ approach 1 as $\theta$ increases. The greatest power comes for $s=0$ and $s=0.5$, with a steady drop-off as smoothing increases. Under Setting 2 in Figure \ref{fig:lvls_experiment} (right), the opposite relationship is observed, with CMMD$_0$ test statistic resulting in the worst power. Test power rises with further smoothing up until $s=1.5$, after which it remains steady for higher levels of CMMD. Overall, Figures \ref{fig:synthetic_experiment} and \ref{fig:lvls_experiment} show that the preferred CMMD test statistic can depend heavily on the scenario.

\subsection{Synthetic Data: Doubly Robust Estimator} \label{sec:exp_dr}
Next, we demonstrate the functionality of the doubly robust CMMD estimator. We take a covariate with domain $\mathcal X= [0,1]$ and continuous outcome $\mathcal Y = \mathbb R$. This time, the covariates have different marginal distributions under $P$ and $Q$; $X \sim P_X = \text{Unif}(0,1)$ and $X' \sim Q_X = \text{Beta}(0.5,0.5)$. This allows for an analytic expression for the propensity, which is given by $e(x) = 1 / (1 + \frac{1}{\pi} x^{-0.5} (1-x)^{-0.5})$ and used both for calculating pseudo-outcomes and also hypothesis testing under Algorithm \ref{alg:2}. The outcomes are such that $Y|X = \cos(4\pi X) + 0.5X^2 + \epsilon$ and $ Z|X' = \cos(4 \pi X') + \epsilon$, where $\epsilon \sim 0.5\mathcal N(0, 1)$. We choose the linear kernel $\ell(y,y') = y y'$ on $\mathcal Y$. In this setting, the true CMEs are $\mu_{Y|X} = \mathbb E[Y|X]$ and $\mu_{Z|X} = \mathbb E[Z|X]$ where the expectations are taken with respect to the conditional distributions $P_{Y|X}$ and $Q_{Z|X}$. For $\mathcal X$ we choose the polynomial kernel of degree two, $k(x,x') = (x x' + 1)^2$. The CME models $\hat \mu_{Y|X}^{model}$ and $\hat \mu_{Z|X}^{model}$ for the DR estimator are fitted using kernel ridge regression with the same polynomial kernel. In all cases, the regularization parameter is fitted using cross-validation. Although $k$ is misspecified for estimating the CMEs individually (see Appendix \ref{ap:experiment_details_dr} for an illustration), it is rich enough to capture the difference between the conditional distributions since Assumption \ref{asm:delta_hs} holds. Through the pseudo-outcomes, DR estimation is able to fit this difference directly, which yields improved results.

\begin{figure}[t]
    \centering
    \begin{subfigure}[b]{0.32\textwidth}
        \centering
        \includegraphics[width=\linewidth]{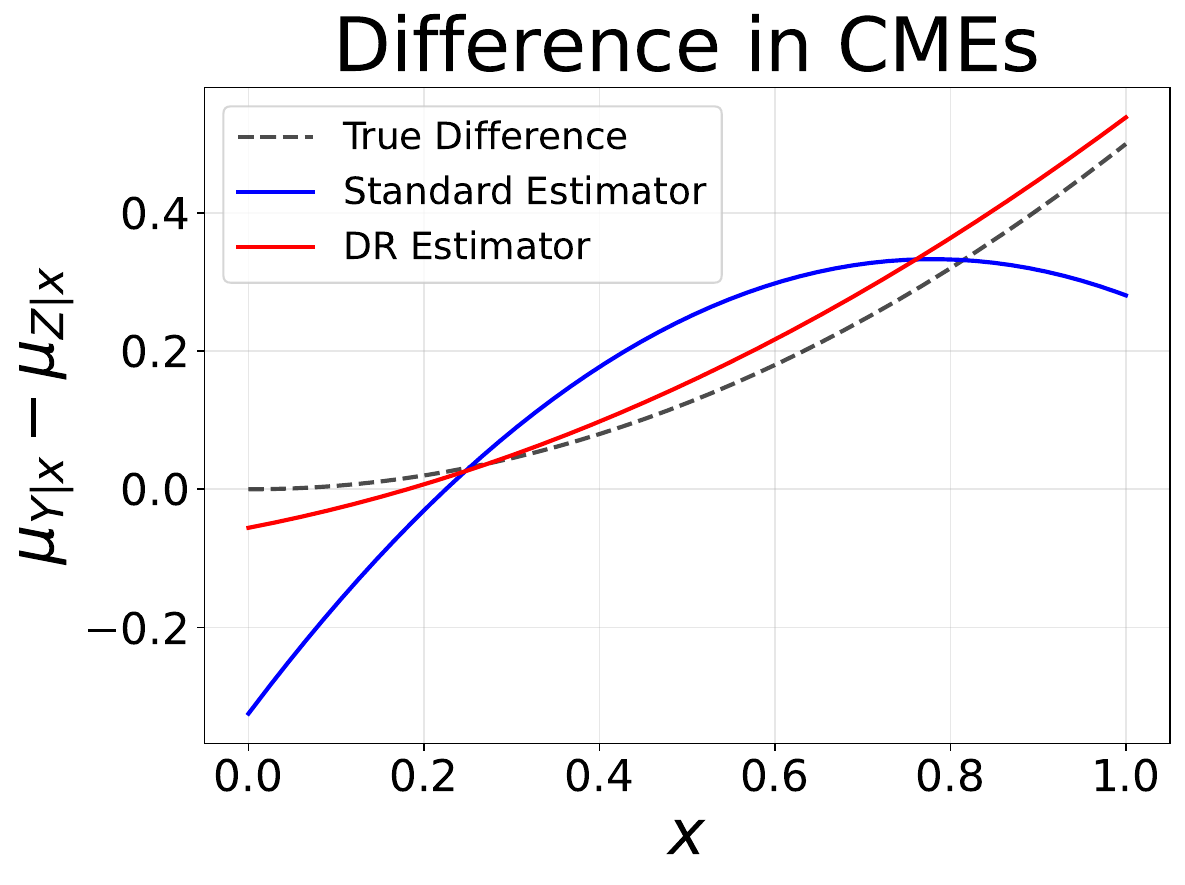}
    \end{subfigure}
    \hfill 
    \begin{subfigure}[b]{0.32\textwidth}
        \centering
        \includegraphics[width=\linewidth]{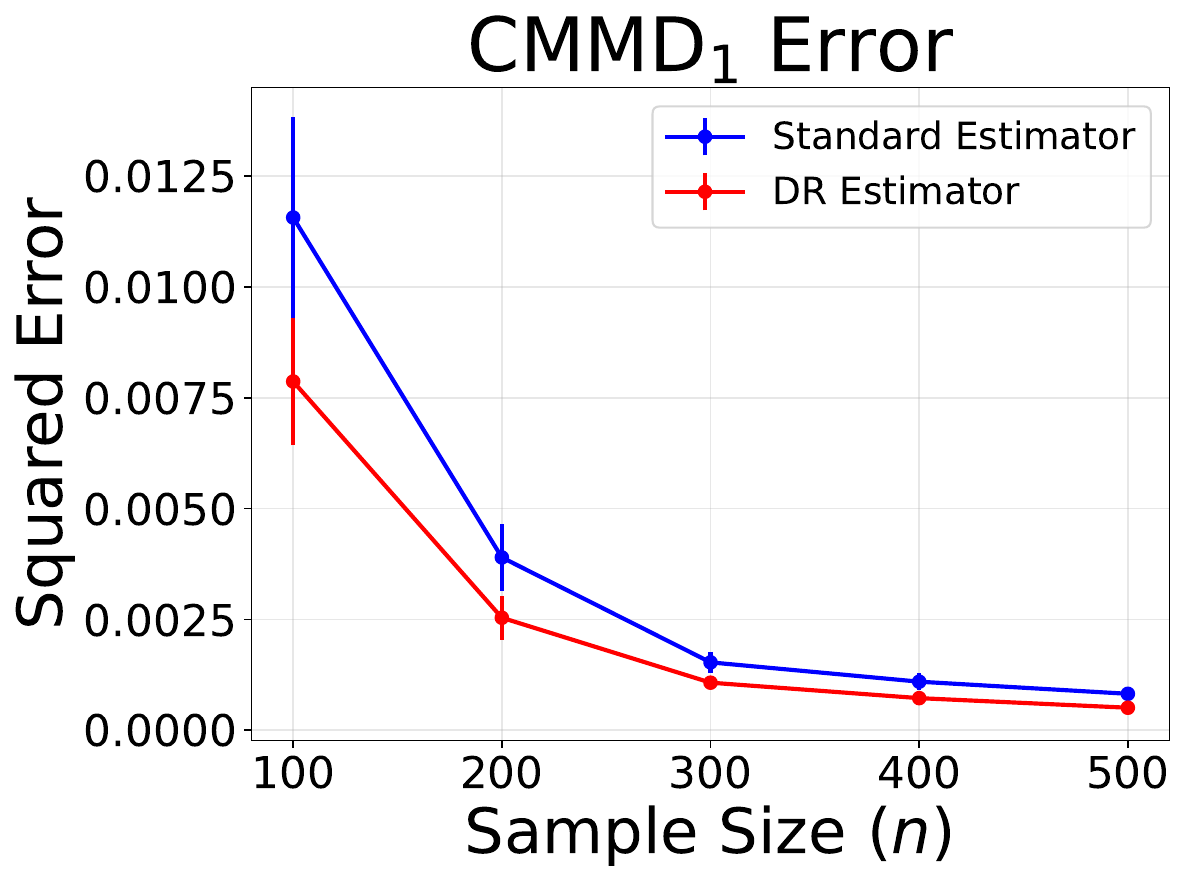}
    \end{subfigure}
    \hfill
    \begin{subfigure}[b]{0.32\textwidth}
        \centering
        \includegraphics[width=\linewidth]{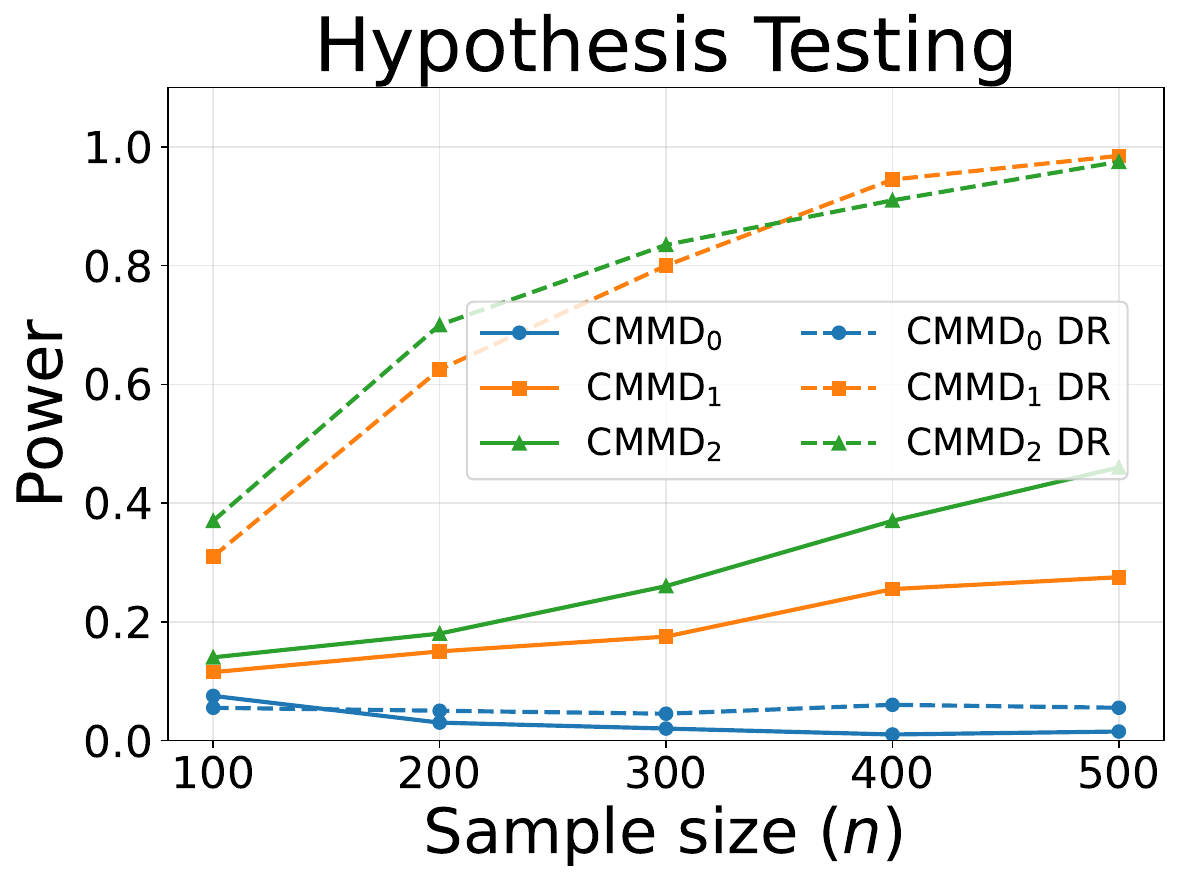}
    \end{subfigure}
    \caption{Left: Plot of difference in CMEs illustrates that the DR estimator gives a closer match to the ground truth. Middle: The DR estimator of CMMD$_1$ converges to the true value faster than the standard counterpart. Right: Rejection rates for hypothesis testing show an improved power when using the DR test statistics (dashed lines).}
    \label{fig:dr_experiment}
\end{figure}

Figure \ref{fig:dr_experiment} (left) shows a plot of the difference between CMEs $\mu_{Y|x} - \mu_{Z|x}$ against $x$. We compare the true value of $0.5x^2$ with the standard and DR estimators which are fitted on $n=500$ samples from each of $P$ and $Q$. Due to the misspecified kernel, simply taking the difference between the standard estimators, $\hat \mu_{Y|x} - \hat \mu_{Z|x}$, is unable to capture the true difference between conditional expectations. On the other hand, using the DR estimator $\hat \Delta_{DR} k(\cdot, x)$ gives a closer match to the ground truth. In Figure \ref{fig:dr_experiment} (middle), we compare the error between the true value of CMMD$_1$ (computed via numerical integration) with estimators $\widehat{\text{CMMD}}_1$ and $\widehat{\text{CMMD}}_{1, DR}$. The latter is estimated using five-fold cross-fitting, with the CME models trained on four folds and used to construct pseudo-outcomes for the held-out fold. The plot shows that the DR estimator has lower error, with error bars indicating the standard deviation from 100 independent trials. Finally, in Figure \ref{fig:dr_experiment} (right) we look at test power, using either the standard or DR estimators (again using cross-fitting) as test statistics. Although power increases slowly when using standard CMMD estimators, the power for DR estimators approaches one much faster. However, the improvement is less effective for CMMD$_0$, with both estimators performing poorly in this setting.

\begin{figure}[t]
    \centering
    \begin{subfigure}[b]{0.48\textwidth}
        \centering
        \includegraphics[width=\linewidth]{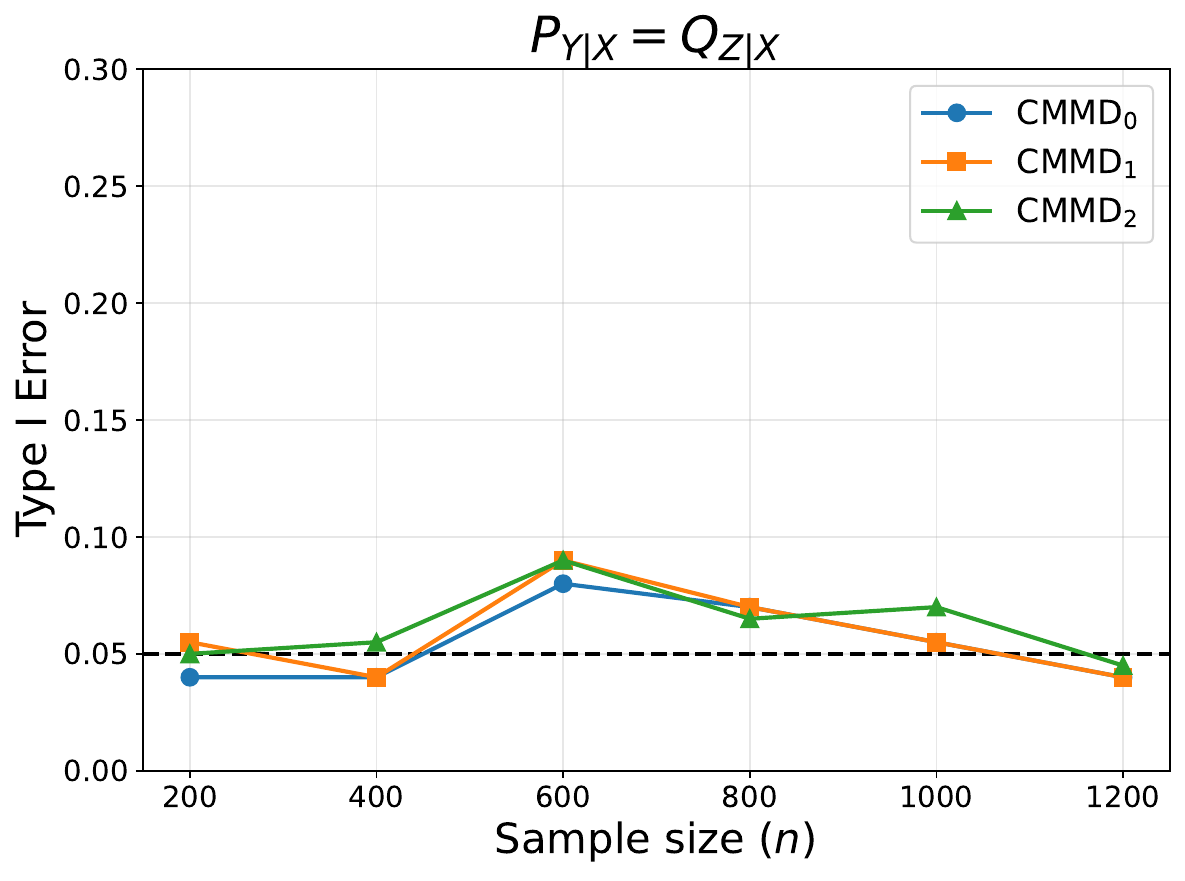}
    \end{subfigure}
    \hfill 
    \begin{subfigure}[b]{0.48\textwidth}
        \centering
        \includegraphics[width=\linewidth]{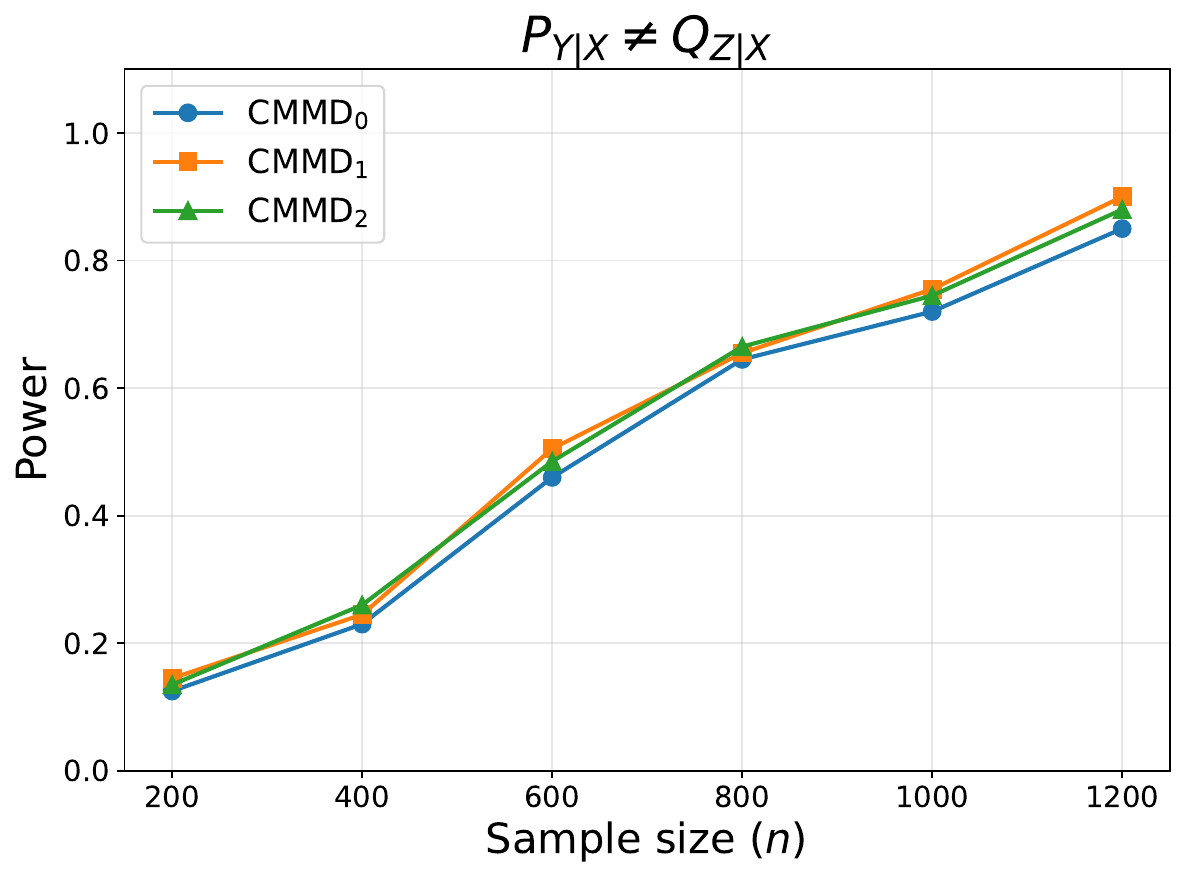}
    \end{subfigure}
    \caption{Rejection rates for CMMD test on MNIST data. Left: Under $H_0$ the rejection rate remains near the significance level (black dashed line). Right: Under $H_1$ the test power approaches one for all three test statistics as sample size increases.}
    \label{fig:mnist_experiment}
\end{figure}

\subsection{Real Data} \label{sec:exp_real}

To demonstrate the applicability of CMMD hypothesis testing on real data, we consider the MNIST digit classification dataset \citep{lecun2010mnist}. We treat the dataset as a population from which we draw samples, letting $X$ be the digit label and $Y$ be the pixel values. To get samples from $P$, we first uniformly sample a digit, and then uniformly sample an image corresponding to that digit. For data from $Q$, we simulate a covariate shift by sampling digits biased towards lower digits. Under the null hypothesis, images are once again sampled uniformly, while under the alternative, we sample images biased towards brighter images, that is, with higher average pixel value. Since $P_X \neq Q_X$, we use Algorithm \ref{alg:2} to perform the hypothesis tests. We select the kernel $k(x,x') = \mathbbm{1}\{x = x'\}$ for the digits, and the Gaussian kernel on the pixel values with bandwidth chosen via the median heuristic \citep{Gretton_2012}. The regularization parameter is set to $\lambda = n^{-\frac{1}{4}}$. The rejection rate plotted in Figure \ref{fig:mnist_experiment} (left) demonstrates that under $H_0$ all three levels of CMMD provide Type I error control. Figure \ref{fig:mnist_experiment} (right) shows that under $H_1$ power approaches one as the sample size increases, with no visible difference between the three CMMD test statistics.

\section{Conclusion}

This paper provides a unified perspective on kernel embedding methods for comparing conditional distributions, which have previously been treated in isolation in the literature. We introduce the \emph{conditional maximum mean discrepancy} (CMMD), comprising a family of metrics with various amounts of smoothing via the covariance operator representing the marginal distribution of the covariate. We show that CMMD has three special cases which recover known metrics, each utilizing a different RKHS representation of conditional distributions: CMMD$_0$ uses conditional mean operators, CMMD$_1$ uses conditional mean embeddings and CMMD$_2$ uses embeddings of joint distributions. We provide empirical estimators for the CMMD, including a novel doubly robust estimator, and show that these can be used for statistical hypothesis testing. Experiments show the practical applicability of our methods. For two-sample hypothesis testing, the CMMD methods demonstrate high test power against the alternative and correct Type I error control under the null. We also give an example of where doubly robust estimators can correctly discern the difference between conditional distributions, even when individual relationships cannot be captured. Future research directions may include a comparison to more complex kernel constructions, including those based on neural representations, as well as other statistical techniques for measuring divergence between conditional distributions with cheaper computational complexity. Moreover, it would be worthwhile to apply the presented CMMD methods to practical settings in causal inference, uncertainty quantification and beyond where determining conditional relationships is important. Lastly, a more thorough analysis of which level provides the optimal amount of smoothing for a given problem would benefit practitioners deciding which of the CMMD metrics to use.



\section*{Disclosure statement}\label{disclosure-statement}

The authors report there are no competing interests to declare.

\section*{Declaration of Generative AI Use}\label{ai-statement}

The authors used ChatGPT (OpenAI; GPT-6 Astra) to improve the clarity and readability of the writing and to assist with writing code for the research. The authors reviewed and revised the suggested text and checked the AI-assisted code for correctness. The authors take full responsibility for the content of the manuscript and the associated code.

\section*{Data Availability Statement}\label{data-availability-statement}

MNIST data used in the experiments are available using the following URL: 

\url{http://yann.lecun.com/exdb/mnist}

\clearpage



\bibliography{bibliography.bib}

\clearpage

\appendix

\phantomsection\label{supplementary-material}
\bigskip

\begin{center}

{\large\bf SUPPLEMENTARY MATERIAL}

\end{center}


\begin{description}
\item[Code:]
Code for performing experiments and reproducing plots is available using the URL:
\url{https://github.com/petermoskvichev/kernel_conditional_distributions}
\end{description}

\section{Extra RKHS Background}

Much of the work presented in this paper relies on the theoretical foundations of kernel methods, for which we provided further background. 

\begin{definition}
Let $\mathcal H$ be a Hilbert space of functions $f:\mathcal X \to \mathbb R$. A function $k: \mathcal X \times \mathcal X \to \mathbb R$ is called a \emph{reproducing kernel} of $\mathcal H$ if it satisfies
\begin{itemize}
    \item for all $x \in \mathcal X$, the function $k(\cdot, x)$  is an element of $\mathcal H$,
    \item for all $x \in \mathcal X$ and all $f \in \mathcal H$, we have that $\langle f, k(\cdot, x) \rangle_{\mathcal H} = f(x)$
\end{itemize}
where the second point is known as the reproducing property. If $\mathcal H$ has a reproducing kernel, it is called a reproducing kernel Hilbert space (RKHS).
\end{definition}

We denote the RKHS associated with the kernel $k$ by $\mathcal H_k$. One can think of $\mathcal H_k$ as a feature space into which a data point may be mapped via the canonical feature map $k(\cdot, x)$. By applying the reproducing property, the inner product between feature mappings can be expressed as
\begin{equation*}
    \langle k(\cdot, x), k(\cdot, x') \rangle_{\mathcal H_k} = k(x,x').
\end{equation*}
The ability to compute inner products without explicit feature representations is known as the ``kernel trick.'' As a result, linear algorithms can be applied to data exhibiting highly nonlinear behavior in the original input space. An additional interpretation of a kernel is as a similarity function between points in $\mathcal X$. However, we can be precise about which functions give valid reproducing kernels.

\begin{definition}
    A symmetric function $k: \mathcal X \times \mathcal X \to \mathbb R$ is called a \emph{positive definite} function if for all  $n \geq 1$, any $a_1, \dots, a_n \in \mathbb R$ and any $x_1, \dots, x_n \in \mathcal X$, we have 
    \begin{equation*}
        \sum_{i}^n \sum_{j}^n a_i a_j k(x_i, x_j) \geq 0.
    \end{equation*}
\end{definition}

It is easy to show that any reproducing kernel is a positive definite function since, by the reproducing property,
\begin{equation*}
    \sum_{i}^n \sum_{j}^n a_i a_j k(x_i, x_j) = \left\| \sum_{i=1}^n a_i k(\cdot, x_i) \right\|_{\mathcal H_k}^2 \geq 0.
\end{equation*}
The Moore-Aronszajn Theorem \citep{aronszajn1950theory} proves the reverse implication that any positive definite function is a reproducing kernel with a unique corresponding RKHS. 

Of particular relevance in this work is the ability to measure discrepancies between probability distributions $P_X$ and $P_Y$ using their RKHS embeddings. Kernel mean embeddings have been defined in Section \ref{sec:kernel_emb} along with the maximum mean discrepancy. An alternative way to define the MMD is as an integral probability metric
\begin{equation*}
    \text{MMD}(P_X, P_{Y}) = \sup_{\|f\|_{\mathcal H_k} \leq 1} (\mathbb E_X[f(X)] - \mathbb E_{Y}[f(Y)])
\end{equation*}
which can be shown to be equal to equation (\ref{eq:mmd}). However, working with equation (\ref{eq:mmd}) provides a more natural way to express the squared MMD as 
\begin{align*}
    \text{MMD}^2(P_X, P_{Y}) &= \| \mu_X - \mu_{Y} \|_{\mathcal H_k}^2 \\
    &= \mathbb E_{X, X'} k(X, X') - 2 \mathbb E_{X, Y} [k(X, Y)] + \mathbb E_{Y, Y'} [k(Y, Y')] 
\end{align*}
where $X$ and $X'$ are independently drawn from $P_X$, and likewise $Y$ and $Y'$ are independently drawn from $P_Y$.

Given two iid samples $\{ x_i \}_{i=1}^n \sim P_X$ and $\{ y_i \}_{i=1}^m \sim P_Y$ one can easily estimate the squared  MMD by replacing expectations with empirical means. A simple unbiased estimator is given by
\begin{equation*}
    \widehat{\text{MMD}}^2 = \frac{1}{n(n-1)} \sum_{i\neq j} k(x_i, x_j) - \frac{2}{nm} \sum_{i=1}^n \sum_{j=1}^m k(x_i, y_j) + \frac{1}{m(m-1)} \sum_{i\neq j} k(y_i, y_j).
\end{equation*}
The scaled squared MMD estimator can be used as a test statistic for comparing two probability distributions. Hypothesis tests can be conducted using a permutation procedure \citep{Gretton_2012}.

\section{Preliminary Results} \label{sec:app_cmo}

This section outlines some preliminary results regarding conditional mean embeddings and operators that will be used in the proofs for theorems in the main paper. We begin with a result that under our assumptions, CMOs are Hilbert-Schmidt operators. 

\begin{theorem} \label{thm:cmo_hs}
Under Assumptions \ref{asm:k_bounded} and \ref{asm:cmo_hs}, the CMO $C_{Y|X}$ is a Hilbert-Schmidt operator. 
\end{theorem}

\begin{proof}
For any $g \in \mathcal H_\ell$, we have that $C_{Y|X}^* g = \mathbb E[g(Y)|X=\cdot]$. Thus, the Assumption \ref{asm:cmo_hs} can be rewritten as $\text{range}(C_{Y|X}^*) \subseteq \text{range}(C_{XX}^\gamma)$. Then by Douglas' lemma, there exists a bounded operator $B: \mathcal H_\ell \to \mathcal H_k$ such that $C_{Y|X}^* = C_{XX}^\gamma B$. Taking adjoints gives $C_{Y|X} = B^* C_{XX}^\gamma$ where we have used the fact that $C_{XX}$ is self adjoint. Let $\{\phi_i \}_{i=1}^\infty$ denote an orthonormal basis for $\mathcal H_k$ consisting of eigenfunctions of $C_{XX}$ with corresponding eigenvalues $\{\sigma_i\}_{i=1}^\infty$. Then
\begin{align*}
    \|C_{Y|X}\|_{\mathcal H_\ell \otimes \mathcal H_k}^2 &= \sum_{i=1}^\infty \| B^* C_{XX}^\gamma \phi_i \|_{\mathcal H_\ell}^2 \\
        &= \sum_{i=1}^\infty \sigma_i^{2 \gamma} \| B^* \phi_i \|_{\mathcal H_\ell}^2 \\
        &\leq \sum_{i=1}^\infty \sigma_i^{2 \gamma} M 
\end{align*}
where $M>0$ is such that for all $f \in \mathcal H_k$, $\|B^* f \|_{\mathcal H_\ell} \leq \sqrt{M} \| f \|_{\mathcal H_k}$. Since $k$ is bounded by Assumption \ref{asm:k_bounded}, $\sum_{i=1}^\infty \sigma_i = Tr(C_{XX})= \mathbb E_X[k(X,X)] \leq k_{\max}$. Therefore, $C_{XX}$ is trace class, and the sum converges whenever $\gamma \geq \frac{1}{2}$. $C_{Y|X}$ is Hilbert-Schmidt. 
\end{proof}

Next, we show that the covariance operator for the marginal distribution $R_X = \rho P_X + (1-\rho) Q_X$ with corresponding kernel $k$ has positive eigenvalues. 
\begin{lemma} \label{lem:cov_op_pos}
    Under Assumptions \ref{asm:full_support} and \ref{asm:k_bounded}, the covariance operator $C_{XX}$ has positive eigenvalues.
\end{lemma}
\begin{proof}
Let $f \in \mathcal H_k$. Then 
\begin{equation*}
    \langle C_{XX} f , f \rangle_{\mathcal H_k} = \mathbb{E}[f(X)^2] = \int_{\mathcal X} f(x)^2 dR_X(x).
\end{equation*}
Suppose for sake of contradiction that there exists $f \neq 0$ such that $\langle C_{XX} f , f \rangle_{\mathcal H_k} = 0$. This implies that $f(x) = 0$ for $R_X$ almost all $x \in \mathcal X$. Let $\mathcal A = \{ x : f(x) \neq 0\} $ which is non-empty but has measure zero. Since $k$ is continuous, any $f \in \mathcal H_k$ must also be continuous. By continuity of $f$, $\mathcal A$ is an open set. However, $R_X$ has full support by assumption, which means any nonempty open subset of $\mathcal X$ must have positive measure. We have a contradiction. 

Therefore, $\langle C_{XX} f , f \rangle_{\mathcal H_k} > 0$ for all $f \neq 0$ and $C_{XX}$ must have positive eigenvalues. 
\end{proof}

An immediate consequence of Lemma \ref{lem:cov_op_pos} is that the operators $C_{XX}$ and $C_{XX}^{\frac{1}{2}}$ are injective. 

The next two results are related to the convergence of CMO and CME estimators $\hat C_{Y|X}$ and $\hat \mu_{Y|x}$ given in Section \ref{sec:back_cmo}. 

\begin{theorem} \label{thm:cmo_est_conv}
    Suppose Assumptions \ref{asm:k_bounded} and \ref{asm:cmo_hs} hold. If the regularization term $\lambda$ satisfies $\lambda \to 0$ and $n\lambda^3 \to \infty$, then $\|\hat C_{Y|X} - C_{Y|X} \|_{\mathcal H_\ell \otimes \mathcal H_k} \overset{p}{\to} 0$. 
\end{theorem}

\begin{proof}
First note that $C_{Y|X}$ is Hilbert-Schmidt following Theorem \ref{thm:cmo_hs}. Next, we define a regularised population operator
\begin{equation*}
    C_{Y|X}^\lambda = C_{YX}(C_{XX}+\lambda I)^{-1}
\end{equation*}
allowing us to make the decomposition
\begin{equation} \label{eq:cmo_est_decomp}
    \|\hat C_{Y|X} - C_{Y|X} \|_{\mathcal H_k \otimes \mathcal H_\ell} \leq \|\hat C_{Y|X} - C_{Y|X}^\lambda \|_{\mathcal H_k \otimes \mathcal H_\ell} + \|C_{Y|X}^\lambda - C_{Y|X} \|_{\mathcal H_k \otimes \mathcal H_\ell}.
\end{equation}
We will call the first term estimation error and the second term bias error. Starting with the estimation error, following \cite{Song_2010} we have that $\|\hat C_{Y|X} - C_{Y|X}^\lambda \|_{\mathcal H_k \otimes \mathcal H_\ell} = O_p(\lambda^{-3/2} n^{-1/2})$. 

For the bias error, 
\begin{align*}
    \|C_{Y|X}^\lambda - C_{Y|X} \|_{\mathcal H_k \otimes \mathcal H_\ell} &= \|C_{YX}(C_{XX}+\lambda I)^{-1} - C_{Y|X} \|_{\mathcal H_k \otimes \mathcal H_\ell} \\
    &= \|C_{Y|X} C_{XX}(C_{XX}+\lambda I)^{-1} - C_{Y|X} \|_{\mathcal H_k \otimes \mathcal H_\ell} \\
    &= \|C_{Y|X} (C_{XX}(C_{XX}+\lambda I)^{-1} - I) \|_{\mathcal H_k \otimes \mathcal H_\ell} \\
    &= \|C_{Y|X} (C_{XX}(C_{XX}+\lambda I)^{-1} - (C_{XX}+\lambda I)(C_{XX}+\lambda I)^{-1}) \|_{\mathcal H_k \otimes \mathcal H_\ell} \\
    &= \|C_{Y|X} (-\lambda (C_{XX}+\lambda I)^{-1}) \|_{\mathcal H_k \otimes \mathcal H_\ell} \\
    &= \|\lambda C_{Y|X} (C_{XX}+\lambda I)^{-1} \|_{\mathcal H_k \otimes \mathcal H_\ell}
\end{align*}
Next, we express the covariance operator in terms of the complete orthonormal system 
\begin{equation*}
    C_{XX} = \sum_{i=1}^\infty \sigma_i \phi_i \otimes \phi_i
\end{equation*}
where $\{\phi_i\}_{i=1}^\infty$ is an orthonormal basis for $\mathcal H_k$ consisting of eigenfunctions of $C_{XX}$ with corresponding eigenvalues $\sigma_i > 0$. Note that the positivity of eigenvalues follows from Lemma \ref{lem:cov_op_pos}. This lets us write the squared Hilbert-Schmidt norm as
\begin{align*}
    \|\lambda C_{Y|X} (C_{XX}+\lambda I)^{-1} \|_{\mathcal H_k \otimes \mathcal H_\ell}^2 &= \sum_{i=1}^\infty \|\lambda C_{Y|X} (C_{XX}+\lambda I)^{-1} \phi_i \|_{\mathcal H_\ell}^2 \\
    &= \sum_{i=1}^\infty \|\frac{\lambda}{\sigma_i + \lambda} C_{Y|X} \phi_i \|_{\mathcal H_\ell}^2 \\
    &= \sum_{i=1}^\infty \left( \frac{\lambda}{\sigma_i + \lambda} \right)^2 \|C_{Y|X} \phi_i \|_{\mathcal H_\ell}^2.
\end{align*}
To show that the series converges to zero, we use the Dominated Convergence Theorem. Let 
\begin{equation*}
    a_i = \left( \frac{\lambda}{\sigma_i + \lambda} \right)^2 \|C_{Y|X} \phi_i \|_{\mathcal H_\ell}^2.
\end{equation*}
Note that for $\sigma_i > 0$, we have $\lambda/(\sigma_i + \lambda) \to 0$ as $\lambda \to 0$. Thus, for every $i \in \mathbb{N}$, we get $a_i \to 0$ as $\lambda \to 0$ and so the terms converge pointwise to zero. Furthermore, observe that for any $\lambda >0$, 
\begin{equation*}
    \frac{\lambda}{\sigma_i + \lambda} < 1, 
\end{equation*}
so that $a_i \leq \|C_{Y|X} \phi_i \|_{\mathcal H_\ell}^2$ for all $i$. Since $C_{Y|X}$ is Hilbert-Schmidt, the sum $\sum_{i=1}^\infty \|C_{Y|X} \phi_i \|_{\mathcal H_\ell}^2 < \infty$, which means the bounding function is summable. Thus, by the DCT,
\begin{equation*}
    \lim_{\lambda \to 0} \|\lambda C_{Y|X} (C_{XX}+\lambda I)^{-1} \|_{\mathcal H_k \otimes \mathcal H_\ell}^2 = 0
\end{equation*}
which implies the bias term in (\ref{eq:cmo_est_decomp}) converges to zero. Combining the results, we have that when $\lambda \to 0$ and $n \lambda^3 \to \infty$, $\|\hat C_{Y|X} - C_{Y|X} \|_{\mathcal H_k \otimes \mathcal H_\ell}$ converges to zero in probability as required. 
\end{proof}

\begin{theorem} \label{thm:cme_conv}
Suppose Assumptions \ref{asm:k_bounded}-\ref{asm:cmo_hs} hold. Then 
\begin{equation*}
    \|\mu_{Y|x} - \hat \mu_{Y|x} \|_{\mathcal H_\ell} = O_p(\lambda^{-1}n^{-\frac{1}{2}} + \lambda^{\frac{1}{2}}).
\end{equation*}
\end{theorem}

\begin{proof}
We express the left hand side as 
\begin{align*}
    \|\mu_{Y|x} - \hat \mu_{Y|x} \|_{\mathcal H_\ell} &= \|(C_{Y|X} - \hat C_{Y|X}) k(\cdot, x) \|_{\mathcal H_\ell} \\
    &\leq \sqrt{k_{\max}} \left( \|C_{Y|X}-C_{Y|X}^\lambda\|_{op} +\|C_{Y|X}^\lambda-\hat C_{Y|X}\|_{op} \right).
\end{align*}
By Assumption \ref{asm:cmo_hs} and Douglas' lemma, $C_{Y|X}=B^*C_{XX}^{\gamma}$ for some bounded operator $B:\mathcal H_\ell\to\mathcal H_k$. Define $C_{Y|X}^\lambda =C_{YX}(C_{XX}+\lambda I)^{-1} = C_{Y|X}C_{XX}(C_{XX}+\lambda I)^{-1}$. Then $\|C_{Y|X}^\lambda\|_{op} \leq \|C_{Y|X}\|_{\mathrm{op}}$ and, using the resolvent identity,
\begin{align*}
    \hat C_{Y|X}-C_{Y|X}^\lambda &= [\hat C_{YX} - C_{YX} (C_{XX}+\lambda I)^{-1} (\hat C_{XX}+\lambda I) ] (\hat C_{XX}+\lambda I)^{-1} \\
    &= [\hat C_{YX} - C_{YX} ((\hat C_{XX}+\lambda I)^{-1} \\
    & \quad \quad \quad - (C_{XX}+\lambda I)^{-1}(C_{XX} - \hat C_{XX}) (\hat C_{XX}+\lambda I)^{-1}) (\hat C_{XX}+\lambda I) ] (\hat C_{XX}+\lambda I)^{-1} \\
    &= [ \hat C_{YX}-C_{YX} - C_{Y|X}^\lambda(\hat C_{XX}-C_{XX})](\hat C_{XX}+\lambda I)^{-1}
\end{align*}
Furthermore, the Hilbert space sample mean variance identity gives $\|\hat C_{YX}-C_{YX}\|_{op} +\|\hat C_{XX}-C_{XX}\|_{op} =O_p(n^{-1/2})$, and $\|(\hat C_{XX} + \lambda I)^{-1}\|_{op} \leq \lambda$. Therefore,
\begin{equation*}
    \|C_{Y|X}^\lambda-\hat C_{Y|X}\|_{op} \leq (\|\hat C_{YX}-C_{YX}\|_{op} +\|C_{Y|X}\|_{op} \|\hat C_{XX}-C_{XX}\|_{op}) \lambda ^{-1} = O_p(\lambda^{-1}n^{-\frac{1}{2}}).
\end{equation*}
As for the bias term,
\begin{align*}
    \|C_{Y|X}-C_{Y|X}^\lambda\|_{op} &= \|C_{Y|X}[I - C_{XX} (C_{XX}+\lambda I)^{-1}]\|_{op} \\
    &= \|C_{Y|X}[I - C_{XX} (C_{XX}+\lambda I)^{-1}]\|_{op} \\
    &= \| \lambda B^* C_{XX}^\gamma (C_{XX} + \lambda I)^{-1} \|_{op} \\
    &\leq \lambda \|B\|_{op} \sup_{t\geq 0} \frac{t^\gamma}{t + \lambda} \\
    &= O_p(\lambda^{1/2}).
\end{align*}
Combining the results gives $\|\mu_{Y|x} - \hat \mu_{Y|x} \|_{\mathcal H_\ell} = O_p(\lambda^{-1}n^{-\frac{1}{2}} + \lambda^{\frac{1}{2}})$.
\end{proof}

\subsection{Illustrative Example} \label{sec:toy_ex}

Despite all of the CMMD levels being valid metrics between conditional distributions under our assumptions, they are measuring different quantities. Thus, depending on the context, a specific level may be preferred. To demonstrate this, we provide a simple illustrative example for model selection. Suppose $\mathcal X = \{0,1,2\}$ and $\mathcal Y = \{0,1\}$ and we choose the kernels $k(x,x') = \mathbbm{1}\{x = x'\}$ and $\ell(y,y') = \mathbbm{1}\{y = y'\}$. In this case CMOs correspond to conditional probability tables, CMEs are conditional probability vectors and cross-covariance operators are joint probability tables \citep{Song_2013}. Consider a conditional distribution $P_{Y|X}$ and marginal distribution $P_X$ with RKHS embeddings
\[C_{Y|X} = \begin{bmatrix} 0.4 & 0.5 & 0.6 \\ 0.6 & 0.5 & 0.4 \end{bmatrix} \quad \text{and} \quad 
\mu_X = (0.3, 0.6, 0.1 )^\top\]
such that $[C_{Y|X}]_{ij} = P(Y=i|X=j)$ and $[\mu_X]_j = P(X=j)$ for $i \in \{0,1\}$ and $j \in \{0,1,2\}$. The resulting joint mean embedding for $P_{XY}$ is
\[\mu_{XY} = \begin{bmatrix} 0.12 & 0.3 & 0.06 \\ 0.18 & 0.3 & 0.04 \end{bmatrix}\]
where $[\mu_{XY}]_{ij} = P(X=j, Y=i)$. Consider three candidate models $Q^1$, $Q^2$ and $Q^3$ with corresponding CMOs:
\[
C^1_{Z|X} = \begin{bmatrix} 0.3 & 0.4 & 0.5 \\ 0.7 & 0.6 & 0.5 \end{bmatrix} \qquad 
C^2_{Z|X} = \begin{bmatrix} 0.3 & 0.5 & 0.8 \\ 0.7 & 0.5 & 0.2 \end{bmatrix}
\qquad 
C^3_{Z|X} = \begin{bmatrix} 0.4 & 0.5 & 0.9 \\ 0.6 & 0.5 & 0.1 \end{bmatrix}
\]
with $[C^m_{Z|X}]_{ij} = Q^m(Z=i|X=j)$ for $m \in \{1,2,3\}$. We assume that the distribution of covariates under all three models match $P_X$. Thus, multiplying each column of the above matrices by the corresponding value of $\mu_X$ gives joint mean embeddings $\mu^m_{XZ}$ of the distributions $Q^m_{XZ}$ which again is in the form of a $2\times 3$ matrix. The goal is to find the model whose conditional distribution most closely matches $P_{Y|X}$. 

\begin{table}[t]
    \centering
    \caption{Values of squared CMMD for each model, with the smallest value in each column in bold.}
    \begin{tabular}{cr|ccc}
    \toprule
                & & CMMD$_0^2$ & CMMD$_1^2$ & CMMD$_2^2$ \\ \hline
         & $Q^1$ & \textbf{0.06} & 0.020 & 0.0092 \\
    Model& $Q^2$ & 0.10 & \textbf{0.014} & 0.0026 \\
         & $Q^3$ & 0.18 & 0.018 & \textbf{0.0018} \\
    \bottomrule
    \end{tabular}
    \label{tab:toy}
\end{table}

The squared Frobenius norm between $C_{Y|X}$ and $C^m_{Z|X}$ produces the CMMD$_0^2$, the squared $L_2$ norm between columns of $C_{Y|X}$ and $C^m_{Z|X}$ weighted by corresponding values of $\mu_X$ produces CMMD$_1^2$, and the squared Frobenius norm between $\mu_{XY}$ and $\mu^m_{XZ}$ produces CMMD$_2^2$. The results are summarized in Table \ref{tab:toy}. Depending on the chosen metric a different model is preferred, demonstrating that the notion of ``best conditional model'' is not well-posed.

\section{Proofs}

Here we present proofs for lemmas and theorems stated in the main body of this work.

\subsection{CMMD Metrics}

We begin this subsection with a crucial result: CMMD$_0$ forms a valid metric between conditional distributions.

\textbf{Proof of Theorem \ref{thm:cmmd0_metric}}
\begin{proof}
By \citet[Theorem 5.2]{park_measure_2020}, under Assumptions \ref{asm:full_support}, \ref{asm:regular} and \ref{asm:ell_char}, we have that $P_{Y|X} = Q_{Z|X}$ almost surely if and only if $\mathbb E_X \| \mu_{Y|X} - \mu_{Z|X} \|_{\mathcal H_\ell}^2 = 0$ where the almost sure equality and the expectation is with respect to $R_X$. Thus, all that is required is to show that $\mathbb E_X \| \mu_{Y|X} - \mu_{Z|X} \|_{\mathcal H_\ell}^2 = 0$ if and only if $\| C_{Y|X} - C_{Z|X}\|_{\mathcal H_\ell \otimes \mathcal H_k} = 0$.

By Assumptions \ref{asm:k_bounded} and \ref{asm:cmo_hs}, we know that CMOs exist and we can express the difference between CMEs as 
\begin{equation*}
    \mu_{Y|X} - \mu_{Z|X} = C_{Y|X} k(\cdot, X) - C_{Z|X} k(\cdot, X) = \Delta k(\cdot, X)
\end{equation*}
where $\Delta = C_{Y|X} - C_{Z|X}$. By Theorem \ref{thm:cmmd1} we can express
\begin{equation*}
    \mathbb E_X \| \mu_{Y|X} - \mu_{Z|X} \|_{\mathcal H_\ell}^2 = Tr(\Delta^* \Delta C_{XX}) = Tr(C_{XX}^{\frac{1}{2}} \Delta^* \Delta C_{XX}^{\frac{1}{2}}) = \|\Delta C_{XX}^{\frac{1}{2}} \|_{\mathcal H_\ell \otimes \mathcal H_k}^2
\end{equation*}
where $C_{XX}$ is the covariance operator corresponding to $R_X$ with kernel $k$ and $C_{XX}^{\frac{1}{2}} : \mathcal H_k \to \mathcal H_k$ is such that $C_{XX} = C_{XX}^{\frac{1}{2}}C_{XX}^{\frac{1}{2}}$ which exists due to $C_{XX}$ being positive and self-adjoint. Furthermore, since $R_X$ has full support and $k$ is continuous, $C_{XX}$ has positive eigenvalues (see Lemma \ref{lem:cov_op_pos}) and hence is injective. This implies that $C_{XX}^{\frac{1}{2}}$ is also injective, along with inheriting boundedness and self-adjointness from $C_{XX}$. 

It is clear that if $\| C_{Y|X} - C_{Z|X}\|_{\mathcal H_\ell \otimes \mathcal H_k} = \| \Delta \|_{\mathcal H_\ell \otimes \mathcal H_k} = 0$, then $\|\Delta C_{XX}^{\frac{1}{2}} \|_{\mathcal H_\ell \otimes \mathcal H_k} = 0$. Now we prove the reverse implication. We apply two results from functional analysis. First, since $C_{XX}^{\frac{1}{2}}$ is a bounded, self-adjoint and injective operator, $\text{range}(C_{XX}^{\frac{1}{2}})$ is dense in $\mathcal H_k$. Second, since $\Delta = C_{Y|X} - C_{Z|X}$ is bounded and $\|\Delta f \|_{\mathcal H_\ell} = 0$ for all $f \in \text{range}(C_{XX}^{\frac{1}{2}})$ with $\text{range}(C_{XX}^{\frac{1}{2}})$ dense in $\mathcal H_k$, this implies that $\|\Delta f \|_{\mathcal H_\ell} = 0$ for all $f \in \mathcal H_k$. Thus, $\|\Delta C_{XX}^{\frac{1}{2}} \|_{\mathcal H_\ell \otimes \mathcal H_k} = 0$ if and only if $\|\Delta \|_{\mathcal H_\ell \otimes \mathcal H_k} = 0$. 
\end{proof}

Next, we provide derivations of CMMD$_1$ and CMMD$_2$ metrics in terms of the operators $\Delta = C_{Y|X} - C_{Z|X}$ and $C_{XX}$.  

\textbf{Proof of Theorem \ref{thm:cmmd1}}
\begin{proof}
Let $C_{Y|X}$ and $C_{Z|X}$ be the CMOs representing distributions $P_{Y|X}$ and $Q_{Z|X}$, such that $\mu_{Y|X} = C_{Y|X} k(\cdot, X)$ and $\mu_{Z|X} = C_{Z|X} k(\cdot, X)$. Then
\begin{align*}
\mathbb{E}_{X}\|\mu_{Y|X} - \mu_{Z|X}\|_{\mathcal H_{\ell}}^2 &= \mathbb{E}_{X}\|C_{Y|X} k(\cdot, X) - C_{Z|X} k(\cdot, X)\|_{\mathcal H_{\ell}}^2 \\
    &= \mathbb{E}_{X}\|\Delta k(\cdot, X) \|_{\mathcal H_{\ell}}^2 \\
    &= \mathbb{E}_{X} \langle \Delta k(\cdot, X), \Delta k(\cdot, X) \rangle_{\mathcal H_{\ell}} \\
    &= \mathbb{E}_{X} \langle \Delta^*\Delta k(\cdot, X), k(\cdot, X) \rangle_{\mathcal H_{k}} \\
    &= \mathbb{E}_{X} \langle \Delta^*\Delta , k(\cdot, X) \otimes k(\cdot, X) \rangle_{\mathcal H_{k} \otimes \mathcal H_{k}} \\
    &= \langle \Delta^*\Delta , \mathbb{E}_{X} [k(\cdot, X) \otimes k(\cdot, X)] \rangle_{\mathcal H_{k} \otimes \mathcal H_{k}} \\
    &= \langle \Delta^*\Delta , C_{XX}\rangle_{\mathcal H_{k} \otimes \mathcal H_{k}} \\
    &= Tr(\Delta^* \Delta C_{XX})
\end{align*}
where we have used properties of adjoint operators and that $C_{XX}$ is self adjoint.
\end{proof}

\textbf{Proof of Theorem \ref{thm:cmmd2}}
\begin{proof}
The joint embeddings $\mu_{XY}$ and $\mu_{XZ}$ are isomorphic to the cross-covariance operators $C_{XY}$ and $C_{XZ}$. Thus,
\begin{align*}
\|\mu_{XY} - \mu_{XZ}\|_{\mathcal H_{k \otimes \ell}}^2 &= \|C_{XY} - C_{XZ}\|_{\mathcal H_{k} \otimes \mathcal H_{\ell}}^2 \\
    &= \|C_{YX} - C_{ZX}\|_{\mathcal H_{\ell} \otimes \mathcal H_{k}}^2 \\
    &= \|C_{Y|X}C_{XX} - C_{Z|X}C_{XX}\|_{\mathcal H_{\ell} \otimes \mathcal H_{k}}^2 \\
    &= \|\Delta C_{XX} \|_{\mathcal H_{\ell} \otimes \mathcal H_{k}}^2 \\
    &= \langle \Delta C_{XX}, \Delta C_{XX} \rangle_{\mathcal H_{\ell} \otimes \mathcal H_{k}} \\ 
    &= \langle \Delta^* \Delta, C_{XX} C_{XX}^* \rangle_{\mathcal H_{k} \otimes \mathcal H_{k}} \\
    &= \langle \Delta^* \Delta, C_{XX}^2 \rangle_{\mathcal H_{k} \otimes \mathcal H_{k}} \\
    &= Tr(\Delta^* \Delta C_{XX}^2)
\end{align*}
where the third equality is due to the kernel chain rule $C_{YX} = C_{Y|X}C_{XX}$ \citep{Song_2009}. 
\end{proof}

We conclude with two theorems regarding the connection between the CMMD metrics.

\textbf{Proof of Theorem \ref{thm:cmmd_rel}}
\begin{proof}
    We start by decomposing the covariance operator $C_{XX} = A A^*$ where $A: L_2(R_X) \to \mathcal H_k$ is Hilbert-Schmidt since $C_{XX}$ is trace class. Then
    \begin{align*}
        \text{CMMD}_2^2(P_{Y|X}, Q_{Z|X}) &= Tr(\Delta^* \Delta C_{XX}^2) \\
            &= \|\Delta C_{XX}\|_{\mathcal H_{\ell} \otimes \mathcal H_{k}}^2 \\
            &= \|\Delta A A^*\|_{HS}^2 \\
            &\leq \|A^*\|_{HS}^2 \|\Delta A \|_{HS}^2 \\
            &= Tr(A A^*) Tr(A^* \Delta^* \Delta A) \\
            &= Tr(C_{XX}) Tr(\Delta^* \Delta C_{XX}) \\
            &= Tr(C_{XX}) \text{CMMD}_1^2(P_{Y|X}, Q_{Z|X})
    \end{align*}
    where we have used the sub-multiplicativity of the Hilbert-Schmidt norm in the fourth line, as well as Theorems \ref{thm:cmmd1} and \ref{thm:cmmd2} to express CMMD as traces. Using a similar procedure, we get the inequality
    \begin{align*}
        \text{CMMD}_1^2(P_{Y|X}, Q_{Z|X}) &= Tr(\Delta^* \Delta C_{XX}) \\
            &= Tr(A^* \Delta^* \Delta A) \\
            &= \|\Delta A\|_{HS}^2 \\
            &\leq \|A\|_{HS}^2 \|\Delta \|_{HS}^2 \\
            &= Tr(A^* A) Tr(\Delta^* \Delta) \\
            &= Tr(A A^*) Tr(\Delta^* \Delta) \\
            &= Tr(C_{XX}) \text{CMMD}_0^2(P_{Y|X}, Q_{Z|X}).
    \end{align*}
    Next we express the trace of $C_{XX}$ as
    \begin{align*}
        Tr(C_{XX}) &= Tr(\mathbb{E}_{X} [k(\cdot, X) \otimes k(\cdot, X)]) \\
        &= \mathbb{E}_{X} [Tr(k(\cdot, X) \otimes k(\cdot, X))] \\
        &= \mathbb{E}_{X} [\langle k(\cdot, X), k(\cdot, X) \rangle] \\
        &=  \mathbb{E}_{X} [k(X,X)]
    \end{align*}
    which after substitution leads to the desired result. 
\end{proof}

\textbf{Proof of Theorem \ref{thm:cmmd_rel_sigma}}
\begin{proof}
Applying equation \eqref{eq:cmmd_s}, 
    \begin{align*}
        \text{CMMD}_s^2(P_{Y|X}, Q_{Z|X}) &= \sum_i \sigma_i^s \| \Delta\phi_i \|_{\mathcal H_\ell}^{2} \\
        &= \sum_i \sigma_i^{s-s'} \sigma_i^{s'} \| \Delta\phi_i \|_{\mathcal H_\ell}^{2} \\
        &\leq \sigma_{\max}^{s-s'} \text{CMMD}_{s'}^2(P_{Y|X}, Q_{Z|X})
    \end{align*}
    where $\sigma_{\max} = \|C_{XX}\|_{op}$ is the largest eigenvalue of $C_{XX}$. Setting $s=1,2$ and $s'=0,1$, we get the special case
    \begin{equation*}
        \textup{CMMD}_2^2(P_{Y|X}, Q_{Z|X}) \leq \sigma_{\max} \textup{CMMD}_1^2(P_{Y|X}, Q_{Z|X}) \leq \sigma_{\max}^2 \textup{CMMD}_0^2(P_{Y|X}, Q_{Z|X}).
    \end{equation*}
\end{proof}

\subsection{Estimators}

This subsection contains proofs regarding the naive and doubly robust estimators for CMMD.

\textbf{Proof of Lemma \ref{lem:cmmd0_conv}}
\begin{proof}
    We want to show that the estimator $\widehat{\textup{CMMD}}_0^2 = \| \hat{C}_{Y|X} - \hat{C}_{Z|X} \|_{\mathcal H_\ell \otimes \mathcal H_k}^2$ converges in probability to the population value $\text{CMMD}_0^2(P_{Y|X}, Q_{Z|X}) =  \| C_{Y|X} - C_{Z|X} \|_{\mathcal H_\ell \otimes \mathcal H_k}^2$. For ease of notation, let $\Delta = C_{Y|X} - C_{Z|X}$ and $\hat \Delta = \hat{C}_{Y|X} - \hat{C}_{Z|X}$. Then 
    \begin{align*}
        \left| \| \hat{\Delta} \|_{\mathcal H_\ell \otimes \mathcal H_k}^2 - \| \Delta \|_{\mathcal H_\ell \otimes \mathcal H_k}^2 \right| &= \left| \| \hat{\Delta} \|_{\mathcal H_\ell \otimes \mathcal H_k} - \| \Delta \|_{\mathcal H_\ell \otimes \mathcal H_k} \right| \left( \| \hat{\Delta} \|_{\mathcal H_\ell \otimes \mathcal H_k} + \| \Delta \|_{\mathcal H_\ell \otimes \mathcal H_k} \right) \\
        &\leq \| \hat \Delta - \Delta \|_{\mathcal H_\ell \otimes \mathcal H_k} ( \| \hat \Delta - \Delta \|_{\mathcal H_\ell \otimes \mathcal H_k} + 2 \| \Delta \|_{\mathcal H_\ell \otimes \mathcal H_k}) \\
        &= \| \hat \Delta - \Delta \|_{\mathcal H_\ell \otimes \mathcal H_k}^2 + 2 \| \Delta \|_{\mathcal H_\ell \otimes \mathcal H_k} \| \hat \Delta - \Delta \|_{\mathcal H_\ell \otimes \mathcal H_k}.
    \end{align*}
    Next we consider
    \begin{align*}
        \| \hat \Delta - \Delta \|_{\mathcal H_\ell \otimes \mathcal H_k} &= \left\| \hat{C}_{Y|X} - C_{Y|X} + C_{Z|X} - \hat{C}_{Z|X} \right\|_{\mathcal H_\ell \otimes \mathcal H_k} \\
        &\leq \| \hat{C}_{Y|X} - C_{Y|X} \|_{\mathcal H_\ell \otimes \mathcal H_k} + \| \hat{C}_{Z|X} - C_{Z|X} \|_{\mathcal H_\ell \otimes \mathcal H_k}
    \end{align*}
    where we have applied the triangle inequality. By our assumptions, Theorem \ref{thm:cmo_est_conv} applies and the CMO estimators converge in Hilbert-Schmidt norm to the population quantity. This means that $\| \hat \Delta - \Delta \|_{\mathcal H_\ell \otimes \mathcal H_k} \overset{p}{\to} 0$ as well as $\| \hat \Delta - \Delta \|_{\mathcal H_\ell \otimes \mathcal H_k}^2 \overset{p}{\to} 0$. Lastly, as $\| \Delta \|_{\mathcal H_\ell \otimes \mathcal H_k}$ is finite, this implies 
    \begin{equation*}
        \left| \| \hat{\Delta} \|_{\mathcal H_\ell \otimes \mathcal H_k}^2 - \| \Delta \|_{\mathcal H_\ell \otimes \mathcal H_k}^2 \right| \overset{p}{\to} 0.
    \end{equation*}
\end{proof}

Next, we show how to derive closed form expressions for the standard CMMD estimators.

\textbf{Proof of Lemma \ref{lem:cmmd0_est_closed}}
\begin{proof}
Using standard CMO estimators, we get
\begin{align*}
    \widehat{\textup{CMMD}}_0^2 &= \| \widehat{C}_{Y|X} - \widehat{C}_{Z|X} \|_{\mathcal H_\ell \otimes \mathcal H_k}^2 \\
    &= Tr( (\widehat C_{Y|X}^* - \widehat C_{Z|X}^*) (\widehat C_{Y|X} - \widehat C_{Z|X})) \\
    &= Tr((\Phi_\mathbf{X} W_\mathbf{X} \Psi_\mathbf{Y}^* - \Phi_\mathbf{X'} W_\mathbf{X'} \Psi_{\mathbf{Z}}^*)(\Psi_\mathbf{Y} W_\mathbf{X} \Phi_\mathbf{X}^* - \Psi_{\mathbf{Z}} W_\mathbf{X'} \Phi_\mathbf{X'}^*)) \\
    &= Tr(\Phi_\mathbf{X} W_\mathbf{X} L_\mathbf{YY} W_\mathbf{X} \Phi_\mathbf{X}^* - \Phi_\mathbf{X} W_\mathbf{X} L_\mathbf{YZ} W_\mathbf{X'} \Phi_\mathbf{X'}^* \\
    &\qquad \qquad - \Phi_\mathbf{X'} W_\mathbf{X'} L_\mathbf{ZY} W_{\mathbf X} \Phi_\mathbf{X}^* + \Phi_\mathbf{X'} W_\mathbf{X'} L_\mathbf{ZZ} W_\mathbf{X'} \Phi_\mathbf{X'}^*) \\
    &= Tr(W_\mathbf{X} L_\mathbf{YY} W_\mathbf{X} K_\mathbf{XX}) - 2Tr(W_\mathbf{X} L_\mathbf{YZ} W_\mathbf{X'} K_\mathbf{X' X}) + Tr(W_\mathbf{X'} L_\mathbf{ZZ} W_\mathbf{X'} K_\mathbf{X' X'}).
\end{align*}
\end{proof}

\clearpage
\textbf{Proof of Lemma \ref{lem:cmmd1_est_closed}}
\begin{proof}
The CME estimator can be expressed as $\hat \mu_{Y|x} = \widehat C_{Y|X} k(\cdot, x)$. Then 
\begin{align*}
    \widehat{\text{CMMD}}_1^2 &= \frac{1}{n+m} \sum_{i=1}^{n+m} \| \hat \mu_{Y|\tilde x_i} - \hat \mu_{Z|\tilde x_i} \|_{\mathcal H_\ell}^2 \\
    &= \frac{1}{n+m} \sum_{i=1}^{n+m} \| (\widehat{C}_{Y|X} - \widehat{C}_{Z|X}) k(\cdot, \tilde x_i) \|_{\mathcal H_\ell}^2 \\
    &= \frac{1}{n+m} \sum_{i=1}^{n+m} \langle (\widehat{C}_{Y|X} - \widehat{C}_{Z|X}) k(\cdot, \tilde x_i), (\widehat{C}_{Y|X} - \widehat{C}_{Z|X}) k(\cdot, \tilde x_i) \rangle_{\mathcal H_\ell} \\
    &= \frac{1}{n+m} \sum_{i=1}^{n+m} \langle (\widehat{C}_{Y|X}^* - \widehat{C}_{Z|X}^*)(\widehat{C}_{Y|X} - \widehat{C}_{Z|X}) , k(\cdot, \tilde x_i) \otimes k(\cdot, \tilde x_i) \rangle_{\mathcal H_k \otimes \mathcal H_k} \\
    &= \frac{1}{n+m} \langle (\widehat{C}_{Y|X}^* - \widehat{C}_{Z|X}^*)(\widehat{C}_{Y|X} - \widehat{C}_{Z|X}) , \Phi_{\mathbf{\tilde X}} \Phi_{\mathbf{\tilde X}}^* \rangle_{\mathcal H_k \otimes \mathcal H_k} \\
    &= \frac{1}{n+m} Tr((\widehat{C}_{Y|X}^* - \widehat{C}_{Z|X}^*)(\widehat{C}_{Y|X} - \widehat{C}_{Z|X}) \Phi_{\mathbf{\tilde X}} \Phi_{\mathbf{\tilde X}}^* ) \\
    &= \frac{1}{n+m} Tr((\Phi_\mathbf{X} W_\mathbf{X} L_\mathbf{YY} W_\mathbf{X} \Phi_\mathbf{X}^* - \Phi_\mathbf{X} W_\mathbf{X} L_\mathbf{YZ} W_\mathbf{X'} \Phi_\mathbf{X'}^* \\
    &\qquad \qquad \qquad - \Phi_\mathbf{X'} W_\mathbf{X'} L_\mathbf{ZY} W_{\mathbf X} \Phi_\mathbf{X}^* + \Phi_\mathbf{X'} W_\mathbf{X'} L_\mathbf{ZZ} W_\mathbf{X'} \Phi_\mathbf{X'}^*) \Phi_{\mathbf{\tilde X}} \Phi_{\mathbf{\tilde X}}^*) \\
    &= \frac{1}{n+m}[Tr(W_\mathbf{X} L_\mathbf{YY} W_\mathbf{X} K_{\mathbf{X \tilde X}} K_{\mathbf{\tilde X X}}) - 2Tr(W_\mathbf{X} L_\mathbf{YZ} W_\mathbf{X'} K_{\mathbf{X' \tilde X}} K_{\mathbf{\tilde X X}}) \\
    &\qquad \qquad \qquad + Tr(W_\mathbf{X'} L_\mathbf{ZZ} W_\mathbf{X'} K_{\mathbf{X' \tilde X}} K_{\mathbf{\tilde X X'}})]
\end{align*}
\end{proof}

\clearpage
\textbf{Proof of Lemma \ref{lem:cmmd2_est_closed}}
\begin{proof}
Using empirical estimators of the CMOs and covariance matrix
\begin{align*}
    \widehat{\text{CMMD}}_2^2 &= \| (\widehat C_{Y|X} - \widehat C_{Z|X}) \widehat C_{\tilde X \tilde X} \|_{\mathcal H_\ell \otimes \mathcal H_k}^2 \\
    &= \frac{1}{(n+m)^2} \| (\Psi_\mathbf{Y} W_\mathbf{X} \Phi_\mathbf{X}^* - \Psi_{\mathbf{Z}} W_\mathbf{X'} \Phi_\mathbf{X'}^* ) \Phi_{\mathbf{\tilde X}} \Phi_{\mathbf{\tilde X}}^* \|_{\mathcal H_\ell \otimes \mathcal H_k}^2 \\
    &= \frac{1}{(n+m)^2} Tr( \Phi_{\mathbf{\tilde X}} \Phi_{\mathbf{\tilde X}}^* (\Phi_\mathbf{X} W_\mathbf{X} \Psi_\mathbf{Y}^* - \Phi_\mathbf{X'} W_\mathbf{X'} \Psi_{\mathbf{Z}}^*) (\Psi_\mathbf{Y} W_\mathbf{X} \Phi_\mathbf{X}^* - \Psi_{\mathbf{Z}} W_\mathbf{X'} \Phi_\mathbf{X'}^* ) \Phi_{\mathbf{\tilde X}} \Phi_{\mathbf{\tilde X}}^* ) \\
    &= \frac{1}{(n+m)^2} Tr( \Phi_{\mathbf{\tilde X}} K_{\mathbf{\tilde X X}} W_\mathbf{X} L_\mathbf{YY} W_\mathbf{X} K_{\mathbf{X \tilde X}} \Phi_{\mathbf{\tilde X}}^* - \Phi_{\mathbf{\tilde X}} K_{\mathbf{\tilde X X}} W_\mathbf{X} L_\mathbf{YZ} W_\mathbf{X'} K_{\mathbf{X' \tilde X}} \Phi_{\mathbf{\tilde X}}^* \\
    & \qquad \qquad \qquad - \Phi_{\mathbf{\tilde X}} K_{\mathbf{\tilde X X'}} W_\mathbf{X'} L_\mathbf{ZY} W_{\mathbf X} K_{\mathbf{X \tilde X}} \Phi_{\mathbf{\tilde X}}^* + \Phi_{\mathbf{\tilde X}} K_{\mathbf{\tilde X X'}} W_\mathbf{X'} L_\mathbf{ZZ} W_\mathbf{X'} K_{\mathbf{X' \tilde X}} \Phi_{\mathbf{\tilde X}}^*) \\
    &= \frac{1}{(n+m)^2}[Tr(W_\mathbf{X} L_\mathbf{YY} W_\mathbf{X} K_{\mathbf{X \tilde X}} K_{\mathbf{\tilde X \tilde X}} K_{\mathbf{\tilde X X}})  \\
    &\qquad \qquad - 2Tr(W_\mathbf{X} L_\mathbf{YZ} W_\mathbf{X'} K_{\mathbf{X' \tilde X}} K_{\mathbf{\tilde X \tilde X}} K_{\mathbf{\tilde X X}})
     + Tr(W_\mathbf{X'} L_\mathbf{ZZ} W_\mathbf{X'} K_{\mathbf{X' \tilde X}}K_{\mathbf{\tilde X \tilde X}} K_{\mathbf{\tilde X X'}})]
\end{align*}
\end{proof}

\textbf{Proof of Theorem \ref{thm:cmmd_s_est_closed}}
\begin{proof}
Once more, starting with the empirical estimators of the CMOs and the covariance operator, 
\begin{align*}
    \widehat{\textup{CMMD}}_s^2 &= Tr(\hat \Delta^* \hat \Delta \hat C_{\tilde X\tilde X}^s) \\
    &= Tr\left( (\Phi_\mathbf{X} W_\mathbf{X} \Psi_\mathbf{Y}^* - \Phi_\mathbf{X'} W_\mathbf{X'} \Psi_{\mathbf{Z}}^*) (\Psi_\mathbf{Y} W_\mathbf{X} \Phi_\mathbf{X}^* - \Psi_{\mathbf{Z}} W_\mathbf{X'} \Phi_\mathbf{X'}^* ) \left(\frac{1}{(n+m)} \Phi_{\mathbf{\tilde X}} \Phi_{\mathbf{\tilde X}}^*\right)^s \right) \\
    &= \frac{1}{(n+m)^s}[ Tr(W_\mathbf{X} L_\mathbf{YY} W_\mathbf{X} \Phi_\mathbf{X}^* (\Phi_\mathbf{\tilde X} \Phi_\mathbf{\tilde X}^*)^s \Phi_\mathbf{X}) - 2Tr(W_\mathbf{X} L_\mathbf{YZ} W_\mathbf{X'} \Phi_\mathbf{X'}^* (\Phi_\mathbf{\tilde X} \Phi_\mathbf{\tilde X}^*)^s \Phi_\mathbf{X}) \\
    & \qquad \qquad \qquad + Tr(W_\mathbf{X'} L_\mathbf{ZZ} W_\mathbf{X'} \Phi_\mathbf{X'}^* (\Phi_\mathbf{\tilde X} \Phi_\mathbf{\tilde X}^*)^s \Phi_\mathbf{X'})].
\end{align*}
Next, consider the compact SVD $\Phi_\mathbf{\tilde X} = U \Sigma V^*$. Then 
\begin{align*}
    \Phi_\mathbf{\tilde X}^* (\Phi_\mathbf{\tilde X} \Phi_\mathbf{\tilde X}^*)^s \Phi_\mathbf{\tilde X} &= V \Sigma U^*(U \Sigma V^* V \Sigma U^*)^s U \Sigma V^* \\
    &= V \Sigma U^* (U \Sigma^2 U^*)^s U \Sigma V^* \\
    &= V \Sigma^{2(s+1)} V^* \\
    &= (V \Sigma U^* U \Sigma V^*)^{s+1} \\
    &= (\Phi_\mathbf{\tilde X}^* \Phi_\mathbf{\tilde X})^{s+1} \\
    &= K_\mathbf{\tilde X \tilde X}^{s+1}.
\end{align*}
Using projection matrices and substituting this into the expression above gives,
\begin{align*}
    \widehat{\textup{CMMD}}_s^2 &= \frac{1}{(n+m)^s}[ Tr(W_\mathbf{X} L_\mathbf{YY} W_\mathbf{X} \Pi_\mathbf{X} \Phi_\mathbf{\tilde X}^* (\Phi_\mathbf{\tilde X} \Phi_\mathbf{\tilde X}^*)^s \Phi_\mathbf{\tilde X} \Pi_\mathbf{X}^\top) \\
    & \qquad \qquad \qquad
    - 2Tr(W_\mathbf{X} L_\mathbf{YZ} W_\mathbf{X'} \Pi_\mathbf{X'} \Phi_\mathbf{\tilde X}^* (\Phi_\mathbf{\tilde X} \Phi_\mathbf{\tilde X}^*)^s \Phi_\mathbf{\tilde X} \Pi_\mathbf{X}^\top) \\
    & \qquad \qquad \qquad
    + Tr(W_\mathbf{X'} L_\mathbf{ZZ} W_\mathbf{X'} \Pi_\mathbf{X'} \Phi_\mathbf{\tilde X}^* (\Phi_\mathbf{\tilde X} \Phi_\mathbf{\tilde X}^*)^s \Phi_\mathbf{\tilde X} \Pi_\mathbf{X'}^\top)] \\
    &= \frac{1}{(n+m)^s}[ Tr(W_\mathbf{X} L_\mathbf{YY} W_\mathbf{X} \Pi_\mathbf{X}  K_\mathbf{\tilde X \tilde X}^{s+1} \Pi_\mathbf{X}^\top) - 2Tr(W_\mathbf{X} L_\mathbf{YZ} W_\mathbf{X'} \Pi_\mathbf{X'}  K_\mathbf{\tilde X \tilde X}^{s+1} \Pi_\mathbf{X}^\top) \\
    & \qquad \qquad \qquad
    + Tr(W_\mathbf{X'} L_\mathbf{ZZ} W_\mathbf{X'} \Pi_\mathbf{X'}  K_\mathbf{\tilde X \tilde X}^{s+1} \Pi_\mathbf{X'}^\top)] \\
\end{align*}
\end{proof}

\clearpage

We continue with proofs for results regarding doubly robust estimators.

\textbf{Proof of Lemma \ref{lem:cme_prop}}
\begin{proof}
Starting from the right hand side
\begin{align*}
    \mathbb{E} \left[ \frac{T \ell(\cdot, W)}{e(X)} \Big| X=x \right] &= \mathbb{E}_T \left[ \mathbb{E} \left[ \frac{T \ell(\cdot, W)}{e(X)} \Big|T, X=x \right] \Big| X=x \right] \\
    &= \frac{1}{e(x)}\mathbb{E}_T \left[ T \mathbb{E} \left[ \ell(\cdot, W) |T, X=x \right] | X=x \right] \\
    &= \frac{1}{e(x)}(P(T=1|X=x) \mathbb{E} \left[ \ell(\cdot, W) |T=1, X=x \right] + P(T=0|X=x) \cdot 0) \\
    &= \frac{e(x)}{e(x)} \mathbb E [\ell(\cdot, W) | T=1, X=x] \\
    &= \mathbb E [\ell(\cdot, Y) | X=x] \\
    &= \mu_{Y|x}
\end{align*}
The proof for $\mu_{Z|x}$ is analogous.
\end{proof}

We introduce the doubly robust CMO estimator
\begin{equation*}
    \hat C_{Y|X}^{DR} = \Psi_{\mathbf{\tilde Y}, DR} (K_{\mathbf{\tilde X \tilde X}} + n \lambda I_n)^{-1} \Phi_{\mathbf{\tilde X}}^*
\end{equation*}
where the operator $\Psi_{\mathbf{\tilde Y}, DR}: \mathbb R^n \to \mathcal H_\ell$ has $i$th column given by 
\begin{equation*}
    [\Psi_{\mathbf{\tilde Y}, DR}]_i = \frac{t_i}{\hat e(\tilde x_i)}(\ell(\cdot, w_i) - \hat\mu_{Y|\tilde x_i}^{model})+\hat\mu_{Y|\tilde x_i}^{model} = \frac{t_i}{\hat e(\tilde x_i)}(\ell(\cdot, \tilde y_i) - \hat \mu_{Y|\tilde x_i}^{model}) +\hat\mu_{Y|\tilde x_i}^{model}.
\end{equation*}
The following lemma holds.
\begin{lemma} \label{lem:cmo_dr}
    Take the same assumptions as in Theorem \ref{thm:cme_dr} but without Assumption \ref{asm:cmo_hs}. Then 
    \begin{equation*}
        \|\hat C_{Y|X} - \hat C_{Y|X}^{DR} \|_{HS} = O_p(\lambda^{-1} n^{-\frac{1}{2}} + \lambda^{-1} \zeta_n \eta_n).
    \end{equation*}
    An analogous result holds for $\hat C_{Z|X}$ and $\hat C_{Z|X}^{DR}$. 
\end{lemma}
\begin{proof}
To estimate $\hat C_{Y|X}$, we use the same data $\{(t_i, \tilde x_i, w_i)\}_{i=1}^n$ but replace each $w_i$ with the potential outcome $\tilde y_i$. In the following, we make use of the property $\|A B\|_{HS} \leq \| A \|_{op} \|B \|_{HS}$ where $\| \cdot \|_{op}$ and $\| \cdot \|_{HS}$ represent the operator and Hilbert-Schmidt norm respectively. We also note that $\|(\hat C_{\tilde X \tilde X} + \lambda I)^{-1}\|_{op} \leq \frac{1}{\lambda}$. We then have
\begin{align*}
    \| \hat C_{Y|X}^{DR} - \hat C_{Y|X} \|_{HS} &= \left\|\frac{1}{n}(\Psi_{\mathbf{\tilde Y}, DR} - \Psi_{\mathbf{\tilde Y}}) \Phi_{\mathbf{\tilde X}}^* (\hat C_{\tilde X \tilde X} + \lambda I)^{-1} \right\|_{HS} \\
    & \leq  \|(\hat C_{\tilde X \tilde X} + \lambda I)^{-1}\|_{op} \left\| \frac{1}{n}(\Psi_{\mathbf{\tilde Y}, DR} - \Psi_{\mathbf{\tilde Y}}) \Phi_{\mathbf{\tilde X}}^* \right\|_{HS} \\
    &\leq \frac{1}{\lambda} \left\| \frac{1}{n}(\Psi_{\mathbf{\tilde Y}, DR} - \Psi_{\mathbf{\tilde Y}}) \Phi_{\mathbf{\tilde X}}^* \right\|_{HS} \\
    &= \frac{1}{\lambda} \left\| \frac{1}{n} \sum_{i=1}^n \left( \frac{t_i - \hat e(\tilde x_i)}{\hat e(\tilde x_i)} \right) (\ell(\cdot, \tilde y_i) - \hat \mu_{Y|\tilde x_i}^{model}) \otimes k(\cdot, \tilde x_i) \right\|_{HS}
\end{align*}
Taking a random sample $(T, X, Y)$ from $P_{T\tilde X \tilde Y}$, we apply the triangle inequality to get
\begin{multline*}
    \frac{1}{\lambda} \left\| \frac{1}{n} \sum_{i=1}^n \left( \frac{t_i - \hat e(\tilde x_i)}{\hat e(\tilde x_i)} \right) (\ell(\cdot, \tilde y_i) - \hat \mu_{Y|\tilde x_i}^{model}) \otimes k(\cdot, \tilde x_i) \right\|_{HS} \\
    \leq \frac{1}{\lambda} \left\| \frac{1}{n} \sum_{i=1}^n \left( \frac{t_i - \hat e(\tilde x_i)}{\hat e(\tilde x_i)} \right)(\ell(\cdot, \tilde y_i) - \hat \mu_{Y|\tilde x_i}^{model}) \otimes k(\cdot, \tilde x_i) - \mathbb E \left[\left( \frac{T - \hat e(X)}{\hat e(X)} \right) (\ell(\cdot, Y) - \hat \mu_{Y|X}^{model}) \otimes k(\cdot, X) \right] \right\|_{HS} \\
    + \frac{1}{\lambda} \left\| \mathbb E \left[\left( \frac{T - \hat e(X)}{\hat e(X)} \right) (\ell(\cdot, Y) - \hat \mu_{Y|X}^{model}) \otimes k(\cdot, X) \right] \right\|_{HS}.
\end{multline*}

The first term is of size $O_p(\lambda^{-1}n^{-\frac{1}{2}})$ which follows from Assumption \ref{asm:ell_char} and $\eta_{n_t} = O(1)$. 

Turning to the second term,
\begin{align*}
    \frac{1}{\lambda} &\left\| \mathbb E \left[\left( \frac{T - \hat e(X)}{\hat e(X)} \right) (\ell(\cdot, Y) - \hat \mu_{Y|X}^{model}) \otimes k(\cdot, X) \right] \right\|_{HS} \\
    &= \frac{1}{\lambda} \left\| \mathbb E_X \left[ \mathbb E \left[\left( \frac{T - \hat e(X)}{\hat e(X)} \right) (\ell(\cdot, Y) - \hat \mu_{Y|X}^{model}) \otimes k(\cdot, X) \right] \Big| X \right] \right\|_{HS} \\
    &= \frac{1}{\lambda} \left\| \mathbb E_X \left[ \left( \frac{ \mathbb E[T|X] - \hat e(X)}{\hat e(X)} \right) (\mathbb E[\ell(\cdot, Y) |X] - \hat \mu_{Y|X}^{model}) \otimes k(\cdot, X) \right] \right\|_{HS} \\
    &= \frac{1}{\lambda} \left\| \mathbb E_X \left[ \left( \frac{ e(X) - \hat e(X)}{\hat e(X)} \right) (\mu_{Y|X} - \hat \mu_{Y|X}^{model}) \otimes k(\cdot, X) \right] \right\|_{HS} \\
    &\leq \frac{1}{\lambda}  \mathbb E_X \left[ \frac{1}{\hat e(X)} \left| e(X) - \hat e(X) \right| \left\|(\mu_{Y|X} - \hat \mu_{Y|X}^{model}) \otimes k(\cdot, X) \right\|_{HS} \right]  \\
    &\leq \frac{1}{\lambda \delta} \mathbb E_X \left[ |e(X) - \hat e(X)| \left\| (\mu_{Y|X} - \hat \mu_{Y|X}^{model}) \otimes k(\cdot, X) \right\|_{HS} \right] \\
    &= \frac{1}{\lambda \delta} \mathbb E_X \left[ |e(X) - \hat e(X)| \left\| \mu_{Y|X} - \hat \mu_{Y|X}^{model} \right\|_{\mathcal H_\ell} \left\| k(\cdot, X) \right\|_{\mathcal H_k} \right] \\
    &\leq \frac{\sqrt{k_{\max}}}{\lambda \delta} \mathbb E_X \left[ |e(X) - \hat e(X)| \left\| \mu_{Y|X} - \hat \mu_{Y|X}^{model} \right\|_{\mathcal H_\ell} \right] \\
    &\leq \frac{\sqrt{k_{\max}}}{\lambda \delta} (\mathbb E_X |e(X)-\hat e(X)|^2)^{\frac{1}{2}} (\mathbb E_X \left\| \mu_{Y|X} - \hat\mu_{Y|X}^{model} \right\|_{\mathcal H_\ell}^2)^{\frac{1}{2}} \\
    &= O_p\left( \frac{\zeta_n \eta_n}{\lambda} \right) 
\end{align*}
where $\delta > 0$ is such that $\hat e(x) > \delta$ for all $x$ and $k_{\max}$ is an upper bound for the kernel $k$, that is, $k(x,x) \leq k_{\max}$ for all $x$. Combining the convergence rates gives the desired result. 
\end{proof}

With this result, we are now ready to prove the convergence rate of the doubly robust CME estimator. 

\textbf{Proof of Theorem \ref{thm:cme_dr}}

\begin{proof}
Suppose we are testing on a set of size $n$, and trained the models $\hat e$ and $\hat \mu_{Y|X}^{model}$ on a set of size $m = O(n)$. Next, we introduce potential outcomes $\tilde y_i$. When $t_i = 1$, we set $\tilde y_i = w_i$, and otherwise $\tilde y_i$ is sampled from $P_{Y|\tilde x_i}$. Note that $P_{\tilde Y|X} = P_{Y|X}$ and so the CME can be estimated from samples $\{(\tilde x_i, \tilde y_i) \}_{i=1}^n$ in the usual way $\hat \mu_{Y|x} = \hat C_{Y|X} k(\cdot, x)$. Applying the triangle inequality,
\begin{equation*}
    \| \mu_{Y|x} - \hat \mu_{Y|x}^{DR} \|_{\mathcal H_\ell} \leq \| \mu_{Y|x} - \hat \mu_{Y|x} \|_{\mathcal H_\ell} + \| \hat \mu_{Y|x} - \hat \mu_{Y|x}^{DR} \|_{\mathcal H_\ell}.
\end{equation*}
By Theorem \ref{thm:cme_conv}, the first term is $O_p(\lambda^{-1}n^{-\frac{1}{2}} + \lambda^{\frac{1}{2}})$. As for the second term,
\begin{align*}
    \| \hat \mu_{Y|x} - \hat \mu_{Y|x}^{DR} \|_{\mathcal H_\ell} &= \| (\hat C_{Y|X} - \hat C_{Y|X}^{DR})k(\cdot, x) \|_{\mathcal H_\ell} \\
    &\leq \|k(\cdot, x)\|_{\mathcal H_k} \|\hat C_{Y|X} - \hat C_{Y|X}^{DR} \|_{HS} \\
    &\leq \sqrt{k_{\max}} \|\hat C_{Y|X} - \hat C_{Y|X}^{DR} \|_{HS}
\end{align*}
where $\hat C_{Y|X}^{DR}$ is the doubly robust estimator for the CMO, and $k_{\max}$ is an upper bound for the kernel $k$, that is, $k(x,x) \leq k_{\max}$ for all $x$. By Lemma \ref{lem:cmo_dr}, $\|\hat C_{Y|X} - \hat C_{Y|X}^{DR} \|_{HS} = O_p(\lambda^{-1} n^{-\frac{1}{2}} + \lambda^{-1} \zeta_n \eta_n)$. Combining the results, we get 
\begin{equation*}
    \| \mu_{Y|x} - \hat \mu_{Y|x}^{DR} \|_{\mathcal H_\ell} = O_p(\lambda^{\frac{1}{2}} + \lambda^{-1}n^{-\frac{1}{2}} + \lambda^{-1}\zeta_n \eta_n).
\end{equation*}
as required.
\end{proof}

For DR CMMD estimators, we show the convergence of $\hat \Delta_{DR}$.

\textbf{Proof of Theorem \ref{thm:cmmd_dr_conv}}

\begin{proof}
By Assumptions \ref{asm:k_bounded} and \ref{asm:delta_hs}, we know that $\Delta$ is Hilbert-Schmidt by a method similar to the proof of Theorem \ref{thm:cmo_hs}. Next, define $\hat \Delta = \frac{1}{n}(\Psi_{\mathbf{Y}} -\Psi_{\mathbf{Z}}) \Phi_{\mathbf{X}}^* (\hat C_{XX} + \lambda I)^{-1}$. Then,
\begin{align*}
    \|\Delta - \hat \Delta_{DR} \|_{\mathcal H_\ell \otimes \mathcal H_k} &\leq \|\Delta - \hat \Delta \|_{\mathcal H_\ell \otimes \mathcal H_k} + \|\hat \Delta - \hat \Delta_{DR} \|_{\mathcal H_\ell \otimes \mathcal H_k} \\
     &\leq \|\Delta - \hat \Delta \|_{\mathcal H_\ell \otimes \mathcal H_k} + \|\hat C_{Y|X} - \hat C_{Y|X}^{DR} \|_{\mathcal H_\ell \otimes \mathcal H_k} + \|\hat C_{Z|X} - \hat C_{Z|X}^{DR} \|_{\mathcal H_\ell \otimes \mathcal H_k}
\end{align*}
By a proof similar to that of Theorem \ref{thm:cmo_est_conv} (now only requiring Assumption \ref{asm:delta_hs} rather than \ref{asm:cmo_hs}), the first term converges in probability to zero. The second and third term are of size $O_p(\lambda^{-1}n^{-\frac{1}{2}} + \lambda^{-1} \zeta_n \eta_n)$ by Lemma \ref{lem:cmo_dr} and also converge to zero under our assumptions. 
\end{proof}

\subsection{Hypothesis Testing} \label{sec:app_testing}

Algorithm~\ref{alg:2} is used for hypothesis testing when $P_X \neq Q_X$, and requires resampling labels $t$ (indicating whether a data point comes from $P$ or $Q$) using the known propensity function $e$. However, the algorithm also requires that the bootstrapped samples have the same number of data as the original data sets $\mathcal D_P$ and $\mathcal D_Q$, that is, $n$ and $m$ respectively. We can ensure this by sampling a vector of label $\mathbf t$ from a fixed-count conditional Bernoulli distribution. A naive exact implementation is computationally expensive, as it requires considering $\binom{n+m}{n}$ possible label assignments. We therefore use a dynamic programming scheme to sample efficiently and exactly from this distribution.

Conditional on $\sum_{i=1}^{n+m} T_i=n$, the probability of sampling a label vector $\mathbf t$ is proportional to
\begin{equation*}
    \prod_{i=1}^{n+m} e(\tilde x_i)^{t_i}(1-e(\tilde x_i))^{1-t_i}
    \propto
    \prod_{i=1}^{n+m} \omega_i^{t_i},
\end{equation*}
where $\omega_i=\frac{e(\tilde x_i)}{1-e(\tilde x_i)}$. Define
\begin{equation*}
    E_{i,r}
    =
    \sum_{\substack{A\subseteq\{i,\ldots,n+m\}\\ |A|=r}}
    \prod_{j\in A} \omega_j,
\end{equation*}
the total weight of all assignments selecting exactly $r$ observations from indices $i,\ldots,n+m$. These quantities satisfy the recursion
\begin{equation*}
    E_{i,r}=E_{i+1,r} + \omega_i E_{i+1,r-1},
\end{equation*}
with boundary conditions $E_{n+m+1,0}=1$ and $E_{n+m+1,r}=0$ for $r>0$. Once this table has been computed, labels can be sampled sequentially. If $r$ observations still need to be assigned to $\mathcal D_P$, then observation $i$ is assigned label $T_i=1$ with probability
\begin{equation*}
    \mathbb P(T_i=1\mid r)
    =
    \frac{\omega_i E_{i+1,r-1}}{E_{i,r}}.
\end{equation*}
After sampling $T_i$, the remaining count is updated and the procedure continues until all labels have been assigned. The dynamic programming table requires $O((n+m)n)$ operations to compute and can be reused for every bootstrap replicate, after which each label vector requires only $O(n+m)$ operations to sample. In practice, the recursion is implemented on the log scale for numerical stability.

\textbf{Proof of Theorem \ref{thm:test_val}}

\begin{proof}
    Under $H_0$, the conditional distribution of the outcome given $X$ is the same under $P$ and $Q$. Hence the group label $T_i$ is conditionally independent of the outcome given $X_i$, and
    \begin{equation*}
        \mathbb P(T_i=1\mid X_i = \tilde x_i, W_i = w_i) = \mathbb P(T_i=1\mid X_i = \tilde x_i) = e(\tilde x_i).
    \end{equation*}
    Therefore, conditional on the pooled observations and the fixed group count,
    \begin{equation*}
        \mathbb P\left(\mathbf T=\mathbf t \,\middle|\, \tilde{\mathcal D}, \sum_{i=1}^{n+m} T_i=n \right)
        \propto
        \mathbf 1\left\{\sum_{i=1}^{n+m} t_i=n\right\}
        \prod_{i=1}^{n+m} e(\tilde x_i)^{t_i}(1-e(\tilde x_i))^{1-t_i},
    \end{equation*}
    which is precisely the resampling distribution used in Algorithm~\ref{alg:2}. The assumption $0<e(X)<1$ ensures that this conditional distribution is well defined.

    It follows that, conditional on $\tilde{\mathcal D}$ and $\sum_{i=1}^{n+m} T_i=n$, the observed label vector $\mathbf T$ and the resampled vectors
    $\mathbf T^{(1)},\ldots,\mathbf T^{(B)}$ are identically distributed and exchangeable. Consequently, the corresponding statistics
    $S,S^{(1)},\ldots,S^{(B)}$ are exchangeable. The Monte Carlo randomization $p$-value
    \begin{equation*}
        p=\frac{1+\sum_{b=1}^B\mathbf 1\{S^{(b)}\geq S\}}{B+1}
    \end{equation*}
    is therefore super-uniform under $H_0$, giving
    \begin{equation*}
        \mathbb P\left( p \leq \alpha \,\middle|\, \tilde{\mathcal D}, \sum_{i=1}^{n+m} T_i = n \right) \leq \alpha.
    \end{equation*}
    Taking expectations over the conditioning variables gives
    $\mathbb P(p\leq\alpha)\leq\alpha$ under $H_0$, proving the result.
\end{proof}

\section{Doubly Robust Estimator Details} \label{sec:app_dr}

Writing down a closed form expression for doubly robust estimators of CMMD requires first describing how to fit propensity and CME models from a training data set $\{(t_i, \tilde x_i, w_i)\}_{i=1}^m$. In our experiments the propensity is known exactly, but in practice one can fit a propensity model using logistic regression with the variable $T$ acting as a label for data $X$. The CME models may use the standard estimator
\begin{equation*}
    \hat \mu_{Y|x}^{model} = \hat \mu_{Y|x} = \Psi_\mathbf{Y} W_\mathbf{X} K_{\mathbf{X}x} \qquad \text{and} \qquad \hat \mu_{Z|x}^{model} = \hat\mu_{Z|x} = \Psi_{\mathbf{Z}} W_\mathbf{X'} K_{\mathbf{X'}x},
\end{equation*}
where $W_\mathbf{X} = (K_\mathbf{XX} + \lambda_p m_p I_{m_p})^{-1}$ and $W_\mathbf{X'} = (K_\mathbf{X' X'} + \lambda_q m_q I_{m_q})^{-1}$, although any models can be chosen in principle.

Consider $\hat \Delta_{DR} = \Psi_{DR}(K_{\mathbf{\tilde X \tilde X}} + n \lambda I_n)^{-1} \Phi_{\mathbf{\tilde X}}^*$ introduced in Section \ref{sec:doubly_robust}. Define $\boldsymbol e \in \mathbb R^n$ with elements $e_i = \hat e(\tilde x_i)$ and $\tilde{\boldsymbol e} \in \mathbb R^n$ with elements $\tilde e_i = \frac{t_i - \hat e(\tilde x_i)}{\hat e(\tilde x_i)(1 - \hat e(\tilde x_i))}$. From these vectors, construct the matrices $E = \text{diag}(\boldsymbol e)$ and $\tilde E = \text{diag}(\tilde{\boldsymbol e})$. We also define $M_Y = [\hat \mu_{Y|\tilde x_1}^{model}, \dots, \hat \mu_{Y|\tilde x_n}^{model}]$ and $M_Z = [\hat \mu_{Z|\tilde x_1}^{model}, \dots, \hat \mu_{Z|\tilde x_n}^{model}]$. The $i$th element of the operator $\Psi_{DR}$ can be expressed as
\begin{align*}
    [\Psi_{DR}]_i &= \frac{t_i}{\hat e(\tilde x_i)}(\ell(\cdot, w_i) - \hat\mu_{Y|\tilde x_i}^{model})+\hat\mu_{Y|\tilde x_i}^{model} - \frac{(1-t_i)}{(1-\hat e(\tilde x_i))}(\ell(\cdot, w_i) - \hat\mu_{Z|\tilde x_i}^{model}) - \hat\mu_{Z|\tilde x_i}^{model} \\
    &= \left( \frac{t_i - \hat e(\tilde x_i)}{\hat e(\tilde x_i)(1- \hat e(\tilde x_i))} \right) \left[ \ell(\cdot, w_i) - (1 - \hat e(\tilde x_i)) \hat \mu_{Y|\tilde x_i}^{model} - \hat e(\tilde x_i) \hat \mu_{Z|\tilde x_i}^{model} \right]
\end{align*}
or in matrix notation, $\Psi_{DR} = (\Psi_\mathbf{W} - M_Y(I_n - E) - M_Z E) \tilde E$. 

The doubly robust estimator for CMMD$_0^2$ is then given by
\begin{align*}
    \widehat{\text{CMMD}}_{0, DR}^2 &= Tr(\hat \Delta_{DR}^* \hat\Delta_{DR}) \\
        &= Tr( \Phi_{\mathbf{\tilde X}}(K_{\mathbf{\tilde X \tilde X}} + n \lambda I_n)^{-1} \Psi_{DR}^* \Psi_{DR}(K_{\mathbf{\tilde X \tilde X}} + n \lambda I_n)^{-1} \Phi_{\mathbf{\tilde X}}^*) \\
        &= Tr(\Psi_{DR}^* \Psi_{DR} W_{\mathbf{\tilde X}} K_{\mathbf{\tilde X \tilde X}} W_{\mathbf{\tilde X}}) \\
        &= Tr\Big(\tilde E\big(L_{WW} - L_{W\hat Y} (I_n - E) - L_{W \hat Z} E - (I_n - E) L_{\hat Y W} + (I_n - E)L_{\hat Y \hat Y} (I_n - E) \\
        &\qquad + (I_n - E) L_{\hat Y \hat Z} E - E L_{\hat Z W} + E L_{\hat Z \hat Y} (I_n - E) + E L_{\hat Z \hat Z} E \big) \tilde E W_{\mathbf{\tilde X}} K_{\mathbf{\tilde X \tilde X}} W_{\mathbf{\tilde X}} \Big)
\end{align*}
where $W_{\mathbf{\tilde X}} = (K_{\mathbf{\tilde X \tilde X}} + n \lambda I_n)^{-1}$, and we have matrices with entries
\begin{align*}
    [L_{WW}]_{ij} &= \ell(w_i, w_j) \\
    [L_{W\hat Y}]_{ij} &= \langle \ell(\cdot, w_i), \hat \mu_{Y|\tilde x_j}^{model} \rangle_{\mathcal H_\ell} = L_{w_iY} W_{\mathbf X} K_{\mathbf{X} \tilde x_j} \\
    [L_{W\hat Z}]_{ij} &= \langle \ell(\cdot, w_i), \hat \mu_{Z|\tilde x_j}^{model} \rangle_{\mathcal H_\ell} = L_{w_iZ} W_\mathbf{X'} K_{\mathbf{X'} \tilde x_j} \\
    [L_{\hat Y \hat Y}]_{ij} &= \langle \hat \mu_{Y|\tilde x_i}^{model}, \hat \mu_{Y|\tilde x_j}^{model} \rangle_{\mathcal H_\ell} = K_{\tilde x_i X} W_\mathbf{X} L_\mathbf{YY} W_{\mathbf X} K_{\mathbf{X} \tilde x_j} \\
    [L_{\hat Y \hat Z}]_{ij} &= \langle \hat \mu_{Y|\tilde x_i}^{model}, \hat \mu_{Z|\tilde x_j}^{model} \rangle_{\mathcal H_\ell} = K_{\tilde x_i X} W_\mathbf{X} L_\mathbf{YZ} W_\mathbf{X'} K_{\mathbf{X'} \tilde x_j} \\
    [L_{\hat Z \hat Z}]_{ij} &= \langle \hat \mu_{Z|\tilde x_i}^{model}, \hat \mu_{Z|\tilde x_j}^{model} \rangle_{\mathcal H_\ell} = K_{\tilde x_i X'} W_\mathbf{X'} L_\mathbf{ZZ} W_\mathbf{X'} K_{\mathbf{X'} \tilde x_j} \\
\end{align*}
and $L_{\hat Y W} = L_{W\hat Y}^\top$, $L_{\hat Z W} = L_{W\hat Z}^\top$, $L_{\hat Z\hat Y} = L_{\hat Y\hat Z}^\top$. 

Doubly robust estimators for CMMD$_1^2$ and CMMD$_2^2$ are computed in a similar fashion, involving extra factors of $K_{\mathbf{\tilde X \tilde X}}$ within the trace calculation and appropriate scaling by $\frac{1}{n}$ and $\frac{1}{n^2}$ respectively.

\section{Further Experimental Details} \label{ap:experiment_details}

Here we include extra details relevant to the experiments conducted in the main paper.

\subsection{Synthetic Data: Hypothesis Testing}

We begin with a scenario similar to that in Section \ref{sec:exp_synthetic} and consider two settings:
\begin{center}
    Setting 1: $P_X = Q_X = \mathcal N(0,1)$, \qquad Setting 2: $P_X = \mathcal N(-0.5, 1)$ and $Q_X = \mathcal N(0.5, 1)$.
\end{center}
These use Algorithms \ref{alg:1} and \ref{alg:2} respectively. The conditional relationships, kernels and hyperparameters are the same as in the main body. To test the Type I error when $P_{Y|X} = Q_{Z|X}$, we independently sample from $Z|X = \exp(-0.5 X^2) \sin(2X) + \epsilon$. Since we know the distributions of $X$ under $P$ and $Q$, we can compute the propensity score $e$ exactly. 

In Figure \ref{fig:synthetic_experiment_ap} (top), we plot simulated data for the two settings, drawing $n=250$ samples from each of $P$ and $Q$. The data is used to estimate each of the CMMD statistics, with Figure \ref{fig:synthetic_experiment_ap} (bottom) illustrating the distribution over 250 trials for both settings on a log scale. Since the Gaussian kernel satisfies $k(x,x) = 1$ for all $x \in \mathcal X$, we see agreement with Corollary \ref{cor:cmmd_rel}. Estimates for the squared CMMD$_2$ give the smallest values, while those for CMMD$_0$ produce the largest. Due to a change in marginal distribution of covariates, all three test statistics increase in magnitude when comparing setting 1 to setting 2. However, this is less prevalent for CMMD$_0$ as its population quantity is independent of $P_X$ and $Q_X$.

\begin{figure}[htbp]
    \centering
    \begin{subfigure}[b]{0.48\textwidth}
        \centering
        \includegraphics[width=\linewidth]{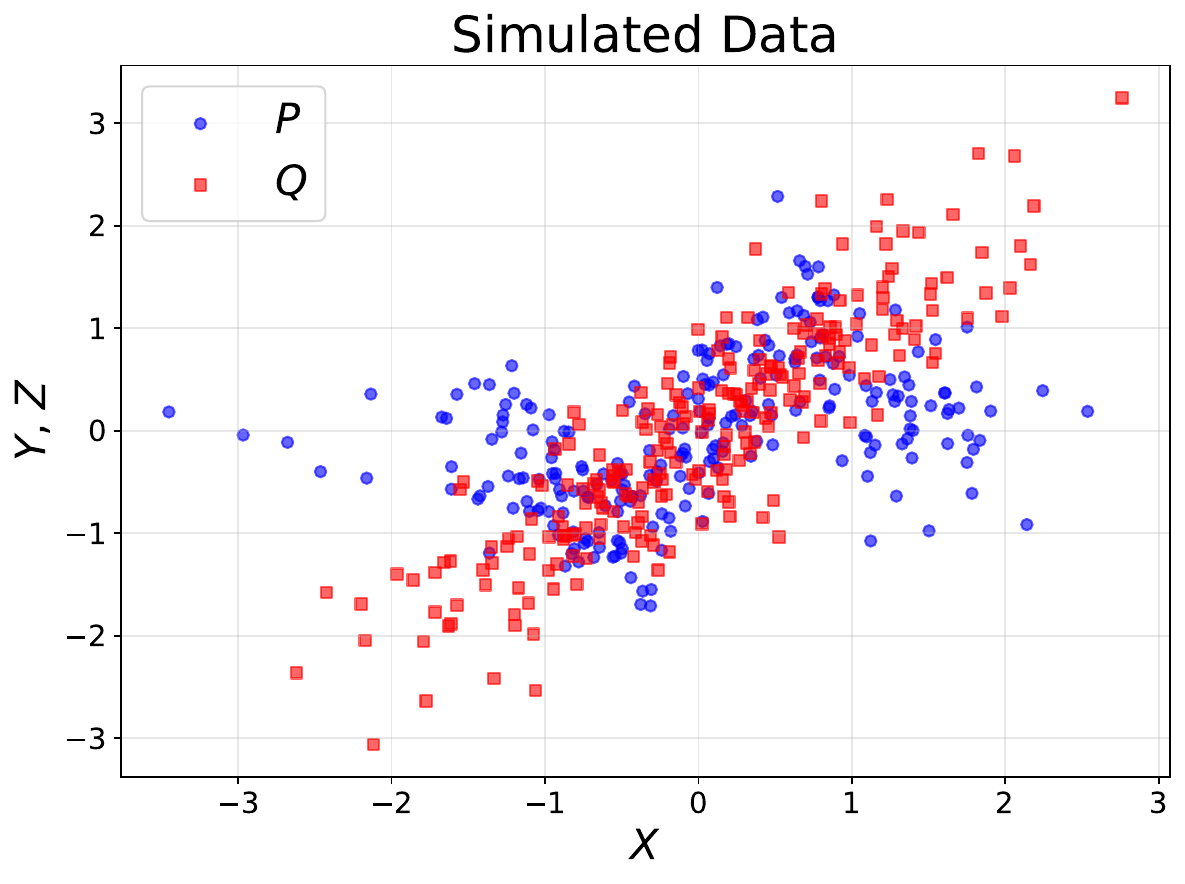}
        \includegraphics[width=\linewidth]{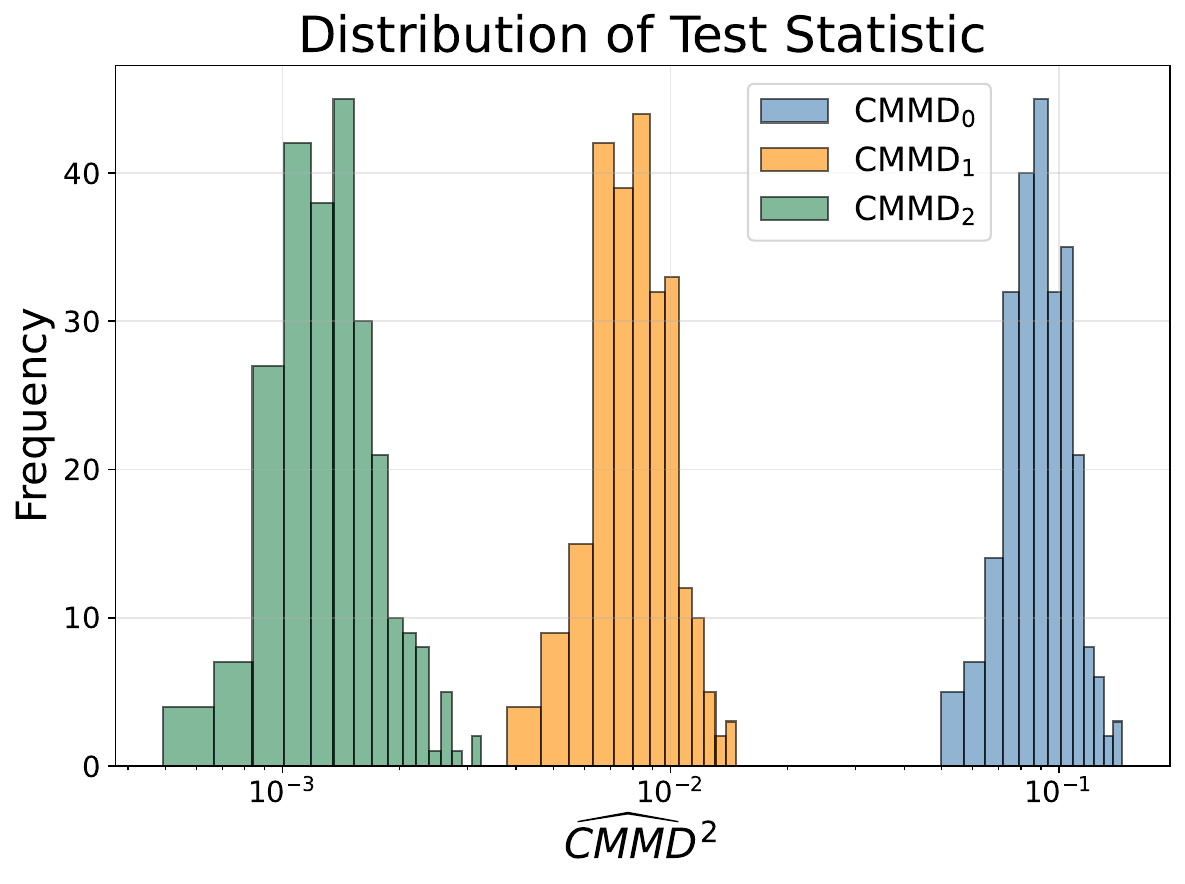}
        \caption{$P_X = Q_X$}
    \end{subfigure}
    \hfill 
    \begin{subfigure}[b]{0.48\textwidth}
        \centering
        \includegraphics[width=\linewidth]{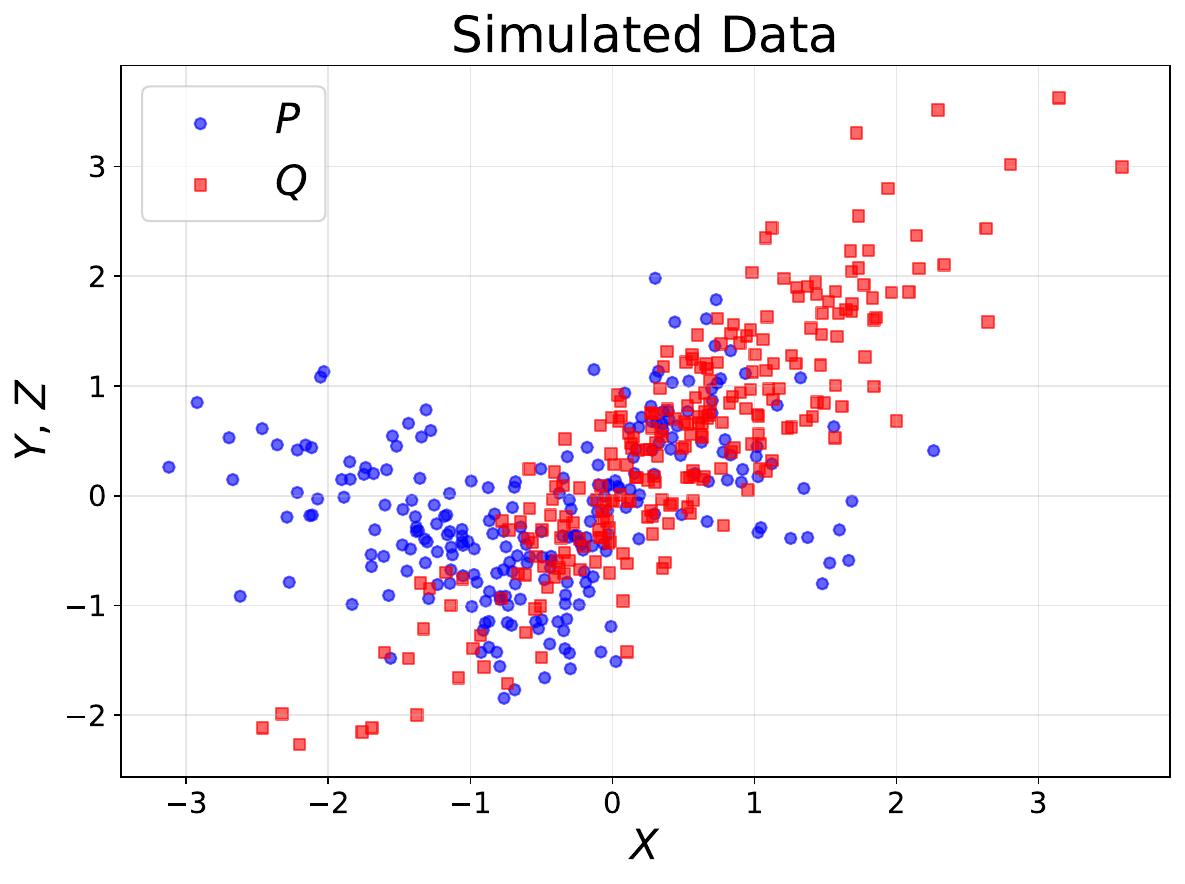}
        \includegraphics[width=\linewidth]{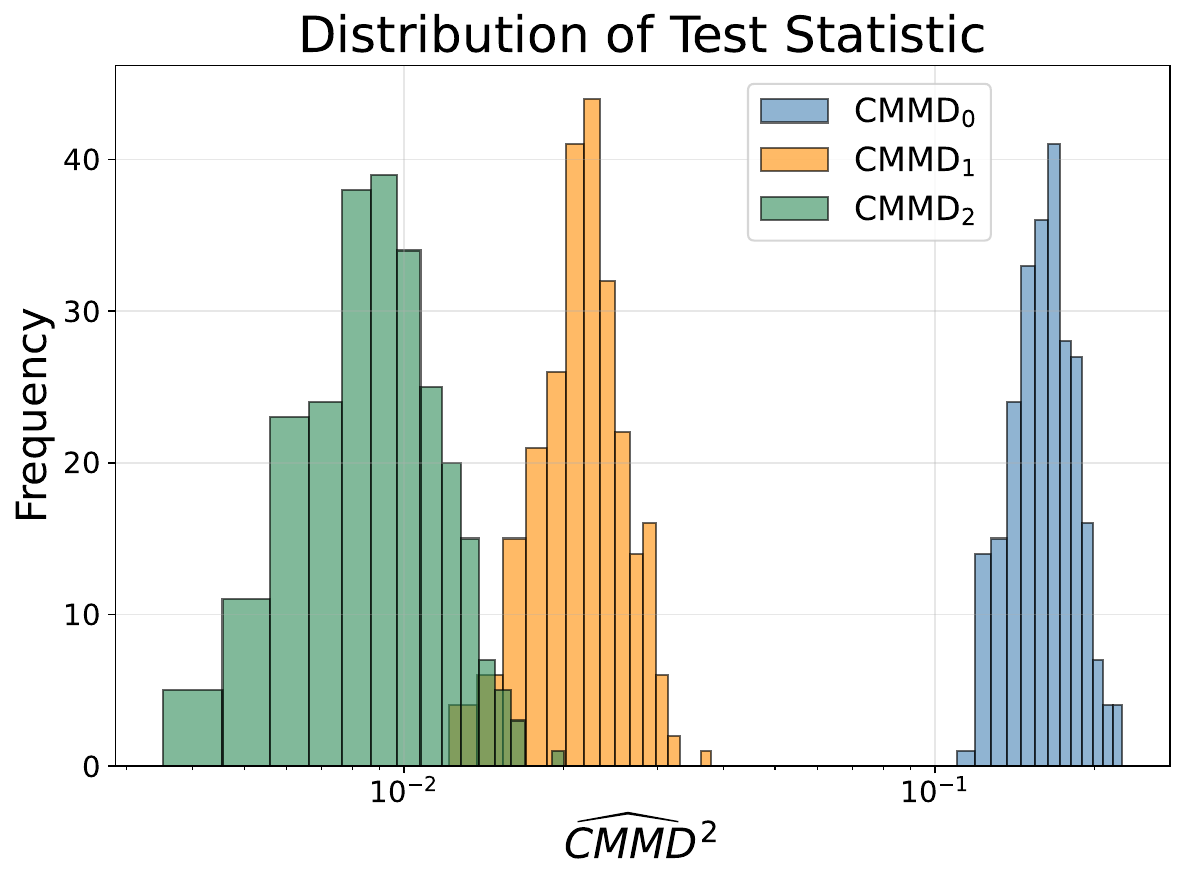}
        \caption{$P_X \neq Q_X$}
    \end{subfigure}
    
    \caption{Plots of simulated data (top) and distribution of test statistics (bottom) under Setting 1 (left) and Setting 2 (right). Top: A sample of data from distributions $P$ and $Q$. Bottom: Distribution of test statistics plotted on a log scale.}
    \label{fig:synthetic_experiment_ap}
\end{figure}

Next we move to hypothesis testing. We use $B=250$ bootstrap samples with a significance level of $0.05$, and use 250 trials to estimate the Type I error and power. The results are displayed in Figure \ref{fig:synthetic_testing_ap}. The top row demonstrates that we have Type I error control for both algorithms. The bottom row illustrates that power converges to 1 under the alternative, however this rate is fastest for CMMD$_0$ and slowest for CMMD$_2$, particularly when $P_X \neq Q_X$. All three test statistics perform better under setting 1 than setting 2. We believe this is due to the higher concentration of data in setting 1 which allows to better fit conditional relationships. 

\begin{figure}[htbp]
    \centering
    \begin{subfigure}[b]{0.48\textwidth}
        \centering
        \includegraphics[width=\linewidth]{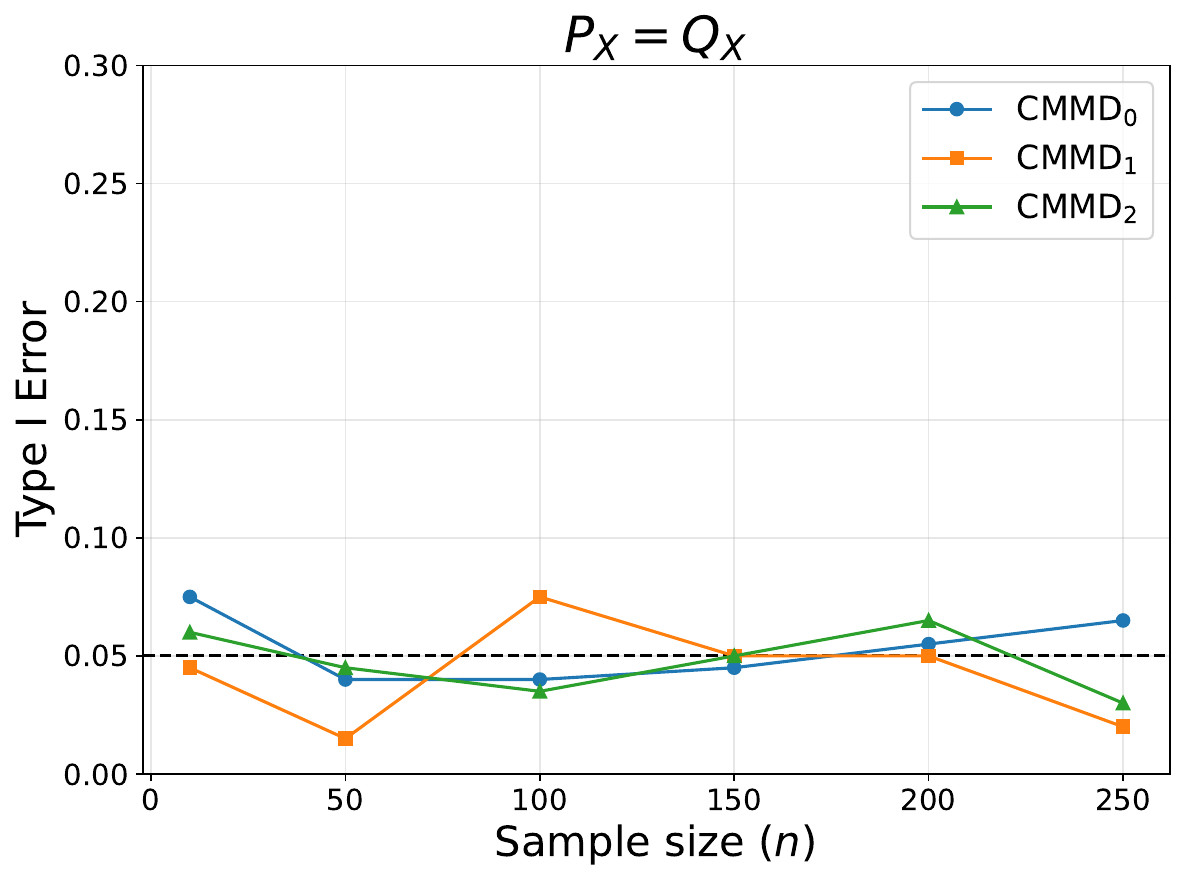}
        \includegraphics[width=\linewidth]{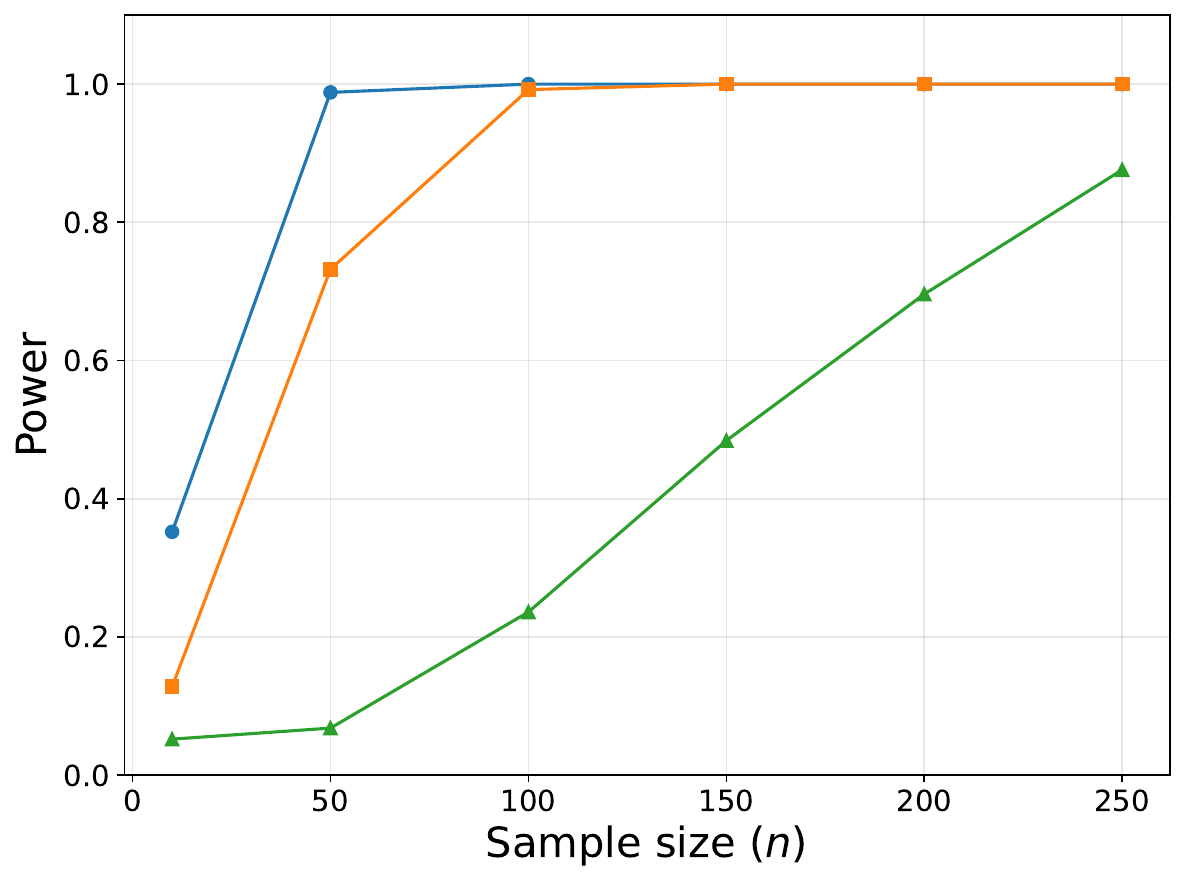}
    \end{subfigure}
    \hfill 
    \begin{subfigure}[b]{0.48\textwidth}
        \centering
        \includegraphics[width=\linewidth]{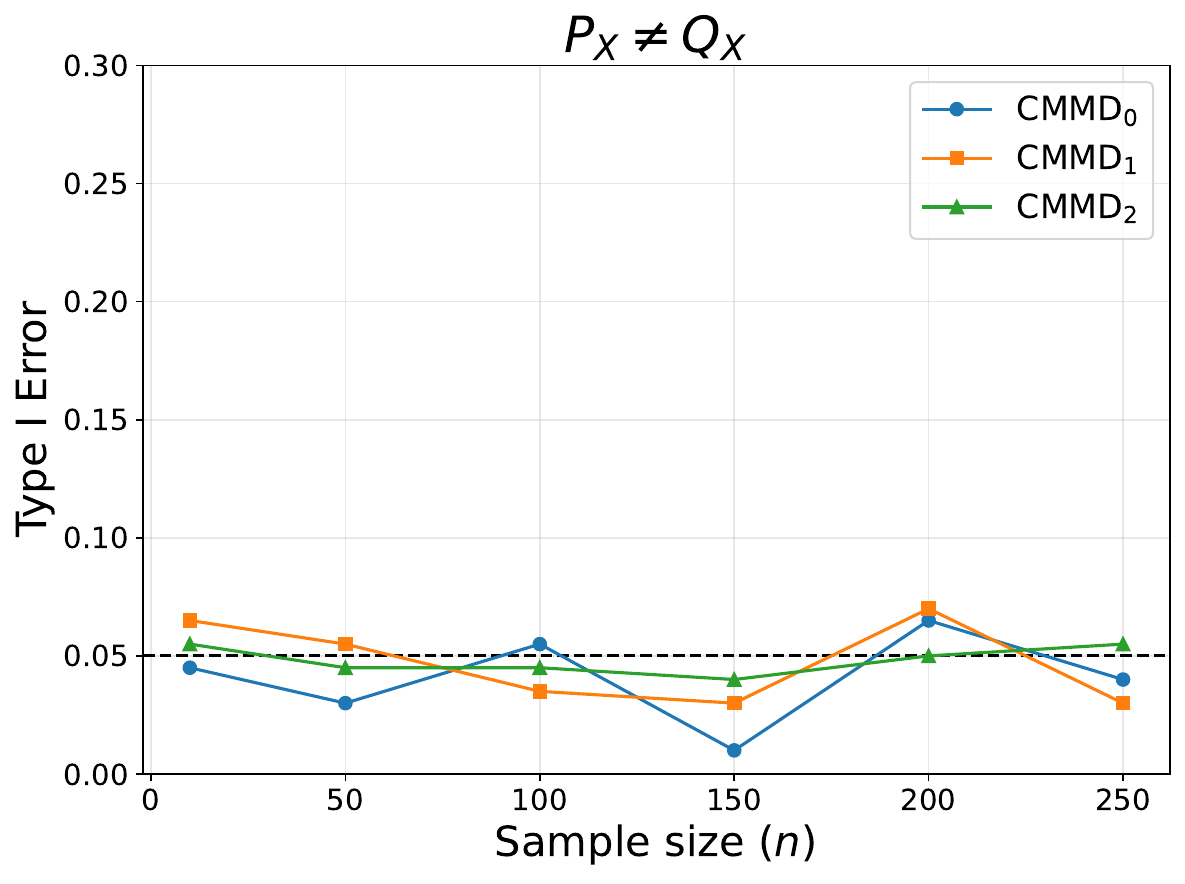}
        \includegraphics[width=\linewidth]{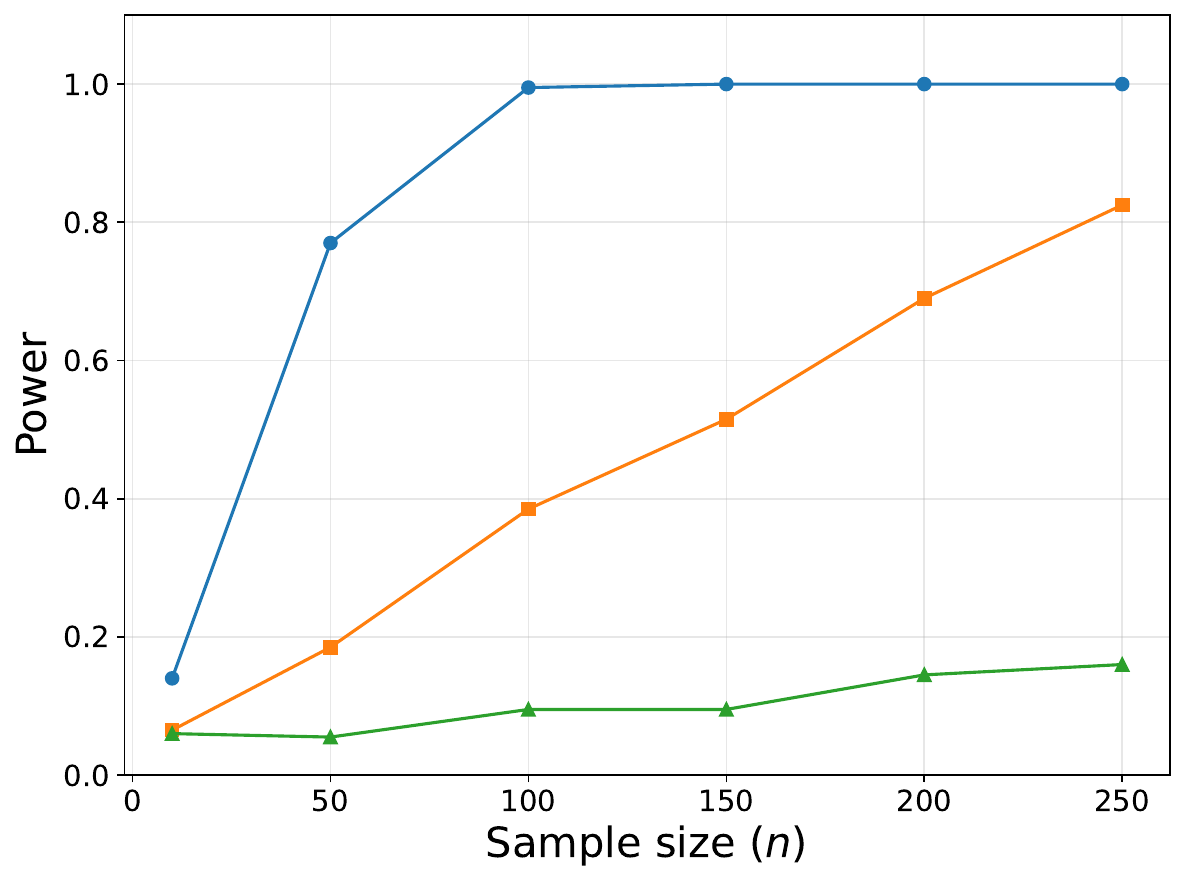}
    \end{subfigure}
    
    \caption{Rejection rates of hypothesis test under Setting 1 (left) and Setting 2 (right). Both Algorithm \ref{alg:1} and Algorithm \ref{alg:2} demonstrate Type I error control under the null (top) and power converging to 1 under the alternative for each CMMD test statistic (bottom).}
    \label{fig:synthetic_testing_ap}
\end{figure}

\subsection{General covariate mixture}

So far, our experiments have all used $n=m$ sample sizes from each of $P$ and $Q$. In this subsection, we investigate the impact of unbalanced sample sizes and how this can be overcome by changing the parameter $\rho$ indicating the reference mixture proportion. Recall that $\rho$ is used in estimating the covariance operator $\hat C_{XX}^R$ as in equation (\ref{eq:cov_op_est}) which in turn is used for estimating CMMD levels $s>0$. While $\rho = n/(n+m)$ is a natural choice, it is not a requirement. We once more take a setup similar to Section \ref{sec:exp_synthetic}, but with $P_X = \mathcal N(-0.5, 1)$ and $Q_X = \mathcal N(0.5, 1)$. Conditional relationships, kernels and hyperparameters remain the same. We sample $n$ points from $P$ and $m$ from $Q$ such that $n+m=200$.

\begin{figure}[htbp]
    \centering
    \begin{subfigure}[b]{0.48\textwidth}
        \centering
        \includegraphics[width=\linewidth]{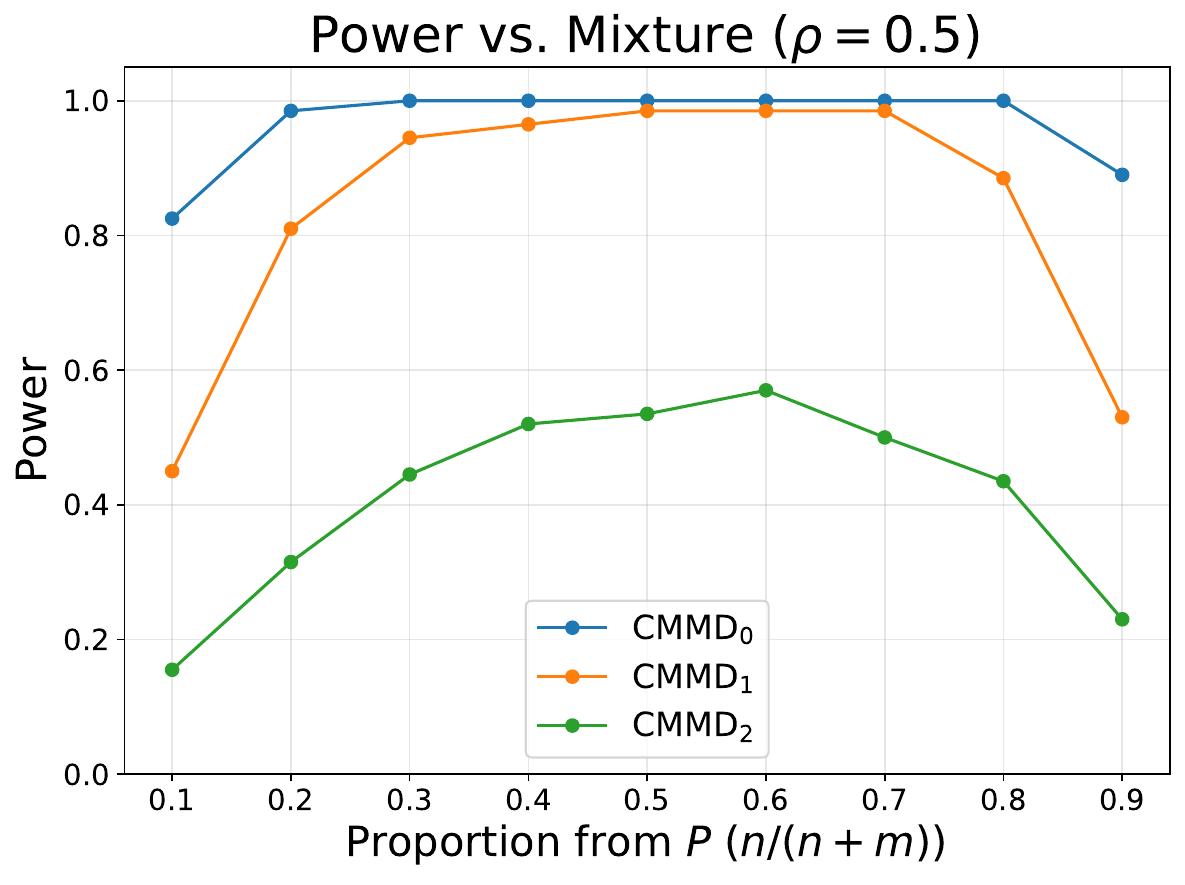}
    \end{subfigure}
    \hfill 
    \begin{subfigure}[b]{0.48\textwidth}
        \centering
        \includegraphics[width=\linewidth]{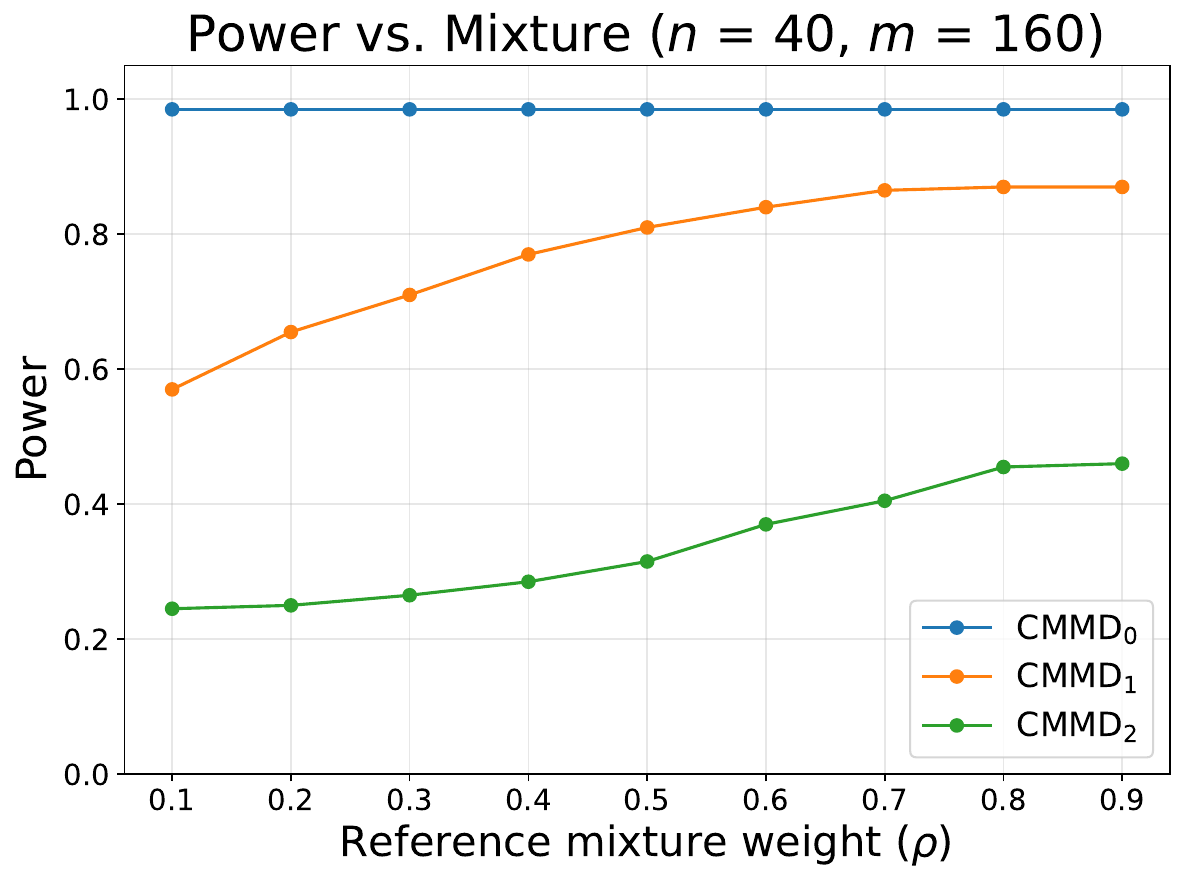}
    \end{subfigure}
    \caption{Rejection rates for CMMD test with unbalanced data. Left: More balanced data sets give higher test power. Right: Increasing $\rho$ to put greater weight on the samples from $P$ gives higher test power.}
    \label{fig:mixture_experiment}
\end{figure}

In Figure \ref{fig:mixture_experiment} (left), we keep $\rho$ fixed at 0.5, and vary the proportion of data samples coming from each of $P$ and $Q$. When the data sets are roughly balanced, that is $n/(n+m) \approx 0.5$, we have higher test power. In Figure \ref{fig:mixture_experiment} (right), we fix the sample sizes to $n=40$ and $m=160$ and vary the reference mixture weight. We see that changing $\rho$ does not affect the power of CMMD$_0$ as the test statistic does not use $\hat C_{XX}^R$. On the other hand, increasing $\rho$, which puts more weight on the $n=40$ samples from $P$, gives greater power when using the CMMD$_1$ and CMMD$_2$ test statistics. Thus, the mixture-weight parameter $\rho$ can help recover lost power when using an unbalanced data set.

\subsection{Doubly Robust Estimator} \label{ap:experiment_details_dr}

Here we consider the same setting as in Section \ref{sec:exp_dr}. We sample $n=500$ data from each of $P$ and $Q$ and fit CME models $\hat \mu_{Y|X}$ and $\hat \mu_{Z|X}$ using the polynomial kernel $k$. The results are illustrated in Figure \ref{fig:dr_experiment_data} (left and middle), along with the true CME functions. Due to the misspecified kernel, the models are unable to capture the true relationship between $X$ and $Y$/$Z$. Next, we compute pseudo-outcomes, giving a total of 1000 data points on which we learn the doubly robust CME difference estimator $\hat \Delta k(\cdot, x)$. The resulting fit is plotted in Figure \ref{fig:dr_experiment_data} (right). The same curve is given in Figure \ref{fig:dr_experiment} (left) where it is shown to more closely match the true difference between CMEs compared to simply taking $\hat \mu_{Y|x} - \hat \mu_{Z|x}$.

\begin{figure}[htbp]
    \centering
    \begin{subfigure}[b]{0.32\textwidth}
        \centering
        \includegraphics[width=\linewidth]{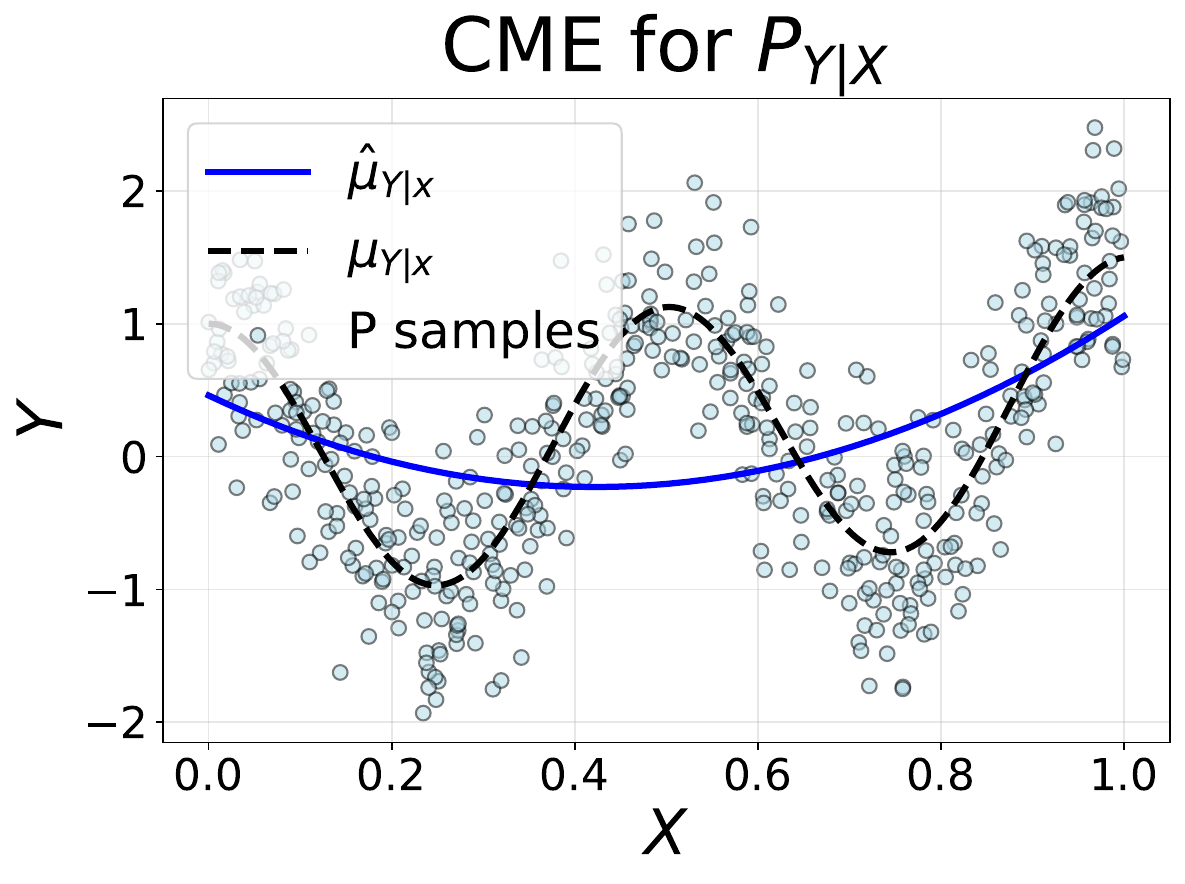}
    \end{subfigure}
    \hfill 
    \begin{subfigure}[b]{0.32\textwidth}
        \centering
        \includegraphics[width=\linewidth]{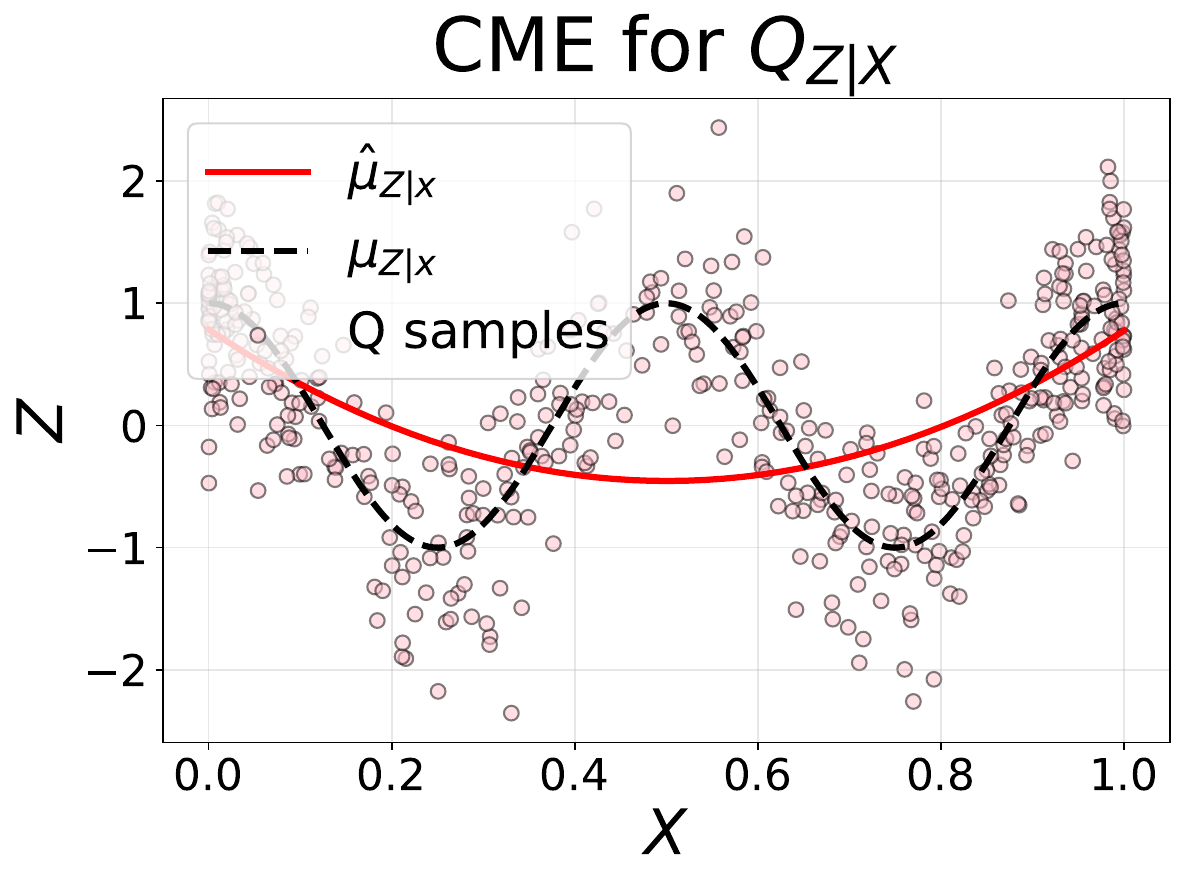}
    \end{subfigure}
    \hfill
    \begin{subfigure}[b]{0.32\textwidth}
        \centering
        \includegraphics[width=\linewidth]{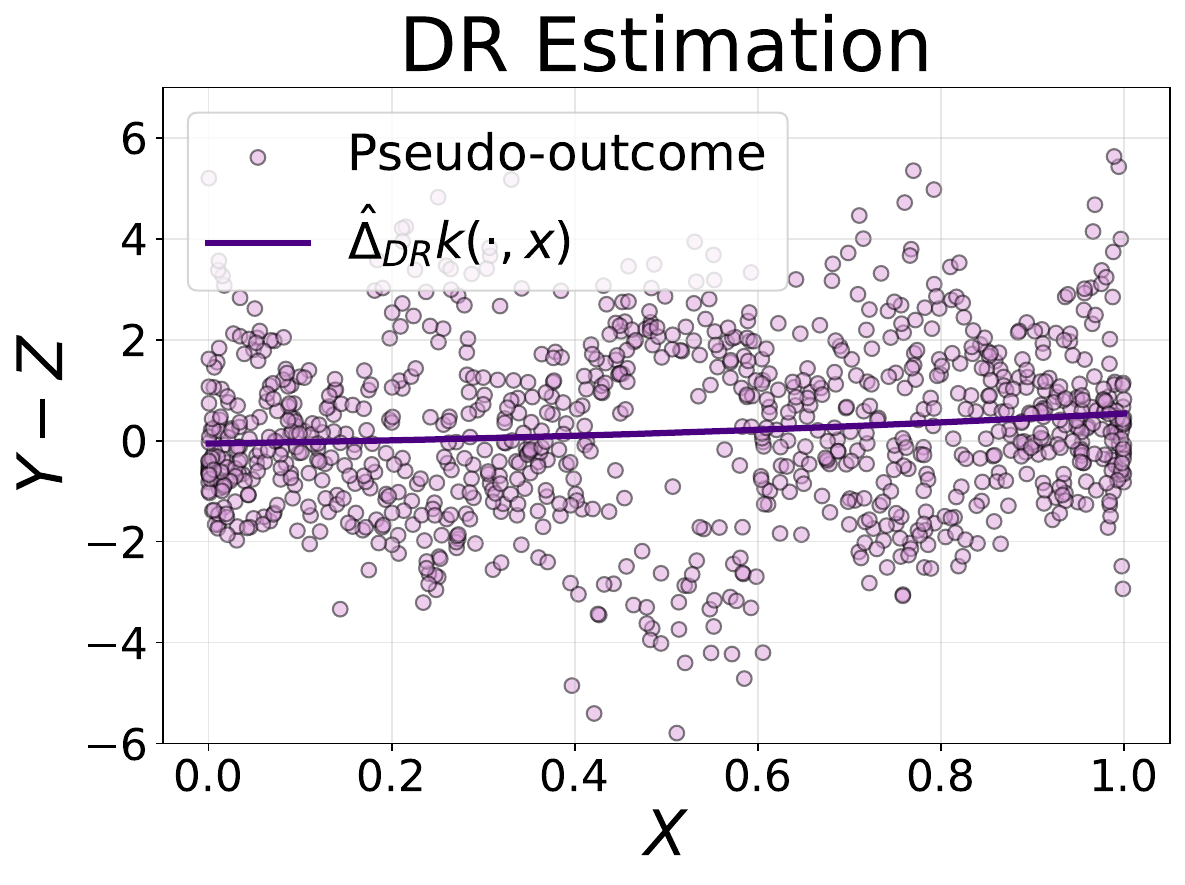}
    \end{subfigure}
    \caption{Left: Data sampled from $P$ and CME model $\hat \mu_{Y|X}$. Middle: Data sampled from $Q$ and CME model $\hat \mu_{Z|X}$. Right: Pseudo-outcomes computed on the combined data and the DR model $\hat \Delta_{DR} k(\cdot, x)$. Despite individual CME models being misspecified, the DR estimator fitted on the pseudo-outcomes correctly models the true difference between CMEs.}
    \label{fig:dr_experiment_data}
\end{figure}

When computing the DR estimators for Figure \ref{fig:dr_experiment} (middle, right), we use cross-fitting in order to learn nuisance models and estimate the CMMD on a single data set. Partitioning the samples into $K=5$ folds of equal size, the CME model is learned on 4 folds and pseudo-outcomes are computed on the remaining data. Thus, for each fold $k$, we compute the estimator $\hat \Delta_{DR,k}$ and then aggregate to get 
\begin{equation*}
    \hat \Delta_{DR,CF} = \frac{1}{K} \sum_{k=1}^K \hat \Delta_{DR,k}.
\end{equation*}
Then,
\begin{equation*}
    \| \Delta - \hat \Delta_{DR,CF} \|_{\mathcal H_\ell \otimes \mathcal H_k} \leq \frac{1}{K} \sum_{k=1}^K \| \Delta - \hat \Delta_{DR,k} \|_{\mathcal H_\ell \otimes \mathcal H_k} \overset{p}{\to} 0
\end{equation*}
which follows from each $\hat \Delta_{DR,k}$ converging to $\Delta$ by Theorem \ref{thm:cmmd_dr_conv}. 

Next, we consider two-sample testing under the null hypothesis, that is, conditional outcomes are given by $Y|X = \cos(4 \pi X) + 0.5 X^2 + \epsilon$ under both $P$ and $Q$. Figure \ref{fig:dr_experiment_typeI} plots the rejection rate with increasing sample size, and illustrates that all CMMD tests, both using standard and doubly robust estimators, exhibit Type I error control.

\begin{figure}[htbp]
    \centering
    \includegraphics[width=0.45\linewidth]{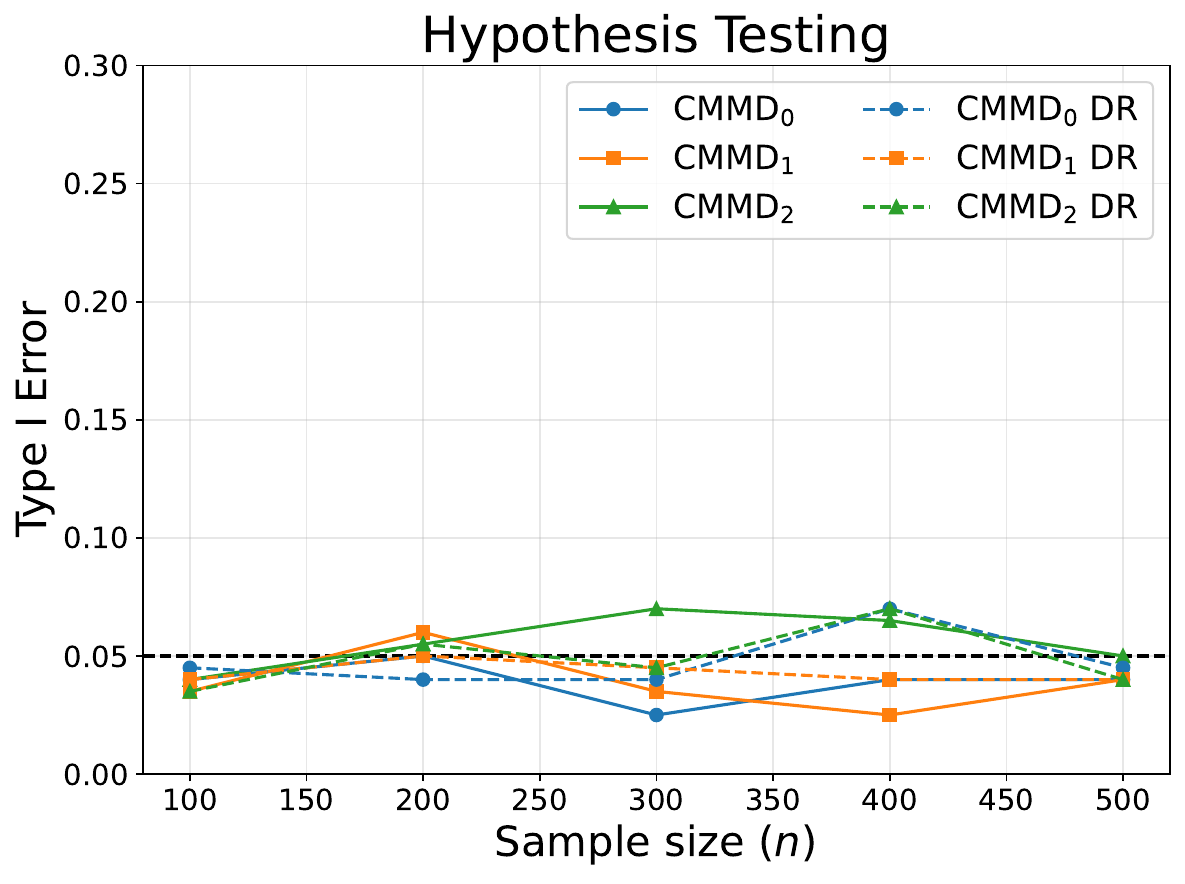}
    \caption{Rejection rate under the null hypothesis. All tests, both using standard and doubly robust CMMD estimators, exhibit Type I error control.}
    \label{fig:dr_experiment_typeI}
\end{figure}

\subsection{MNIST Data}\label{ap:experiment_details_mnist}

The experiment in Section \ref{sec:exp_real} uses the test set of the MNIST digit classification dataset. The covariates are the digits $\mathcal X = \{0, 1, 2, 3, 4, 5, 6, 7, 8, 9\}$ and the outcomes are all possible 28 $\times$ 28 black and white pixel images. The pixels can take integer values in the range 0 to 255. Before completing the analysis, pixel values were standardized, and the dimension was reduced to 100 through PCA, meaning outcomes are represented as vectors in $\mathbb R^{100}$ which we set as $\mathcal Y$. To sample digits from $P_X$, we use the probability mass function $p(x) = 0.1$ for $x \in \mathcal X$. For biased sampling with $Q_X$, we use the probability mass function $q(x) = 0.145 - 0.01x$ for $x \in \mathcal X$. Thus, the propensity is known to be exactly $e(x) = \frac{1}{2.45 - 0.1x}$. 

We note that computing the test statistics as described in Section \ref{sec:estimation} requires taking the inverse of an $n \times n$ matrix, which has $O(n^3)$ computational complexity. To ease computational costs, we instead estimate the CMOs in primal form
\begin{equation*}
    \hat C_{Y|X} = \hat C_{YX} (\hat C_{XX} + \lambda I)^{-1} = \Psi_\mathbf{Y} \Phi_\mathbf{X}^* (\Phi_\mathbf{X} \Phi_\mathbf{X}^* + \lambda nI)^{-1}
\end{equation*}
Since $k$ is chosen to be the Kronecker delta kernel, the feature mapping $k(\cdot, x)$ is just the one-hot encoding of $x$ and is represented as a vector in $\mathbb R^{10}$. Therefore, we are still able to compute all test statistics in closed form.

\end{document}